\documentclass{article}
\usepackage[numbers]{natbib}
\usepackage{blindtext}
\usepackage[final, nonatbib]{neurips_2019}
\usepackage[utf8]{inputenc} 
\usepackage[T1]{fontenc}    
\usepackage{hyperref}       
\usepackage{url}            
\usepackage{booktabs}       
\usepackage{amsfonts}       
\usepackage{nicefrac}       
\usepackage{microtype}      
\usepackage{lipsum}
\usepackage{microtype}
\usepackage{graphicx}
\usepackage{color}
\usepackage{amsmath}
\usepackage{amssymb}
\usepackage{algorithm}
\usepackage{algorithmic}
\usepackage[margin=0.1pt]{subfig}
\usepackage{wrapfig}
\usepackage{tikz-cd}

\newenvironment{proof}{\par\noindent{\bf Proof\ }}{\hfill\BlackBox\\[2mm]}
\newtheorem{theorem}{Theorem}

\newtheorem{lemma}{Lemma}
\newtheorem{proposition}{Proposition}

\newtheorem{definition}{Definition}

\newtheorem{assumption}{Assumption}
\newcommand{\BlackBox}{\rule{1.5ex}{1.5ex}}

\def\EB{{\mathbb E}}

\def\TP{{\mathcal T}}
\def\AP{{\mathcal A}}
\def\SP{{\mathcal S}}
\def\DP{{\mathcal D}}
\def\pio{{\pi_{\text{old}}}}
\def\pin{{\pi_{\text{new}}}}
\def\barone{{\bar{\theta}_{1}}}

\def\argmax{\mathop{\rm argmax}}

\def\liml{\mathop{\lim}\limits}

\title{A Regularized Approach to Sparse Optimal Policy in Reinforcement Learning}

\author{
Xiang Li\thanks{Equal contribution.}\\
School of Mathematical Sciences\\
Peking University\\
Beijing, China \\
\texttt{lx10077@pku.edu.cn} \\
\And
Wenhao Yang$^*$\\
Center for Data Science\\
Peking University\\
Beijing, China \\
\texttt{yangwenhaosms@pku.edu.cn}\\
   \And
Zhihua Zhang\\
National Engineering Lab for Big Data Analysis and Applications \\
School of Mathematical Sciences\\
Peking University\\
Beijing, China \\
\texttt{zhzhang@math.pku.edu.cn} \\
}

\begin{document}
\maketitle
\begin{abstract}
We propose and study a general framework for regularized Markov decision processes (MDPs) where the goal is to find an optimal policy that maximizes the expected discounted total reward plus a policy regularization term. 
The extant entropy-regularized MDPs can be cast into our framework. 
Moreover, under our framework, many regularization terms can bring multi-modality and sparsity, which are potentially useful in reinforcement learning. 
In particular, we present sufficient and necessary conditions that induce a sparse optimal policy. We also conduct a full mathematical analysis of the proposed regularized MDPs, including the optimality condition, performance error, and sparseness control. We provide a generic method to devise regularization forms and propose off-policy actor critic algorithms in complex environment settings. We empirically analyze the numerical properties of optimal policies and compare the performance of different sparse regularization forms in discrete and continuous environments. 
\end{abstract}

\section{Introduction}
\label{sec:introduction}
Reinforcement learning (RL) aims to find an optimal policy that maximizes the expected discounted total reward in an MDP \cite{bertsekas2008neuro, sutton2018reinforcement}. 
It's not an easy task to solve the \textit{nonlinear} Bellman equation~\cite{bellmann1957dynamic} greedily in a high-dimension action space or when function approximation (such as neural networks) is used.
Even if the optimal policy is obtained precisely, it is often the case the optimal policy is deterministic.
Deterministic policies are bad to cope with unexpected situations when its suggested action is suddenly unavailable or forbidden.
By contrast, a multi-modal policy assigns positive probability mass to both optimal and near optimal actions and hence has multiple alternatives to handle this case.
For example, an autonomous vehicle aims to do path planning with a pair of departure and destination as the state.
When a suggested routine is unfortunately congested, an alternative routine could be provided by a multi-modal policy, which can't be provided by a deterministic policy without evoking a new computation.
Therefore, in a real-life application, we hope the optimal policy to possess thee property of multi-modality.


Entropy-regularized RL methods have been proposed to handle the issue. 
More specifically, an entropy bonus term is added to the expected long-term returns. 
As a result, it not only softens the non-linearity of the original Bellman equation but also forces the optimal policy to be stochastic, which is desirable in problems where dealing with unexpected situations is crucial. 
In prior work, the \textit{Shannon entropy} is usually used. 
The optimal policy is of the form of \textit{softmax}, which has been shown can encourage exploration \cite{fox2015taming, vamplew2017softmax}. 
However, a softmax policy assigns a non-negligible probability mass to all actions, including those really terrible and dismissible ones, which may result in an unsafe policy. 
For RL problems with high dimensional action spaces, a sparse distribution is preferred in modeling a policy function, because it implicitly does \textit{action filtration}, i.e., weeds out suboptimal actions and maintains near optimal actions. 
Thus, \citet{lee2018sparse} proposed to use \textit{Tsallis entropy} \cite{tsallis1988possible} instead, giving rise to a \textit{sparse} MDP where only few actions have non-zero probability at each state in the optimal policy. 
\citet{lee2019tsallis} empirically showed that general Tsallis entropy\footnote{The general Tsallis entropy is defined with an additional real-valued parameter, called an entropic index. \citet{lee2019tsallis} shows that when this entropic index in large enough, the optimal policy is sparse.}
also leads to a sparse MDP. 
Moreover, the Tsallis regularized RL has a lower performance error, i.e., the optimal value of the Tsallis regularized RL is closer to the original optimal value than that of the Shannon regularized RL. 

The above discussions manifest that an entropy regularization characterizes the solution to the corresponding regularized RL. 
From \citet{neu2017unified}, any entropy-regularized MDP can be viewed as a regularized convex optimization problem where the entropy serves as the regularizer and the decision variable is a stationary policy. 
\citet{geist2019theory} proposed a framework in which the MDP is regularized by a general strongly concave function. 
It analyzes some variants of classic algorithms under that framework but does not provide insight into the choice of regularizers. 
On the other hand, a sparse optimal policy distribution is more favored in large action space RL problems. 
Prior work \citet{lee2018sparse, nachum2018path} obtains a sparse optimal policy by the Tsallis entropy regularization. 
Considering the diversity and generality of regularization forms in convex optimization, it is natural to ask whether other regularizations can lead to sparseness. 
The answer is that there does exist other regularizers that induces sparsity.

In this paper, we propose a framework for regularized MDPs, where a general form of regularizers is imposed on the expected discounted total reward. 
This framework includes the entropy regularized MDP as a special case, implying certain regularizers can induce sparseness. 
We first give the optimality condition in regularized MDPs under the framework and then give necessary and sufficient conditions to show which kind of regularization can lead to a sparse optimal policy.
Interestingly, there are lots of regularization that can lead to the sparseness, and the degree of sparseness can be controlled by the regularization coefficient. 
Furthermore, we show that regularized MDPs have a regularization-dependent performance error caused by the regularization term, which could guide us to choose an effective regularization when it comes to dealing with problems with a continuous action space. 
To solve regularized MDPs, we employ the idea of generalized policy iteration and propose an off-policy actor-critic algorithm to figure out the performance of different regularizers.

\section{Notation and preliminaries}
\label{sec:notation}

\textbf{Markov Decision Processes} In reinforcement learning (RL) problems, the agent's interaction with the environment is often modeled as an Markov decision process (MDP). An MDP is defined by a tuple $(\mathcal{S}, \mathcal{A}, \mathbb{P}, r, \mathbb{P}_0, \gamma)$, where $\mathcal{S}$ is the state space and $\mathcal{A}$ the action space with $|\mathcal{A}|$ actions. We use $\Delta_{\mathcal{X}}$ to denote the simplex on any set $\mathcal{X}$, which is defined as the set of distributions over $\mathcal{\mathcal{X}}$, i.e., $\Delta_{\mathcal{X}} = \{P: \sum_{x \in \mathcal{X}} P(x) =1, P(x) \ge 0\}$. The vertex set of $\Delta_{\mathcal{X}}$ is defined as $V_{\mathcal{X}}=\{P \in \Delta_{\mathcal{X}}:  \exists \ x \in \mathcal{X}, \text{s.t.} \ P(x) = 1\}$. $\mathbb{P}: \mathcal{S} \times \mathcal{A} \rightarrow \Delta_{\mathcal{S}}$ is the unknown state transition probability distribution and $r: \mathcal{S} \times \mathcal{A} \rightarrow [0, R_{\max}]$ is the bounded reward on each transition. $\mathbb{P}_0$ is the distribution of initial state and $\gamma \in [0, 1)$ is the discount factor. 

\textbf{Optimality Condition of MDP}

The goal of RL is to find a stationary policy which maps from state  space to a simplex over the actions $\pi: \mathcal{S} \rightarrow \Delta_{\mathcal{A}}$ that maximizes the expected discounted total reward, i.e., 
\begin{equation}
\label{prob:original}
\max_{\pi} \, \mathbb{E}\left[ \sum_{t=0}^{\infty} \gamma^t r(s_t, a_t) \Big|\pi, \mathbb{P}_0\right],
\end{equation}
where $s_0 \sim \mathbb{P}_0, a_t \sim \pi(\cdot|s_t)$, and $s_{t+1} \sim \mathbb{P}(\cdot|s_t, a_t)$. Given any policy $\pi$, its state value and Q-value functions are defined respectively as
\begin{align*}
V^{\pi}(s) &= \mathbb{E} \left[\sum_{t=0}^{\infty}\gamma^t r(s_t, a_t)|s_0=s, \pi \right],\\
Q^{\pi}(s, a) &= \mathbb{E}_{a \sim \pi(\cdot|s)} \big[r(s, a) + \gamma \EB_{s'|s, a}  V^{\pi}(s') \big].
\end{align*}
Any solution of the problem~\eqref{prob:original} is called an \textit{optimal} policy and  denoted by $\pi^*$. Optimal policies may not be unique in an MDP, but the optimal state value is unique (denoted $V^*$). Actually, $V^*$ is the unique fixed point of the Bellman operator  $\TP$, i.e., $V^*(s) = \TP V^*(s)$ and
\begin{equation*}
\TP V(s) \triangleq \max_{\pi} \EB_{a \sim \pi(\cdot|s)} \big[r(s, a) + \gamma \mathbb{E}_{s'|s, a} V(s') \big]. 
\end{equation*}
$\pi^*$ often is a \textit{deterministic} policy which puts all probability mass on one action\cite{puterman2014markov}. Actually, it can be obtained as the greedy action w.r.t. the optimal Q-value function, i.e., $\pi^*(s) \in \argmax_{a} Q^*(s, a)$ \footnote{$\pi^*$ is not necessarily deterministic. If there are two actions $a_1, a_2$ that obtain the maximum of $\TP V(s)$ for a fixed $s \in \SP$, one can show that the stochastic policy $\pi(a_1|s)=1-\pi(a_2|s)=p \in [0, 1]$ is also optimal.}. The optimal Q-value can be obtained from the state value $V^*(s)$ by definition.

As a summary, any optimal policy $\pi^*$ and its optimal state value $V^*$ and Q-value $Q^*$ satisfy the following \textit{optimality condition} for all states and actions,
\begin{align*}
Q^*(s, a) & = r(s,a) + \gamma \mathbb{E}_{s'|s, a}V^*(s),\\
V^*(s) &= \max_{a} Q^*(s, a), \ \pi^*(s) \in \argmax_{a} Q^*(s, a).
\end{align*}

\section{Regularized MDPs}
\label{sec:rmdp}
To obtain a sparse but multi-modal optimal policy, we impose a general regularization term to the objective~\eqref{prob:original} and solve the following \textit{regularized} MDP problem
\begin{equation}
\label{prob:genera-regularized-rl}
\max_{\pi} \, \mathbb{E}\left[ \sum_{t=0}^{\infty} \gamma^t (r(s_t, a_t) + \lambda \phi(\pi(a_t|s_t))) \Big|\pi, \mathbb{P}_0\right],
\end{equation}
where $\phi(\cdot)$ is a regularization function. Problem~\eqref{prob:genera-regularized-rl} can be seen as a RL problem in which the reward function is the sum of the original reward function $r(s, a)$ and a term $\phi(\pi(a|s))$ that provides regularization. If we take expectation to the regularization term $\phi(\pi(a|s))$, it can be found that the quantity
\begin{equation}
\label{eq: entropy-like}
H_{\phi}(\pi) = \mathbb{E}_{a\sim \pi(\cdot|s)} \phi(\pi(a|s)) ,
\end{equation}
is entropy-like but not necessarily an entropy in our work. However, Problem~\eqref{prob:genera-regularized-rl}  is not well-defined since arbitrary regularizers would be more of a hindrance than a help. In the following, we make some assumptions about $\phi(\cdot)$.

\subsection{Assumption for regularizers}\label{sec:assumption}

A  regularizer $\phi(\cdot)$ characterizes solutions to Problem~\eqref{prob:genera-regularized-rl}. In order to make Problems~\eqref{prob:genera-regularized-rl} analyzable, a basic assumption (Assumption~\ref{assump:1}) is prerequisite. Explanation and examples will be provided to show that such an assumption is reasonable and not strict.

\begin{assumption}
\label{assump:1}
The regularizer $\phi(x)$ is assumed to satisfy the following conditions on the interval $(0, 1]$: \emph{(1)} \textbf{Monotonicity}: $\phi(x)$ is non-increasing; \emph{(2)} \textbf{Non-negativity}: $\phi(1)=0$; \emph{(3)} \textbf{Differentiability}: $\phi(x)$ is differentiable; \emph{(4)} \textbf{Expected Concavity}: $x\phi(x)$ is strictly concave.
\end{assumption}

The assumptions of monotonicity and non-negativity make the regularizer  be an positive exploration bonus. 
The bonus for choosing an action of high probability is less than that of choosing an action of low probability. When the policy becomes deterministic, the bonus is forced to zero. The assumption of differentiability facilitates theoretic analysis and benefits practical implementation due to the widely used automatic derivation in deep learning platforms. The last assumption of expected concavity makes $H_{\phi}(\pi)$ a concave function w.r.t. $\pi$. Thus any solution to~Eqn.\eqref{prob:genera-regularized-rl} hardly lies in the vertex set of the action simplex $V_{\mathcal{A}}$ (where the policy is deterministic). As a byproduct, let $f_{\phi}(x) = x\phi(x)$. Then its derivative $f_{\phi}'(x)=\phi(x)+x\phi'(x)$ is a strictly decreasing function on $(0, 1)$ and thus $(f_{\phi}')^{-1}(x)$ exists. For simplicity, we denote $g_{\phi}(x) = (f_{\phi}')^{-1}(x)$.

There are plenty of options for the regularizer $\phi(\cdot)$ that satisfy Assumption~\ref{assump:1}. First, entropy can be recovered by $H_{\phi}(\pi)$ with specific $\phi(\cdot)$. For example, when $\phi(x)=-\log x$, the Shannon entropy is recovered; when $\phi(x)=\frac{k}{q-1}(1{-}x^{q-1})$ with $k>0$, the Tsallis entropy is recovered. Second, there are many instances that are not viewed as an entropy but can serve as a regularizer. We find two families of such functions,  namely, the exponential function family $q - x^k q^x $ with $k \geq 0, q \geq 1$ and the trigonometric function family $\cos(\theta x)-\cos(\theta)$ and $\sin(\theta)-\sin(\theta x)$ both with hyper-parameter $\theta \in (0, \frac{\pi}{2}]$.  Since these functions are simple, we term them \textit{basic} functions. 

Apart from the basic functions mentioned earlier, we come up with a generic method to combine different basic functions. Let $\mathcal{F}$ be the set of all functions satisfying Assumption~\ref{assump:1}. By Proposition~\ref{prop:phi-operation}, the operations of positive addition and minimum can preserve the properties shared among $\mathcal{F}$. Therefore, the finite-time application of such operations still leads to an available regularizer.

\begin{proposition}
\label{prop:phi-operation}
Given  $\phi_1, \phi_2 \in \mathcal{F}$, we have $\alpha\phi_1+\beta\phi_2 \in \mathcal{F}$ for all $\alpha, \beta \ge 0$ and $\min\{\phi_1, \phi_2\} \in \mathcal{F}$.
\end{proposition}
Here we only consider those differentiable $\min\{\phi_1, \phi_2\}$ in theoretical analysis, because the minimum of any two functions in $\mathcal{F}$ may be non-differentiable  on some points. For instance, given any $q > 1$, the minimum of $-\log(x)$ and $q(1-x)$ has a unique non-differentiable point on $(0, 1)$. 

\subsection{Optimality and sparsity}

Once the regularizer $\phi(\cdot)$ is given, similar to non-regularized case, the (regularized) state value and Q-value functions of any given policy $\pi$ in a  regularized MDP are defined as 
\begin{small}
\begin{align}
\label{eq:Q=r+V}
V^{\pi}_{\lambda}(s)&=\mathbb{E}\Big[\sum_{t=0}^{+\infty}\gamma^{t}(r(s_{t},a_{t})+\lambda\phi(\pi(a_{t}|s_{t}))) \Big|s_{0}=s, \pi \Big], \notag \\
Q^{\pi}_{\lambda}(s, a)&=r(s, a) + \gamma\mathbb{E}_{a \sim \pi(\cdot|s)} \mathbb{E}_{s'|s, a}V^{\pi}_{\lambda}(s'). 
\end{align}
\end{small}
Any solution to Problem~\eqref{prob:genera-regularized-rl} is call the \textit{ regularized optimal} policy (denoted $\pi_{\lambda}^*$). The corresponding \textit{regularized optimal} state value and Q-value are also optimal and denoted by $V_{\lambda}^*$ and $Q_{\lambda}^*$ respectively. If the context is clear, we will omit the word \textit{regularized} for simplicity. 
In this part, we aim to show the optimality condition for the regularized MDPs (Theorem \ref{thm:optimality}). The  proof of Theorem \ref{thm:optimality} is given in Appendix~\ref{ap: optimality}.

\begin{theorem}
\label{thm:optimality}
Any optimal policy $\pi_{\lambda}^*$ and its optimal state value $V_{\lambda}^*$ and Q-value $Q_{\lambda}^*$ satisfy the following \textit{optimality condition} for all states and actions:
\begin{align}
Q_{\lambda}^*(s, a) &= r(s,a) + \gamma \mathbb{E}_{s'|s, a}V_{\lambda}^*(s'), \notag \\
\pi_{\lambda}^*(a|s) &=  \max\left\{g_{\phi}\left(\frac{\mu_{\lambda}^*(s)-Q_{\lambda}^*(s, a)}{\lambda} \right), 0\right\}, \label{eq:optimal-pi} \\
V_{\lambda}^*(s) &= \mu_{\lambda}^*(s) - \lambda \sum_{a}\pi_{\lambda}^*(a|s)^2\phi'(\pi_{\lambda}^*(a|s)), \notag
\end{align}
where $g_{\phi}(x)=(f_{\phi}')^{-1}(x)$ is strictly decreasing and $\mu_{\lambda}^*(s)$ is a normalization term so that $\sum_{a \in \mathcal{A}} \pi_{\lambda}^*(a|s) = 1$.
\end{theorem}

Theorem \ref{thm:optimality} shows how the regularization influences the optimality condition. Let $f_{\phi}'(0)\triangleq\liml_{x \to 0+}f_{\phi}'(x)$ for short. From~\eqref{eq:optimal-pi}, it can be shown that the optimal policy $\pi_{\lambda}^*$ assigns zero probability to the actions whose Q-values $Q_{\lambda}^*(s, a)$ are below the threshold $ \mu_{\lambda}^*(s)-\lambda f_{\phi}'(0)$ and assigns positive probability to near optimal actions in proportion to their Q-values (since $g_{\phi}(x)$ is decreasing). 
The threshold involves  $\mu_{\lambda}^*(s)$ and $f_{\phi}'(0)$. $\mu_{\lambda}^*(s)$ can be uniquely solved from the equation obtained by plugging~Eqn.\eqref{eq:optimal-pi} into $\sum_{a \in \mathcal{A}} \pi_{\lambda}^*(a|s) = 1$. Note that the resulting equation only involves $\{Q_{\lambda}^*(s, a)\}_{a \in \AP}$. Thus $\mu_{\lambda}^*(s)$ is actually always a multivariate finite-valued function of $\{Q_{\lambda}^*(s, a)\}_{a \in \AP}$. However, the value $f_{\phi}'(0)$ can be infinity, making the threshold out of function. To see this, if $f_{\phi}'(0)=\infty$, the threshold will be $-\infty$ and all actions will be assigned positive probability in any optimal policy. To characterize the number of zero probability actions, we define a $\delta$-sparse policy as Definition~\ref{def:sparse} shows. It is trivial that $\frac{1}{|\mathcal{A}|}\le\delta\le1$. For instance, a deterministic optimal policy in non-regularized MDP is $\frac{1}{|\mathcal{A}|}$-sparse. 

\begin{definition}
\label{def:sparse}
A given policy $\pi: \mathcal{S} \rightarrow \Delta_{\mathcal{A}}$ is called $\delta$-sparse if it satisfies:
\begin{align}
    \frac{|\{(s,a)\in\mathcal{S}\times\mathcal{A}|\pi(a|s)\not=0\}|}{|\mathcal{S}||\mathcal{A}|}\le\delta.
\end{align}
If $\pi(a|s)>0$ for all $(s, a)\in\mathcal{S}\times\mathcal{A}$, we call it has no sparsity.
\end{definition}

\begin{theorem}
\label{thm:condition}
If $\liml_{x\to 0+}f_{\phi}'(x) = \infty$ \emph{(or $0 \not\in \mathrm{dom} f_{\phi}'$)}, the optimal policy of regularized MDP is not sparse.
\end{theorem}

Theorem~\ref{thm:condition} provides us a criteria to determine whether a regularization could render its corresponding regularized optimal policy the property of sparseness. To facilitate understanding, let us see two examples. When $\phi(x)=-\log(x)$, we have that $\liml_{x\to 0+} f_{\phi}'(x)=\liml_{x\to 0+} -\log(x)-1=\infty$, which implies that the optimal policy of Shannon entropy-regularized MDP does not have sparsity. When $\phi(x)=\frac{k}{q-1}(1- x^{q-1})$ for $q>1$ and $\lambda$ is small enough, the corresponding optimal policy can be spare if $\lambda$ is small enough because $\liml_{x\to 0+} f_{\phi}'(x)= \frac{k}{q-1}$. What's more, the sparseness property of Tsallis entropy still keeps for $1<q < \infty$ and small $\lambda$, which is empirically proved true in \cite{lee2019tsallis}. Additionally, when $0 <q <1$, the Tsallis entropy could no longer lead to sparseness due to 
$\liml_{x\to 0+} f_{\phi}'(x)= \liml_{x\to 0+} \frac{k}{1-q}(q x^{q-1}-1)=\infty$.

The sparseness property is first discussed in \cite{lee2018sparse} which shows the Tsallis entropy with  $k=\frac{1}{2}$ and $q=2$ can devise a sparse MDP. However, we point out that any $\phi(\cdot)$ such that $f_{\phi}'(0) < \infty$ with a proper $\lambda$ leads to a \textit{sparse} MDP. Let $\mathcal{G} \subseteq \mathcal{F}$ be the set that satisfies $\forall \phi \in \mathcal{G}, 0\in\mathrm{dom}f_{\phi}'$. The positive combination of any two regularizers belonging to $\mathcal{G}$ still belongs to $\mathcal{G}$.
\begin{proposition} 
\label{prop:phi-operation-}
Given  $\phi_1, \phi_2 \in \mathcal{G}$, we have that $\alpha\phi_1+\beta\phi_2 \in \mathcal{G}$ for all $\alpha, \beta \ge 0$. However, if $\phi_1 \in \mathcal{G}$ but $\phi_2 \notin \mathcal{G}$, $\alpha\phi_1+\beta\phi_2 \notin \mathcal{G}$ for any positive $\beta$.
\end{proposition}

It is easily checked that the two families (i.e., exponential functions and trigonometric functions) given in Section~\ref{sec:assumption} can also induce a sparse MDP with a proper $\lambda$. For convenience, we  prefer to term the  function $\phi(x)=\frac{k}{q-1}(1{-}x^{q-1})$ that defines the Tsallis entropy as a \textit{polynomial} function, because when $q>1$ it is a polynomial function of degree $q{-}1$. Additionally, these three \textit{basic} families of functions  could be combined to construct  more complex regularizers  (Propositions~\ref{prop:phi-operation}, ~\ref{prop:phi-operation-}).

\textbf{Control the Sparsity of Optimal Policy.}
Theorem~\ref{thm:condition} shows $0\in\mathrm{dom}f_{\phi}'$ is  necessary but not sufficient for that the optimal policy $\pi_{\lambda}^{*}$ is sparse.
The sparsity of optimal policy is also controlled by $\lambda$. Theorem~\ref{thm:control} shows how the sparsity of optimal policy can be controlled by $\lambda$ when $f'_{\phi}(0)<\infty$. The proof is detailed in Appendix~\ref{ap:control-sparse}.


\begin{theorem}
\label{thm:control}
Let $Q_{\lambda}^*(s, a)$ and $\mu_{\lambda}^*(s)$ be defined in Theorem~\ref{thm:optimality} and assume $f'_{\phi}(0)<\infty$. When $\lambda\to0$, the sparsity of the optimal policy $\pi_{\lambda}^{*}$ shrinks to $\delta=\frac{1}{|\mathcal{A}|}$. When $\lambda\to\infty$, the optimal policy has no sparsity. More specifically,  $\pi_{\lambda}^{*}(a|s)\to\frac{1}{|\mathcal{A}|}$ for all $(s, a)\in\mathcal{S}\times\mathcal{A}$ as $\lambda\to\infty$.
\end{theorem}

\subsection{Properties of regularized MDPs}

In this section, we present some properties of regularized MDPs. We first prove the uniqueness of the optimal policy and value. Next, we give the bound of the performance error between $\pi_{\lambda}^*$ (the optimal policy obtained by a  regularized MDP) and $\pi^*$ (the policy obtained by the original MDP). In the proofs of this section, we need an additional assumption for regularizers. Assumption~\ref{assump:control-entropy} is quite weak. All the functions introduced in Section~\ref{sec:assumption} satisfy it.

\begin{assumption}
\label{assump:control-entropy}
The regularizer $\phi(\cdot)$ satisfies $ f_{\phi}(0) \triangleq \liml_{x \to 0^+} x\phi(x)=0$.
\end{assumption} 

\textbf{Generic Bellman Operator $\mathcal{T}_{\lambda}$}
We define a new operator $\mathcal{T}_{\lambda}$ for  regularized MDPs, which defines a smoothed maximum. Given one state $s \in \mathcal{S}$ and current value function $V_{\lambda}$,  $\mathcal{T}_{\lambda}$ is defined as
\begin{equation}\label{prob:current}
 \mathcal{T}_{\lambda} V_{\lambda}(s) \triangleq \max_{\pi} \sum_{a}\pi(a|s)\left[Q_{\lambda}(s, a) {+} \lambda \phi(\pi(a|s))\right],
\end{equation}
where $Q_{\lambda}(s, a) = r(s, a) + \gamma \mathbb{E}_{s'|s, a}V_{\lambda}(s')$ is Q-value function derived from one-step foreseeing according to $V_{\lambda}$. By definition, $\mathcal{T}_{\lambda}$ maps $V_{\lambda}(s)$ to its possible highest value which considers both future discounted rewards and regularization term. We provide simple upper and lower bounds of $\mathcal{T}_{\lambda}$ w.r.t. $\mathcal{T}$, i.e., 

\begin{theorem}
\label{thm: max}
Under Assumptions \ref{assump:1} and \ref{assump:control-entropy}, for any value function $V$ and $s \in \SP$, we have
\begin{equation}
\label{eq:operator-bound}
\TP V(s) \le  \mathcal{T}_{\lambda} V(s) \le \TP V(s)  + \lambda\phi(\frac{1}{|\mathcal{A}|}). 
\end{equation}
\end{theorem}

The bound~\eqref{eq:operator-bound} shows that $\mathcal{T}_{\lambda}$ is a bounded and smooth approximation of $\TP$. When $\lambda = 0$, $\mathcal{T}_{\lambda}$ degenerates to the Bellman operator $\mathcal{T}$. Moreover, it can be proved that $\mathcal{T}_{\lambda}$ is a $\gamma$-contraction. By the Banach fixed point theorem \cite{smart1980fixed}, $V^*_{\lambda}$, the fixed point of $\mathcal{T}_{\lambda}$, is unique.  As a result of Theorem~\ref{thm:optimality}, $Q_{\lambda}^*$ and $\pi_{\lambda}^*$ are both unique. We formally state the conclusion and give the proof in Appendix~\ref{ap:operator}.

\textbf{Performance Error Between $V_{\lambda}^*$ and $V^*$}
In general, $V^* \neq V_{\lambda}^*$. But their difference is controlled by both $\lambda$ and $\phi(\cdot)$. The behavior of $\phi(x)$ around the origin represents the regularization ability of $\phi(x)$. Theorem~\ref{thm: difference} shows that when $|\mathcal{A}|$ is quite large (which means $\phi(\frac{1}{|\mathcal{A}|}$) is close to $\phi(0)$ due to its continuity), the closeness of $\phi(0)$ to 0 also determines their difference. As a result, the Tsallis entropy regularized MDPs have always tighter error bounds than the Shannon entropy regularized MDPs, because the value at the origin of the concave function $\frac{k}{q-1}(1-x^{q-1})(q > 1)$ is much lower than that of  $-\log x$, both function satisfying in Assumption ~\ref{assump:control-entropy}. Our theory incorporates the result of \citet{lee2018sparse, lee2019tsallis} which shows a similar performance error for (general) Tsallis entropy RL. The proof of Theorem~\ref{thm: difference} is detailed in Appendix~\ref{ap:error}.

\begin{theorem}
\label{thm: difference}
Under Assumptions \ref{assump:1} and \ref{assump:control-entropy}, the error between $V_{\lambda}^*$ and $V^*$ can be bounded as \[\|V_{\lambda}^*-V^*\|_{\infty} \le \frac{\lambda}{1-\gamma}\phi(\frac{1}{|\mathcal{A}|}).\]
\end{theorem}

\section{Regularized Actor-Critic}


To solve the problem~\eqref{prob:genera-regularized-rl} in complex environments, we propose an off-policy algorithm \textit{Regularized Actor-Critic} (RAC), which alternates between policy evaluation and policy improvement. In practice, we apply neural networks to parameterize the Q-value and policy to increase expressive power. In particular, we model the regularized Q-value function $Q_{\theta}(s, a)$ and a tractable policy $\pi_{\psi}(a|s)$. We use Adam \cite{kingma2014adam} to optimize $\psi, \theta$. Actually, RAC is created by consulting the previous work SAC \cite{haarnoja2018soft, haarnoja2018soft1} and making some necessary changes so that it is able to be agnostic to the form of regularization.

The goal for training regularized Q-value parameters is to minimize the  general Bellman residual:
\begin{equation}
\label{eq:loss-Q}
J_{Q}(\theta) = \frac{1}{2}\hat{\EB}_{\DP}(Q_{\theta}(s_t, a_t) - y)^2,
\end{equation}
where $\DP$ is the replay buffer used to eliminate the correlation of sampled trajectory data and $y$ is the target function defined as follows
\begin{equation*}
\label{eq:target}
y = r(s_t, a_t) {+} \gamma\left[ Q_{\bar{\theta}}(s_{t+1}, a_{t+1})  {+} \lambda \phi(\pi_{\psi}(a_{t+1}|s_{t+1})) \right].
\end{equation*}
The target involves a target regularized Q-value function with parameters $\bar{\theta}$ that are updated in a moving average fashion, which can stabilize the training process \cite{mnih2015human, haarnoja2018soft}. Thus the  gradient of $J_{Q}(\theta)$ w.r.t. $\theta$ can be estimated by
\begin{equation*}
\label{eq:grad-Q}
\hat{\nabla}J_{Q}(\theta) = \hat{\EB}_{\DP}  \nabla_{\theta}Q_{\theta}(s_t, a_t)\left(Q_{\theta}(s_t, a_t) {-} y \right).
\end{equation*}

For training policy parameters, we minimize the negative total reward:
\begin{equation}
\label{eq:loss-policy}
J_{\pi}(\psi) = \hat{\EB}_{\DP}\left[ \mathbb{E}_{a\sim\pi_{\psi}(\cdot|s_t)}\left[-\lambda \phi(\pi_{\psi}(a|s_t))  - Q_{\theta}(s_t, \phi(\pi_{\psi}(a|s_t)) )\right] \right].
\end{equation}

RAC is formally described in Algorithm~\ref{alg:rac}. The method alternates between data collection and parameter updating. Trajectory data is collected by executing the current policy in environments and then stored in a replay buffer. Parameters of the function approximators are updated by descending along the stochastic gradients computed from the batch sampled from that replay buffer. The method makes use of two Q-functions to overcome the positive bias incurred by overestimation of Q-value, which is known to yield a poor performance \cite{hasselt2010double, fujimoto2018addressing}. Specifically, these two Q-functions are parametrized by different parameters $\theta_i$ and are independently trained to minimize $J_{Q}(\theta_i)$. The minimum of these two Q-functions is used to compute the target value $y$ which is involved in the computation of $\hat{\nabla}J_{Q}(\theta)$ and $\Hat{\nabla}J_{\pi}(\psi)$. 

\begin{wrapfigure}{R}{0.5\textwidth}
\begin{minipage}{0.5\textwidth}
\vspace{-10pt}
\begin{algorithm}[H]
    \caption{Regularized Actor-Critic (RAC)}
    \label{alg:rac}
    \begin{algorithmic}
        \STATE {\bfseries Input:} $\theta_1, \theta_2, \psi$
        \STATE {\bfseries Initialization:} $\barone \leftarrow \theta_{1}, \barone \leftarrow \theta_{2}, \DP \leftarrow \emptyset$
        \FOR{ each iteration}
                \FOR{ each environment step}
                \STATE sample action, $a_t \sim \pi_{\psi}(\cdot|s_t)$
                \STATE receive reward $r_t \sim r_t(s_t, a_t)$
                \STATE receive next state $s_{t+1}$ from environment
                \STATE $\DP \leftarrow \DP \bigcup \{ (s_t, a_t, r_t, s_{t+1}) \}$
                \ENDFOR
                \FOR{ each gradient step}
                \STATE $\theta_{i} \leftarrow\theta_{i} - \eta_{Q} \hat{\nabla}J_{Q}(\theta_i)$ for $i \in \{1, 2\}$
                \STATE $\psi \leftarrow \psi - \eta_{\pi} \hat{\nabla}J_{\pi}(\psi)$
                \STATE $\bar{\theta}_{i} \leftarrow \tau \theta_{i} + (1-\tau) \bar{\theta}_{i} $ for $i \in \{1, 2\}$
                \ENDFOR  
        \ENDFOR
    \STATE {\bfseries Output:} $\theta_1, \theta_2, \psi$
    \end{algorithmic}
\end{algorithm}
\vspace{-50pt}
\end{minipage}
\end{wrapfigure}

\section{Experiments}
\label{sec: experiments}
We investigate the performance of different regularizers among diverse environments. We first test basic and combined regularizers in two numerical environments. Then we test basic regularizers in Atari discrete problems. In the end, we explore the possible application in Mujoco control environments.

\subsection{Numerical results}
The two discrete numerical environments we consider include a simple random generated MDP ($S=50, A=10$) and a Gridworld environment ($S=81, A=4$). Refer to Appendix~\ref{ap:exp_discrete} for more detail settings.

\textbf{Regularizers.}
Four \textit{basic} regularizers include \texttt{shannon} $(-\log x)$, \texttt{tsallis} $(\frac{1}{2}(1-x))$, \texttt{cos} $(\cos(\frac{\pi}{2}x))$ and \texttt{exp} $(\exp(1)-\exp(x))$. Proposition~\ref{prop:phi-operation} and~\ref{prop:phi-operation-} allow three combined regularizers: (1) \texttt{min}: the minimum of \texttt{tsallis} and \texttt{shannon}, i.e., $\min\{-\log(x), 2(1-x)\}$, (2) \texttt{poly}: the positive addition of two polynomial functions, i.e., $\frac{1}{2}(1-x) + (1-x^2)$, and (3) \texttt{mix}: the positive addition of \texttt{tsallis} and \texttt{shannon}, i.e., $-\log(x) + \frac{1}{2}(1-x)$.

\textbf{Sparsity and Convergence.}
From (a)(b) in Figure~\ref{fig:numerical}, when $\lambda$ is extremely large, $\delta = 1$ for all regularizers. (c) shows how the probability of each action in the optimal policy at a given state varies with $\lambda$ (one curve represents one action). These results validate the Theorem~\ref{thm:control}. A reasonable explanation is that large $\lambda$ reduces the importance of discounted reward sum and makes $H_{\phi}(\pi)$ dominate the loss, which forces the optimal policy to put probability mass evenly on all actions in order to maximize $H_{\phi}(\pi)$. We regard the ability to defend the tendency towards converging to a uniform distribution as sparseness power. From our additional experiments in Appendix~\ref{ap:exp}, \texttt{cos} has the strongest sparseness power. (d) shows the convergence speed of RPI on different regularizers. It also shows that $\|V^{*}-V^{\pi_{\lambda}^{*}}\|_{\infty}$ is bounded as Theorem~\ref{thm: max} states. 

\begin{figure*}[ht]
    \vspace{-10pt}
    \centering
    \subfloat[Random MDP]{
    \includegraphics[width=0.25\textwidth] {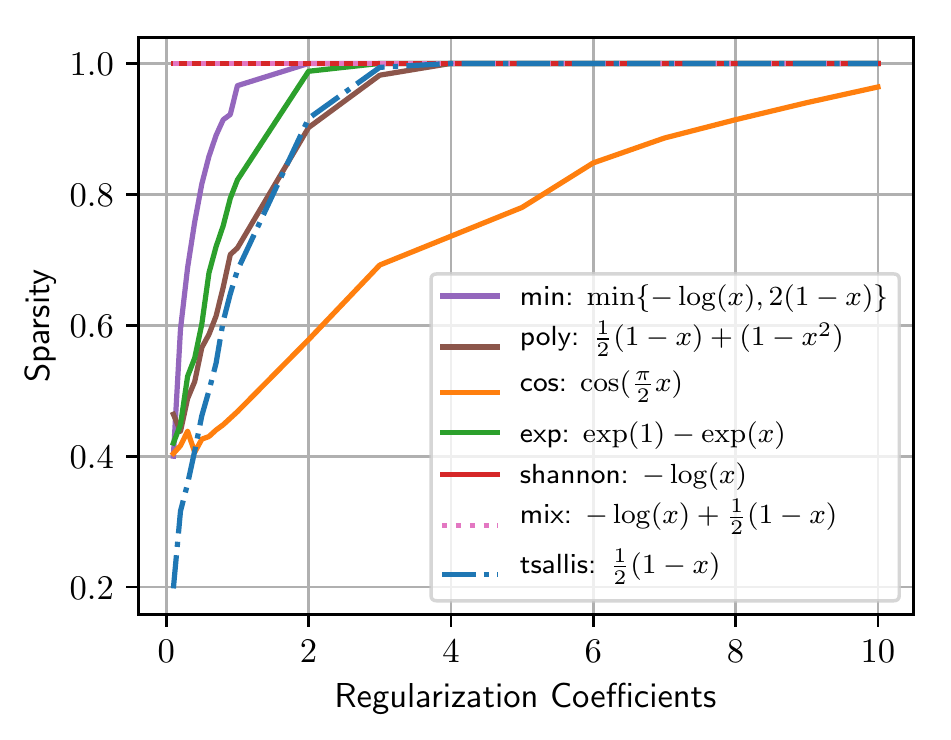}}
    \hspace{-.1in} 
    \subfloat[Gridworld]{
    \includegraphics[width=0.25\textwidth]{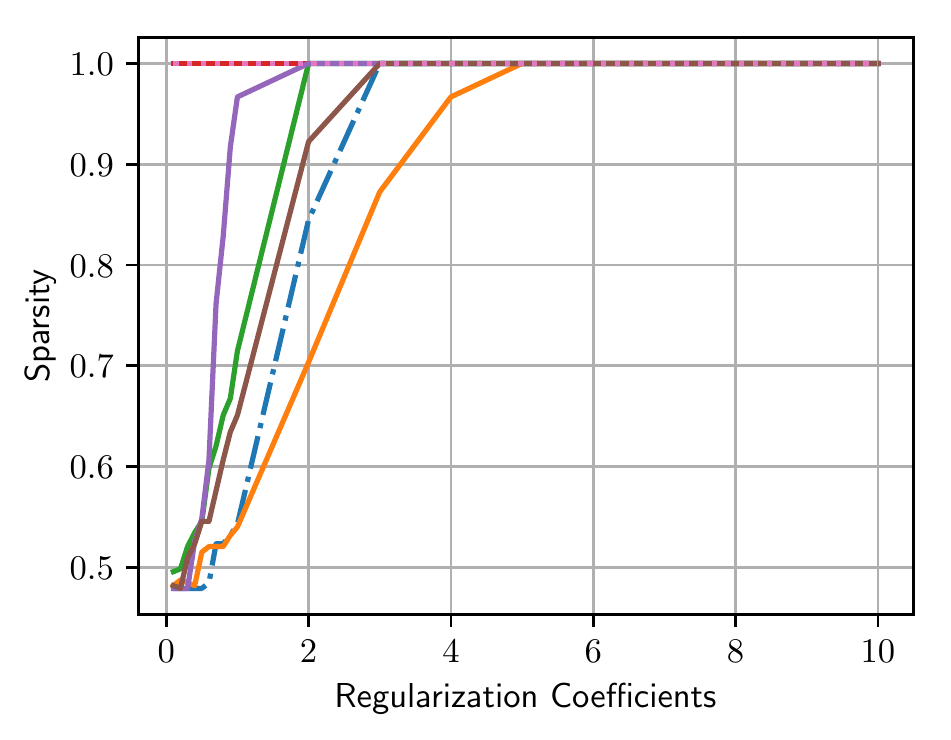}}    \hspace{-.1in}    
    \subfloat[\texttt{cos}: $\cos(\frac{\pi}{2}x)$]{
    \includegraphics[width=0.25\textwidth]{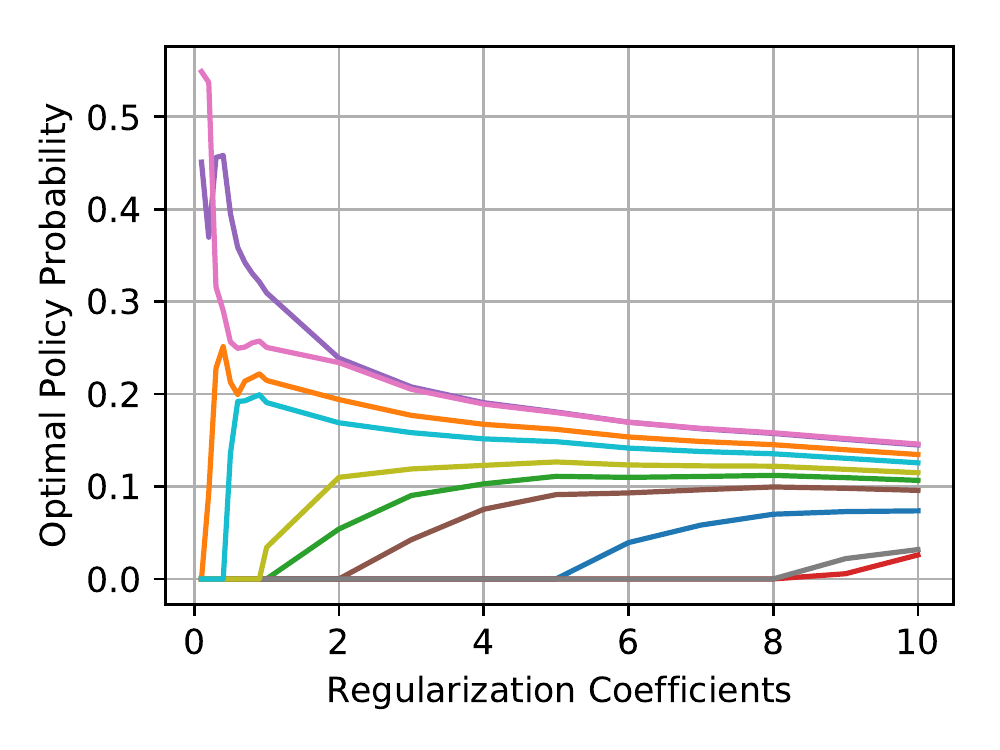}}    
    \hspace{-.1in}
    \subfloat[Error of different regularizers on Random MDP]{
    \includegraphics[width=0.25\textwidth]{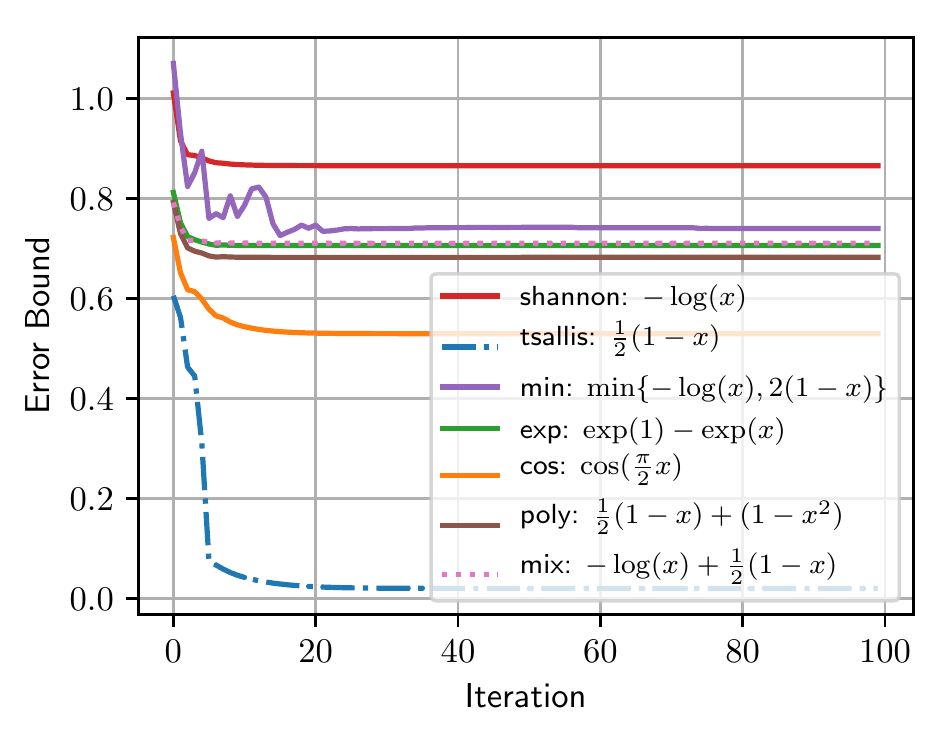}}\\
    \caption{ (a) and (b) show the results of the sparsity $\delta$ of optimal policies on Random MDP and Gridworld. (c) shows the changing process of the probability of each action in optimal policy regularized by $\cos(\frac{\pi}{2}x)$ on Random MDP. (d) shows the $\ell_{\infty}$-error between $V^{*}$ and $V^{\pi_{\lambda}^{*}}$.}
    \label{fig:numerical}
    \vspace{-10pt}
\end{figure*}

\subsection{Atari results}
\textbf{Regularizers.} We test four basic regularizers across four discrete control tasks from OpenAI Gym benchmark \cite{brockman2016openai}. All the training details are in Appendix~\ref{ap:exp_arari}.

\textbf{Performance.} Figure~\ref{atari_train} shows the score during training for RAC with four regularization forms with best performance over $\lambda=\{0.01, 0.1, 1.0\}$. Except Breakout, \texttt{Shannon} performs worse than other three regularizers. \texttt{Cos} performs best in Alien and Seaquest while \texttt{tsallis} performs best in Boxing and \texttt{exp} performs quite normally. Appendix~\ref{ap:exp_arari} gives all the results with different $\lambda$ and sensitive analysis. In general, \texttt{shannon} is the most insensitive among others.

\begin{figure*}[ht]
    \vspace{-10pt}
    \centering
    \subfloat[Alien]{
    \includegraphics[width=0.25\textwidth] {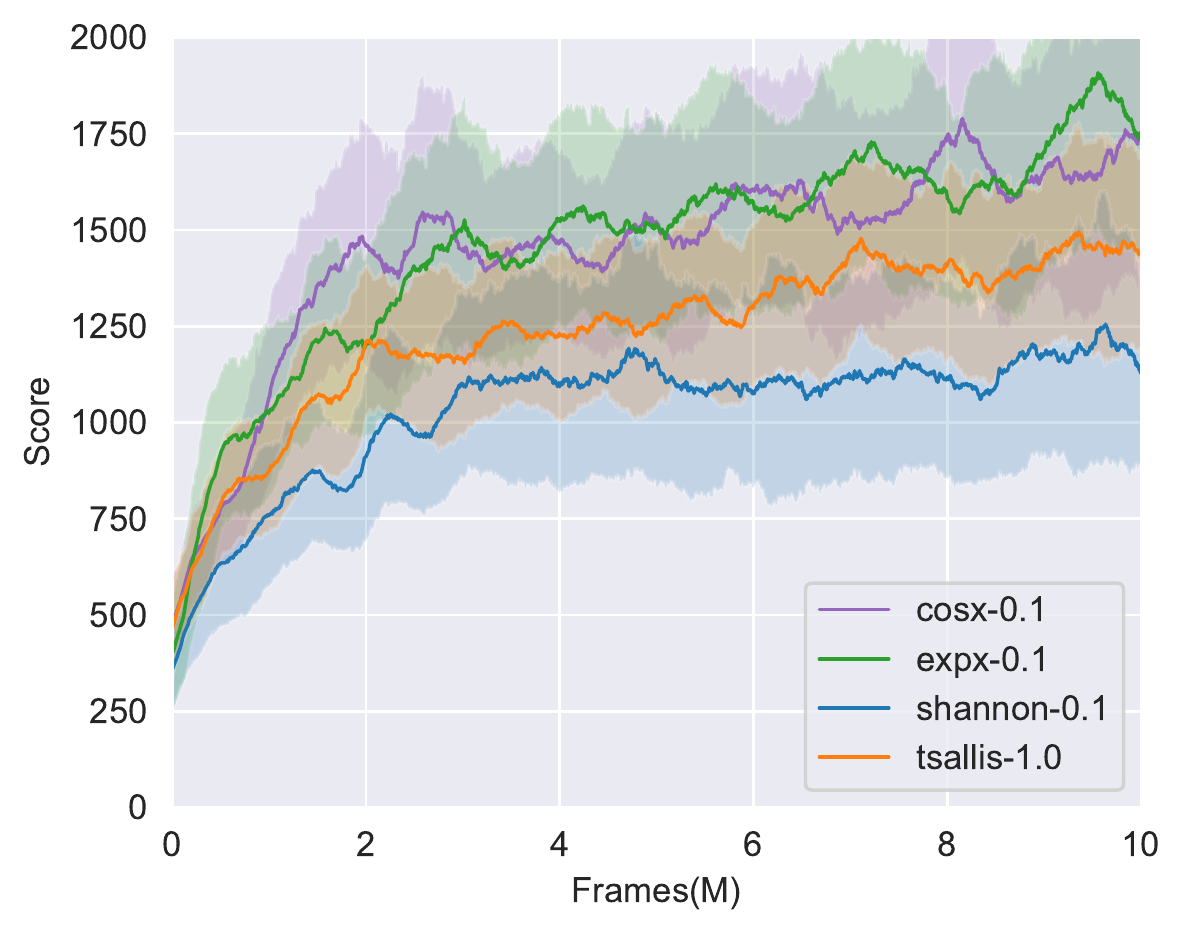}}
    \hspace{-.1in}
    \subfloat[Boxing]{
    \includegraphics[width=0.25\textwidth] {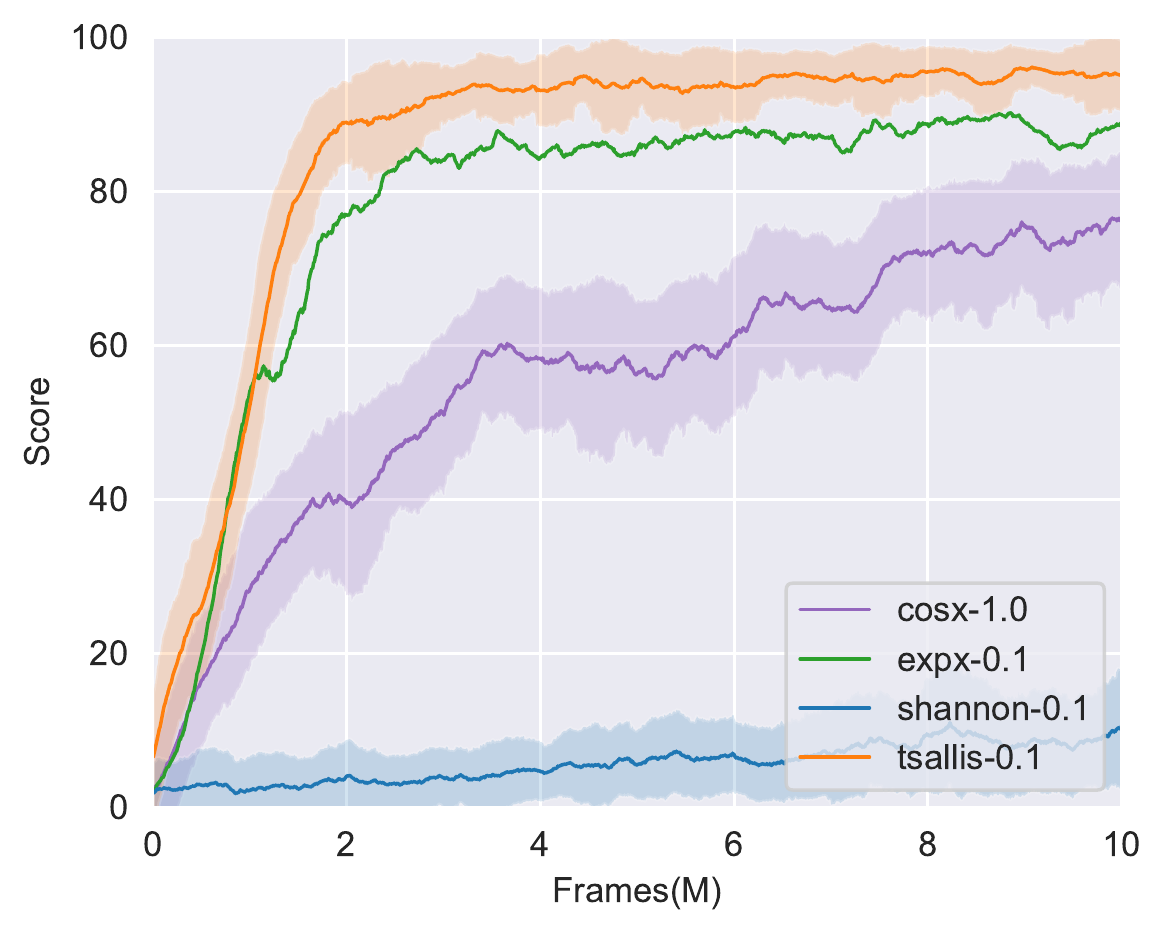}}
    \hspace{-.1in}
    \subfloat[Breakout]{
    \includegraphics[width=0.25\textwidth] {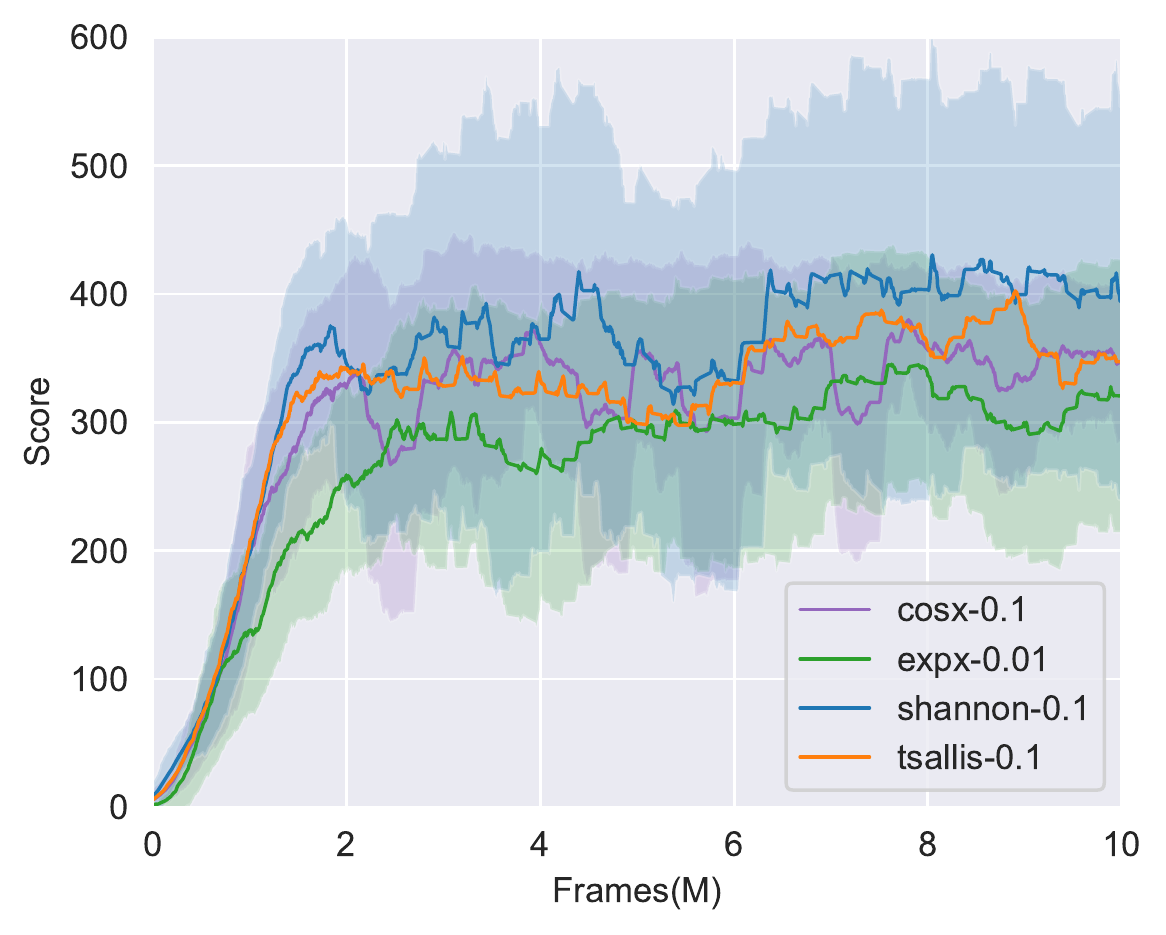}}
    \hspace{-.1in} 
    \subfloat[Seaquest]{
    \includegraphics[width=0.25\textwidth] {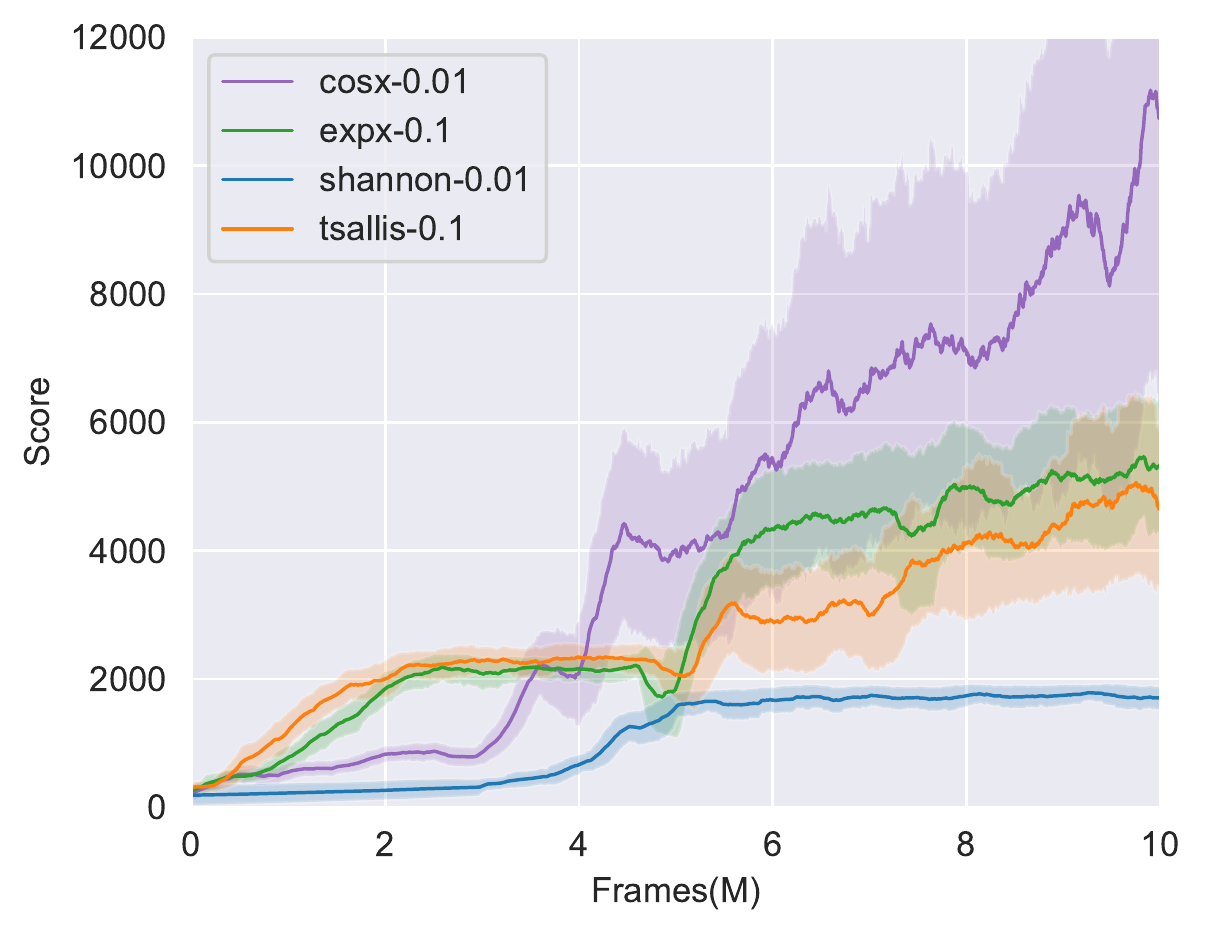}}
    \caption{Training curves on Atari games. Each entry in the legend is named with the rule $\texttt{the regularization form} + \lambda$. The score is smoothed with 100 windows while the shaded area is the one standard deviation.}
    \label{atari_train}
    \vspace{-10pt}
\end{figure*}

\subsection{Mujoco results}

\textbf{Regularizers.} We explore basic regularizers across four continuous control tasks from OpenAI Gym benchmark \cite{brockman2016openai} with the MuJoCo simulator \cite{todorov2012mujoco}. Unfortunately \texttt{cos} is quite unstable and prone to gradient exploding problems in deep RL training process. We speculate it instableness roots in
numerical issues where the probability density function often diverges into infinity. What's more, the periodicity of $\cos(\frac{\pi}{2}x)$ makes the gradients vacillate and the algorithm hard to converge. All the details of the following experiments are given in Appendix~\ref{ap:exp_mujoco}. 

\begin{figure*}[ht]
    \vspace{-10pt}
    \centering
    \subfloat[Ant-v2]{
    \includegraphics[width=0.25\textwidth] {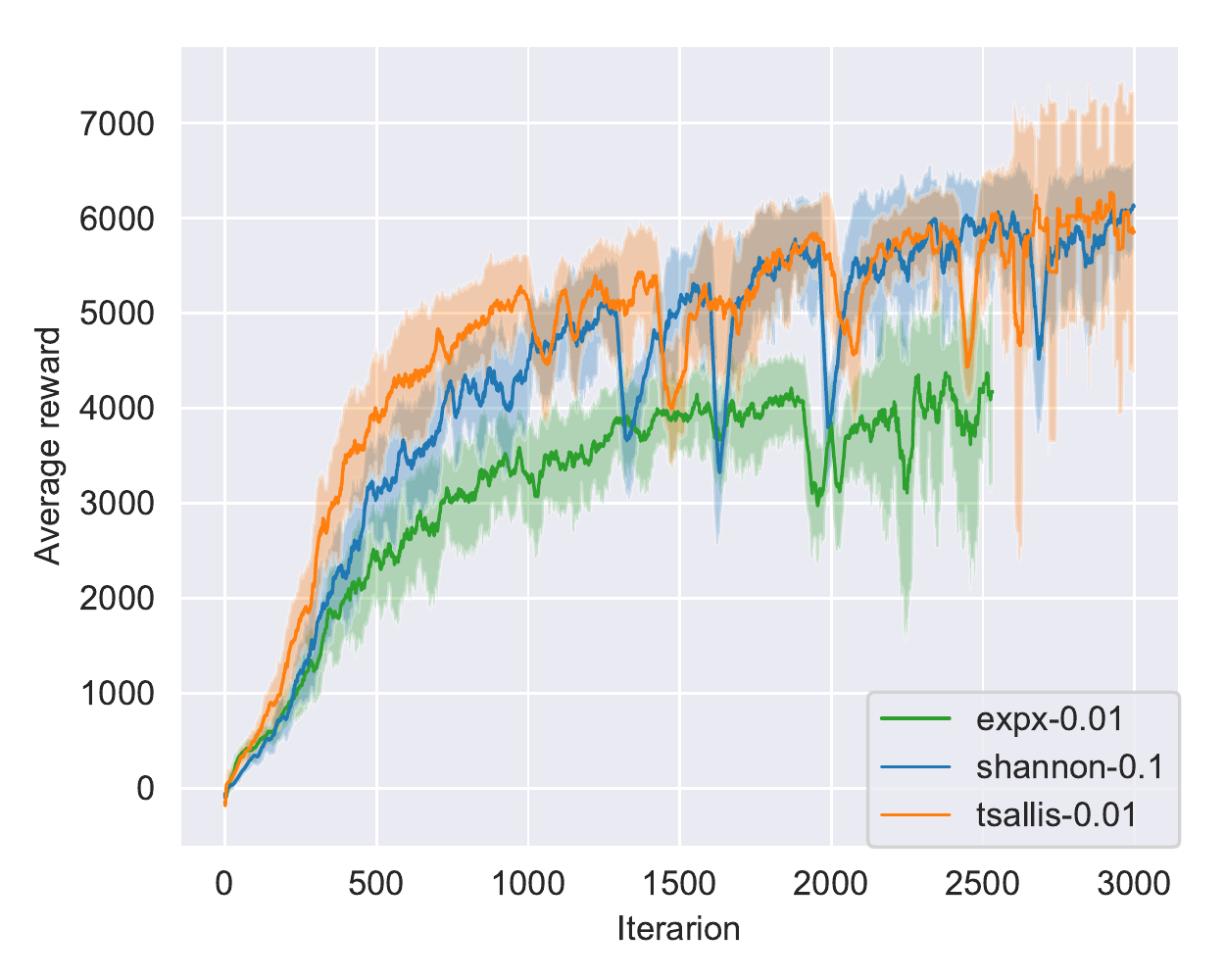}}
    \hspace{-.1in}
    \subfloat[Walker-v2]{
    \includegraphics[width=0.25\textwidth] {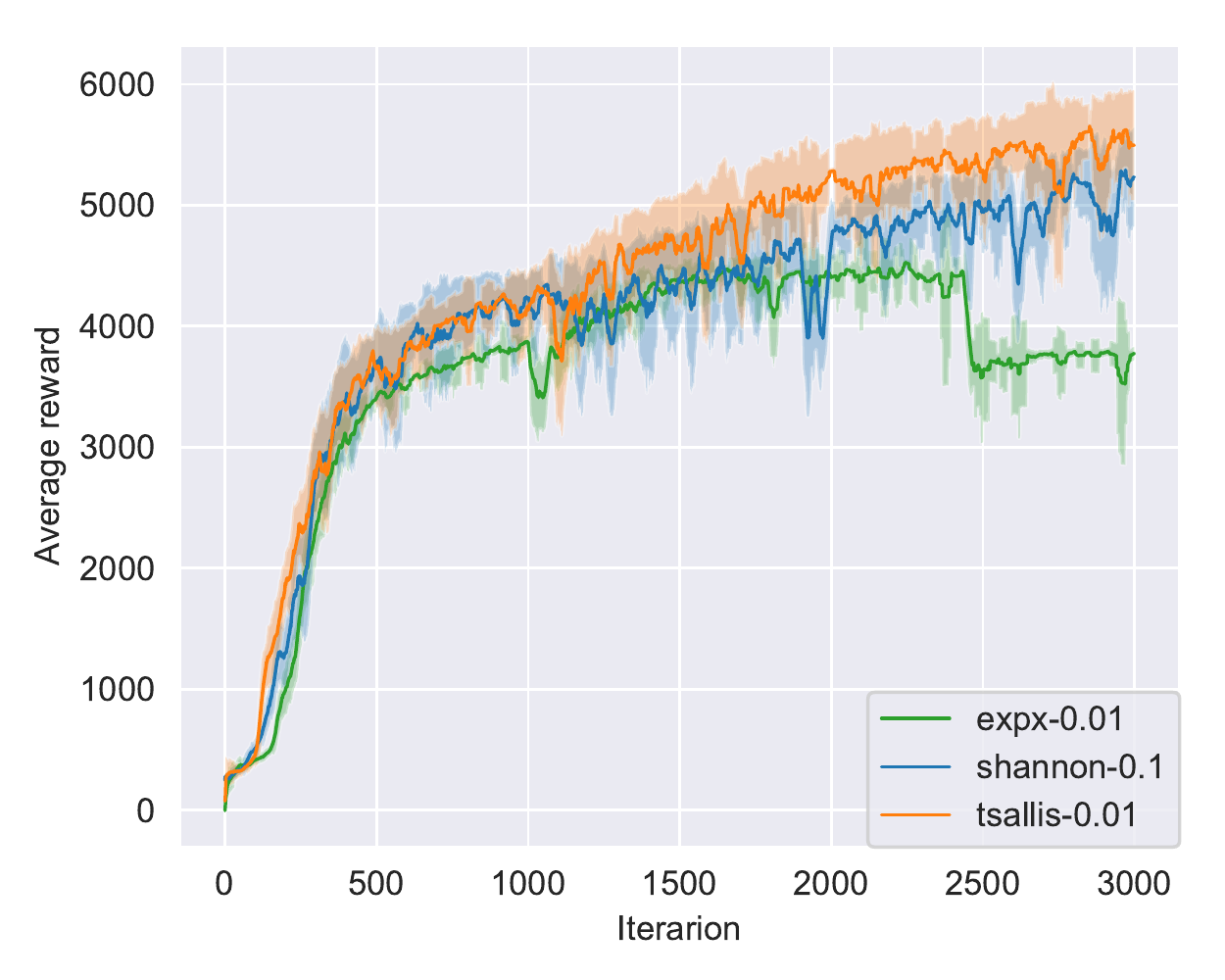}}
    \hspace{-.1in}
    \subfloat[Hopper-v2]{
    \includegraphics[width=0.25\textwidth] {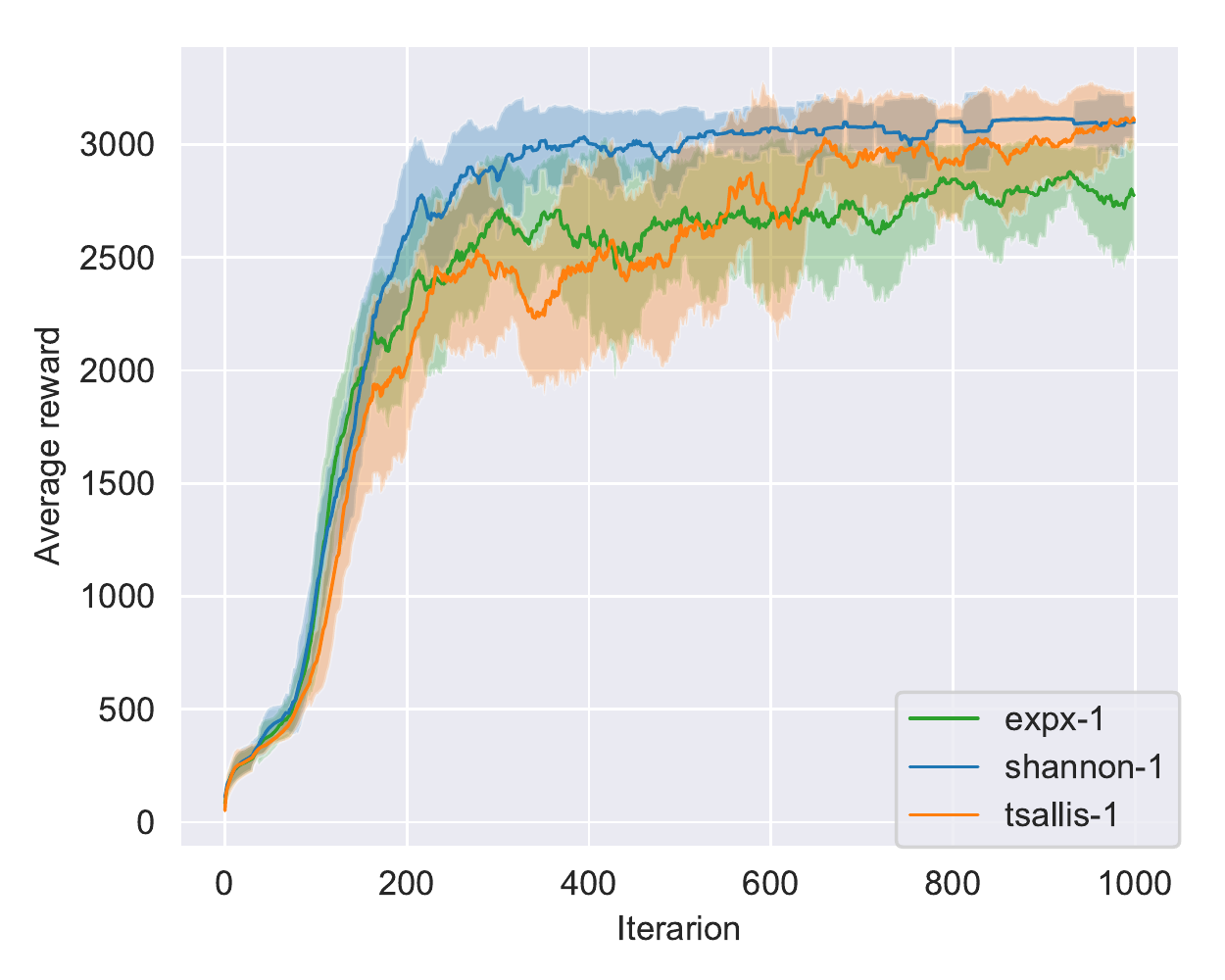}}
    \hspace{-.1in} 
    \subfloat[HalfCheetah-v2]{
    \includegraphics[width=0.25\textwidth] {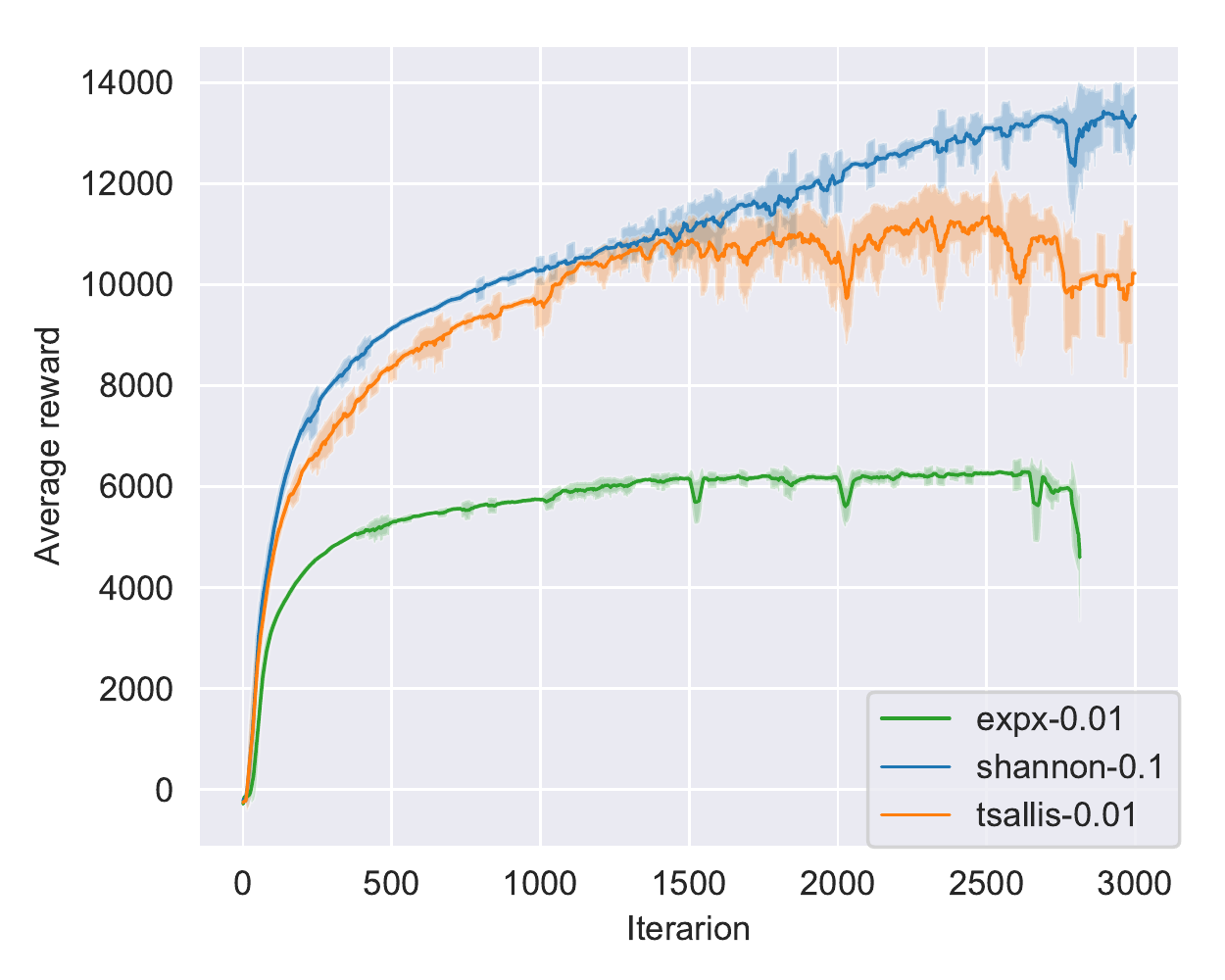}}
    \caption{Training curves on continuous control benchmarks. Each curve is the average of four experiments with different seeds. Each entry in the legend is named with the rule $\texttt{the regularization form} + \lambda$. The score is smoothed with 30 windows while the shaded area is the one standard deviation.}
    \label{eval_rl}
    \vspace{0pt}
\end{figure*}

\textbf{Performance.} Figure~\ref{eval_rl} shows the total average return of rollouts during training for RAC with three regularization forms and different regularization coefficients ($[0.01, 0.1, 1]$). For each curve, we train four different instances with different random seeds. \texttt{Tsallis} performs steadily better than \texttt{shannon} given the same regularization coefficient $\lambda$. \texttt{Tsallis} is also more stable since its shaded area is thinner than \texttt{shannon}. \texttt{Exp} performs almost as good as \texttt{tsallis} in Ant-v2 and Hopper-v2 but performs badly in the rest two environments. From the sensitivity analysis provided in Appendix~\ref{ap:exp_mujoco}, \texttt{tsallis} is less sensitive to $\lambda$ than \texttt{cos} and \texttt{shannon}.

\section{Conclusion}

In this paper, we have proposed a unified framework for regularized reinforcement learning, which includes entropy-regularized RL as a special case. Under this framework, the regularization function characterizes the optimal policy and value of the corresponding regularized MDPs. We have shown there are many regularization functions that can lead to a sparse but multi-modal optimal policy such as trigonometric and exponential functions. We have specified a necessary and sufficient condition for these regularization functions that could lead to sparse optimal policies and how the sparsity is controlled with $\lambda$. We have presented the logical and mathematical foundations of these properties and also conducted the experimental results.

\section*{Acknowledgements}

This work is sponsored by the Key Project of MOST of China (No. 2018AAA0101000), by Beijing Municipal Commission of Science and Technology under Grant No. 181100008918005, and by Beijing Academy of Artificial Intelligence (BAAI).

\bibliography{refer}
\bibliographystyle{plainnat}

\newpage

\appendix
\begin{appendix}
\onecolumn
\begin{center}
    {\huge \textbf{Appendix}}
\end{center}

\section{Related Work}
\paragraph{Regularization in RL.}

The first class aims to control the complexity of value function approximation. The use of function approximation makes it possible to model value (or Q-value) function when the state space is large or even infinite. The main regularization form is $L_2$ or $L_1$ regularization. For example, \cite{massoud2009regularized, farahmand2009regularized} uses $L_2$ regularization to control the complexity of fitting value (or Q-value) functions. \cite{kolter2009regularization, johns2010linear} uses $L_1$ regularization for sparse feature selection.

The second class aims to capture the geometry of parameter spaces better and confine the information loss of policy search \cite{peters2010relative}. A lot of works propose to constraint the updated policy $\pin$ so that it is \textit{close} to the old one $\pio$ in some sense. \cite{peters2010relative, schulman2015trust, martens2015optimizing, schulman2017proximal, liu2017stein} use the Kullback-Leibler (KL) divergence as the measure for closeness and \cite{belousov2017f} considers a more general class of f-divergences. 

The third class aims to modify the original MDP to a more tractable one. One considers the case the transition probabilities can be rescaled \cite{todorov2007linearly}. Others add a policy-related regularization term to the rewards, where entropy-regularized RL belongs. \cite{o2016combining, nachum2017bridging, haarnoja2017reinforcement, haarnoja2018soft} consider using the Shannon entropy, which is shown to improve both exploration and robustness. An MDP with Shannon entropy maximization is termed as \textit{soft MDP} where the hard max operator is replaced by a softmax \cite{asadi2017alternative}. However, the optimal policy in soft MDPs put probability mass on all actions, implying some significantly unimportant actions would be executed. To fix this problem, \cite{grau2018soft} proposes to dynamically learn a prior that weights the importance of actions by using mutual information. Alternatively, \cite{nachum2018path, lee2018sparse} replace Shannon entropy with Tsallis entropy, since a special case ($q=2$ in our notation) of Tsallis entropy can devise a sparse optimal policy \cite{lee2018sparse}. Recently, \cite{lee2019tsallis} analyzes a more general Tsallis entropy family with an additional real-valued parameter (i.e., $q$ mentioned above), called an \textit{entropic index}, which is able to control the exploration tendency. \cite{geist2019theory} considers a more general class of regularized MDP where any strongly concave function replaces the entropy-like regularization term.

To address the issues discussed in the introduction (i.e., to obtain a sparse but multi-modal optimal policy), only the regularization in the third class could work. However, they either focus on entropy regularization or consider too large function, the former ignoring various regularization forms in convex optimization and the latter having no implications for the choice of regularization. Thus we are motivated to propose a unified framework for regularized RL which extends current entropy-regularized RL and provides enough practical guidance.

\paragraph{Optimization for Entropy-regularized MDPs.}
In the literature, there are many algorithms to solve entropy-regularized MDP problems. Similarly, these methods can be modified to solve regularized MDPs since the regularization we proposed is an extension of the traditional entropy. 

\cite{haarnoja2017reinforcement, lee2018sparse} consider the general modified value iteration approach. They repeatedly solve greedily the target regularized Q-values and updates the Q-value function in a Q-learning-like pattern. \cite{schulman2017equivalence} discussed the equivalence between policy gradients and Q-learning where the entropy regularizer is Shannon entropy.  \cite{haarnoja2018soft} adopted actor-critic methods to solve the Shannon regularized MDP in an off-policy fashion and achieves the state-of-the-art performance in continuous control tasks in RL. \cite{lee2019tsallis} proposes TAC, a variant of SAC, by replacing Shannon entropy with general Tsallis entropy. \cite{nachum2017bridging} point out there exists a path consistency equation which only the (near) optimal value and policy satisfy and propose to minimize the residual of that equation by simultaneously updating value and policy functions. This method is called as \textit{Path Consistency Learning}(PCL). \cite{nachum2017trust, nachum2018path, dai2018sbeed} share the same methodology with PCL for Shannon entropy.

\cite{neu2017unified} provides a unified view of entropy-regularized MDPs which enables us to formalize most entropy-regularized RL algorithms as approximate variants of Mirror Descent or Dual Averaging. \cite{geist2019theory} extends this result such that a broader class of regularizers is considered. They propose a modified policy iteration and give error propagation analyses for many existing algorithmic schemes.

\section{Proof for Optimality Condition of Regulazied MDPs}\label{ap: optimality}
In this section, we give the detail proof for Theorem ~\ref{thm:optimality}, which states the optimality condition of regularized MDPs. 
The proof follows from the Karush-Kuhn-Tucker (KKT) conditions where the derivative of a Lagrangian objective function with respect to policy $\pi(a|s)$ is set zero. 
Hence, our main theory is necessary and sufficient.

\begin{proof}{\textbf{for Theorem ~\ref{thm:optimality}}}
The Lagrangian function of ~\eqref{prob:genera-regularized-rl} obtained by the optimal policy is written as follows
\begin{align*}
L(\pi, \beta, \mu) 
&= \sum_{s} d_{\pi}(s)  \sum_{a}\pi(a|s)   \left( Q_{\lambda}^{*}(s, a) + \lambda \phi(\pi(a|s)) \right)\\
&\quad -  \sum_{s} d_{\pi}(s) [  \mu(s)(\sum_{a} \pi(a|s) -1) + \sum_{a} \beta(a|s) \pi(a|s) ]
\end{align*}
where $d_{\pi}$ is the stationary state distribution of the policy $\pi$, $\mu$ and $\beta$ are Lagrangian multipliers for the equality and inequality constraints respectively. Let $f_{\phi}(x) = x\phi(x)$. Then the KKT condition of~\eqref{prob:genera-regularized-rl} are as follows, for all states and actions
\begin{align}
&0 \le \pi(a|s) \le 1 \ \text{and} \  \sum_{a}\pi(a|s) = 1   \label{eq:primal-feasibility} \\
&0 \le \beta(a|s)   \label{eq:dual-feasibility} \\
&\beta(a|s) \pi(a|s) = 0  \label{eq:complementary-slackness}  \\
&Q_{\lambda}^{*}(s, a) + \lambda f_{\phi}'(\pi(a|s))-\mu(s)+\beta(a|s) = 0   \label{eq:stationarity-condition} 
\end{align}
where~\eqref{eq:primal-feasibility} is the feasibility of the primal problem,~\eqref{eq:dual-feasibility} is the feasibility of the dual problem,~\eqref{eq:complementary-slackness} results from the complementary slackness and~\eqref{eq:stationarity-condition} is the stationarity condition. We eliminate $d_{\pi}(s)$ since we assume all policies induce an irreducible Markov chain. 

Since $f_{\phi}(x) = x\phi(x)$ is a strictly decreasing function due to (4) in Assumption ~\ref{assump:1}, its inverse function $g_{\phi}(x) = (f_{\phi}')^{-1}(x)$ is also strictly decreasing. From~\eqref{eq:stationarity-condition}, we can resolve $\pi(a|s)$ as
\begin{equation}
\pi(a|s) = g_{\phi}\left(\frac{1}{\lambda}(\mu(s) - Q_{\lambda}^{*}(s, a)- \beta(a|s))\right). 
\end{equation}
Fix a state $s$. For any positive action, its corresponding Lagrangian multiplier $\beta(a|s)$ is zero due to the complementary slackness and $Q_{\lambda}^*(s, a) > \mu(s) - \lambda f_{\phi}'(0)$ must hold. For any zero-probability action, its Lagrangian multiplier $\beta(a|s)$ will be set such that $\pi(a|s) = 0$. Note that $\beta(a|s) \ge 0$, thus $Q_{\lambda}^*(s, a) \le \mu(s) - \lambda f_{\phi}'(0)$ must hold in this case. From these observations, $\pi(a|s)$ can be reformulated as
\begin{equation}
\label{eq:pi}
\pi(a|s) = \max \left\{g_{\phi}\left(\frac{1}{\lambda}(\mu(s) - Q_{\lambda}^{*}(s, a))\right), 0\right\}
\end{equation}
By plugging~\eqref{eq:pi} into~\eqref{eq:primal-feasibility}, we obtain an new equation
\begin{equation}
\label{eq:solve-mu}
 \sum_{a}  \max \left\{g_{\phi}\left(\frac{1}{\lambda}(\mu(s) - Q_{\lambda}^{*}(s, a))\right), 0\right\} = 1 
\end{equation}
Lemma~\ref{lem:mu} states that~\eqref{eq:solve-mu} has and only has one solution denoted as $\mu_{\lambda}^*$. Therefore, $\mu_{\lambda}^*$ can be solved uniquely. We defer the proof of Lemma~\ref{lem:mu} later in this section.

Next we aim to obtain the optimal state value $V_{\lambda}^{*}$. It follows that
\begin{align*}
&V_{\lambda}^{*}(s) = \TP_{\lambda} V_{\lambda}^{*}(s) \\
&= \sum_{a} \pi_{\lambda}^{*}(a|s) \left( Q_{\lambda}^{*}(s, a)+ \lambda \phi(  \pi_{\lambda}^{*}(a|s)) \right) \\
&= \sum_{a} \pi_{\lambda}^{*}(a|s) \left( \mu_{\lambda}^*(s) - \lambda \pi_{\lambda}^{*}(a|s) \phi(\pi_{\lambda}^{*}(a|s)) \right)\\
&= \mu_{\lambda}^*(s) - \lambda \sum_{a}\pi_{\lambda}^*(a|s)^2\phi'(\pi_{\lambda}^*(a|s)).
\end{align*}
The first equality follows from the definition of the optimal state value. The second equality holds because $\pi_{\lambda}^*$ maximizes $ \TP_{\lambda} V_{\lambda}^{*}(s)$. The third equality results from plugging~\eqref{eq:stationarity-condition}. 

To summarize, we obtain the optimality condition of regularized MDPs as follows
\begin{align*}
Q_{\lambda}^*(s, a) &= r(s,a) + \gamma \mathbb{E}_{s'|s, a}V_{\lambda}^*(s'), \\
\pi_{\lambda}^*(a|s) &=  \max\left\{g_{\phi}\left(\frac{1}{\lambda}(\mu_{\lambda}^*(s)-Q_{\lambda}^*(s, a))\right), 0\right\},\\
V_{\lambda}^*(s) &= \mu_{\lambda}^*(s) - \lambda \sum_{a}\pi_{\lambda}^*(a|s)^2\phi'(\pi_{\lambda}^*(a|s)),
\end{align*} \\
where $g_{\phi}(x)=(f_{\phi}')^{-1}(x)$ is strictly decreasing and $\mu_{\lambda}^*(s)$ is a normalization term so that $\sum_{a \in \mathcal{A}} \pi_{\lambda}^*(a|s) = 1$.
\end{proof}

\begin{lemma}
\label{lem:mu}For any Q-value function $Q(s, a)$, the equation
\begin{equation}
\label{eq:solve-mu-from-anyQ}
 \sum_{a}  \max \left\{g_{\phi}\left(\frac{1}{\lambda}(\mu(s) - Q(s, a))\right), 0\right\} = 1 
\end{equation}
has and only has one $\mu^{*}$ satisfying it.
\end{lemma}
\begin{proof}
Denote the left hand side of~\eqref{eq:solve-mu-from-anyQ} which is a continuous function of $\mu$ as $h(\mu)$. We first prove that $h(\mu)$ is a strictly decreasing function on $(-\infty, \mu_{\max})$, where $\mu_{\max} = \max_{a}Q(s, a) + \lambda f_{\phi}'(0)$. Let $\Lambda(s, \mu)$ the set of actions such that their maximum term in~\eqref{eq:solve-mu-from-anyQ} is not obtained at 0, i.e., $ \Lambda(s, \mu)=\{ a:  Q(s, a) > \mu(s) - \lambda f_{\phi}'(0) \} $. Then for $\mu_1 < \mu_2 < \mu_{\max}$,  it follows that $\Lambda(s, \mu_2) \subseteq \Lambda(s, \mu_1)$ and
\begin{align*}
h(\mu_1) - h(\mu_2) = \sum_{a \in\Lambda(s, \mu_2)}  \Delta(\mu_1, \mu_2)  +  \sum_{a  \in \Lambda(s, \mu_1) - \Lambda(s, \mu_2)}g_{\phi} \left(\frac{1}{\lambda}(\mu_1(s) - Q(s, a))\right)
\end{align*}
where
\begin{align*}
\Delta(\mu_1, \mu_2) = g_{\phi} \left(\frac{1}{\lambda}(\mu_1(s) - Q(s, a))\right)  -g_{\phi}\left(\frac{1}{\lambda}(\mu_2(s) - Q(s, a))\right)
\end{align*}
is positive for all actions in $\Lambda(s, \mu_2)$. Since there must be at least one action in $\Lambda(s, \mu_2)$, $h(\mu_1) - h(\mu_2) > 0$.  Therefore, we have proved that $h(\mu)$ decreases strictly on $(-\infty, \mu_{\max})$. Note that $h(\mu_{\max}) = 0<1$ and $h(\mu_{\min}) > 1$ where $\mu_{\min}=\min_{a}Q(s,a) + \lambda f_{\phi}'(1)$. This result implies there exist a unique $\mu^* \in (\mu_{\min}, \mu_{\max})$ satisfying~\eqref{eq:solve-mu-from-anyQ} as the result of the intermediate value theorem.
\end{proof}

\section{Proof for General Bellman Operator}\label{ap:operator}
In~\eqref{prob:current}, we define a general Bellman operator $\mathcal{T}_{\lambda}$ for regularized MDPs. Given one state $s \in \mathcal{S}$ and current value function $V_{\lambda}$,
\[ (\mathcal{T}_{\lambda} V_{\lambda})(s):=\max_{\pi} \sum_{a}\pi(a|s)\left[Q_{\lambda}(s, a)+\lambda \phi(\pi(a|s))\right],  \]
where $Q_{\lambda}(s, a) = r(s, a) + \gamma \mathbb{E}_{s'|s, a}V_{\lambda}(s')$ is Q-value function deriving from one-step foreseeing according to $V_{\lambda}$. In Lemma~\ref{lem:contraction}, we prove $\mathcal{T}_{\lambda}$ is a $\gamma$-contraction. In Theorem~\ref{thm: max}, we prove the simple lower and upper bound for $\mathcal{T}_{\lambda}$ under Assumption~\ref{assump:control-entropy}.

\begin{lemma}
\label{lem:contraction}
$\mathcal{T}_{\lambda}$ is a $\gamma$-contraction.
\end{lemma}
\begin{proof}
For any two state value functions $V_1$ and $V_2$, let $\pi_i$ be the policy that maximize  $\mathcal{T}_{\lambda} V_i$, $ i \in \{1, 2\}$. Then it follows that for any state $s$ in $\SP$,
\begin{align*}
&(\mathcal{T}_{\lambda} V_1)(s) -(\mathcal{T}_{\lambda} V_2)(s) \\
&=\sum_{a}\pi_1(a|s)\left[r(s, a) + \gamma \mathbb{E}_{s'|s, a}V_1(s')+\lambda \phi(\pi_1(a|s))\right]
- \max_{\pi} \sum_{a}\pi(a|s)\left[r(s, a) + \gamma \mathbb{E}_{s'|s, a}V_2(s')+\lambda \phi(\pi(a|s))\right]\\
&\le\sum_{a}\pi_1(a|s)\left[r(s, a) + \gamma \mathbb{E}_{s'|s, a}V_1(s')+\lambda \phi(\pi_1(a|s))\right]
-  \sum_{a}\pi_1(a|s)\left[r(s, a) + \gamma \mathbb{E}_{s'|s, a}V_2(s')+\lambda \phi(\pi_1(a|s))\right]\\
&=\gamma   \sum_{a} \pi_1(a|s)  \mathbb{E}_{s'|s, a} (V_1(s') - V_2(s'))\le \gamma  \|V_1 - V_2\|_{\infty}.
\end{align*}

By symmetry, it follows that for any state $s$ in $\SP$,
\[ (\mathcal{T}_{\lambda} V_2)(s) -(\mathcal{T}_{\lambda} V_1)(s) \le \gamma  \|V_1 - V_2\|_{\infty}   \]
Therefore, it follows that
\[  \|\mathcal{T}_{\lambda} V_2- \mathcal{T}_{\lambda} V_1\|_{\infty} \le \gamma  \|V_1 - V_2\|_{\infty}  \]
\end{proof}

\begin{proof}{\textbf{for Theorem~\ref{thm: max}}}
Fix any value function $V$ and $s \in \SP$. Note that $\phi(\pi(a|s))$ is non-negative due to (1) and (2) in Assumption~\ref{assump:1}. Therefore, by definition the left inequality follows from
\begin{align*}
\mathcal{T}_{\lambda} V(s)&= \max_{\pi} \sum_{a}\pi(a|s)\left[r(s, a) + \gamma \mathbb{E}_{s'|s, a}V(s')+\lambda \phi(\pi(a|s))\right]\\
&\ge \max_{\pi} \sum_{a}\pi(a|s)\left[r(s, a) + \gamma \mathbb{E}_{s'|s, a}V(s')\right]=\mathcal{T}V(s).
\end{align*}

For the right inequality, note that
\begin{align*}
\mathcal{T}_{\lambda} V(s)&= \max_{\pi} \sum_{a}\pi(a|s)\left[r(s, a) + \gamma \mathbb{E}_{s'|s, a}V(s')+\lambda \phi(\pi(a|s))\right]\\
&\le \max_{\pi} \sum_{a}\pi(a|s)\left[r(s, a) + \gamma \mathbb{E}_{s'|s, a}V(s')\right] + \lambda \max_{\pi} H_{\phi}(\pi)\\
&= \mathcal{T} V(s) +  \lambda \max_{\pi} H_{\phi}(\pi).
\end{align*}
where $H_{\phi}(\pi) = \sum_{a}\pi(a|s) \phi(\pi(a|s))$ defined in~\eqref{eq: entropy-like} is what we next aim to bound.

The Lagrangian of solving $\max_{\pi} H_{\phi}(\pi)$ is
\[ L(\pi, \beta, \mu) = H_{\phi}(\pi)  + \mu (\sum_{a}\pi(a|s) -1 ) + \beta_a \pi(a|s).  \]
Its stationary condition is
\[ \frac{ \partial L}{\partial \pi(a|s)} = f_{\phi}'(\pi(a|s)) + \mu + \beta_a = 0.   \]
If $\pi(a|s) > 0$ then $\beta_a=0$ from the complementary slackness. Let $\pi^*$ be the policy that maximizes $H_{\phi}(\pi)$ and $S = \{a: \pi^*(a|s) > 0 \}$ be its support set. Then $\pi(a|s) = g_{\phi}(-\mu) = \text{constant}$ for all $a \in S$. Hence, $\pi(a|s) = \frac{1}{|S|}$ for $a \in S$ and $= 0$ for $a \notin S$. Note that $g_{\phi}$ is strictly decreasing and the assumption $\liml_{x \to 0^+} x\phi(x) =0$, 
\[ H_{\phi}(\pi) = \sum_{a \in S(s)} \pi^*(a|s)\phi(\pi^*(a|s)) = \phi(\frac{1}{|S|})\le \phi(\frac{1}{|\mathcal{A}|}) \]
where the last inequality use the fact $\phi$ is decreasing and $|S| \le |\mathcal{A}|$.
\end{proof}

\section{Proof for Performance Error}\label{ap:error}
We prove Theorem~\ref{thm: difference} in that the difference of $V^*$ and $ V_{\lambda}^*$  is controlled by both $\lambda$ and $\phi(\cdot)$ under Assumption~\ref{assump:control-entropy}. To that end, we first introduce several useful lemmas which give some properties of $\TP_{\lambda}$ including monotonicity, translation and convergence of repeated applications. Then a combination of these lemmas will prove Theorem~\ref{thm: difference}.

\begin{lemma}[Monotonicity]
	\label{lem:monotonicity}
	$\TP_{\lambda}$ has the property of monotonicity, i.e., if $V_1(s) \le V_2(s)$ for all $s \in \SP$, then $ \TP_{\lambda} V_1(s) \le \TP_{\lambda} V_2(s)$ for all $s \in \SP$. 
\end{lemma}
\begin{proof}
The conclusion directly follows from
\begin{align*}
\TP_{\lambda} V_1(s)&=\max_{\pi} \sum_{a}\pi(a|s)\left[r(s, a) + \gamma \mathbb{E}_{s'|s, a}V_1(s')+\lambda \phi(\pi(a|s))\right]\\
&\le\max_{\pi} \sum_{a}\pi(a|s)\left[r(s, a) + \gamma \mathbb{E}_{s'|s, a}V_2(s')+\lambda \phi(\pi(a|s))\right]=\TP_{\lambda} V_2(s)
\end{align*}
\end{proof}

\begin{lemma}[Translation]
	\label{lem:plus-comstant}
	Let $c$ denote any constant. Define $(V + c)(s) \triangleq V(s) + c$ as the value function shifted by $c$. Then it follows that for any $s \in \SP$,
	\[  (\TP_{\lambda} (V + c))(s) = (\TP_{\lambda} V)(s) + \gamma c  \]
\end{lemma}
\begin{proof}
By definition, it directly follows from
\begin{align*}
(\TP_{\lambda} (V + c))(s) &=\max_{\pi} \sum_{a}\pi(a|s) \left[r + \gamma \mathbb{E}_{s'|s, a}(V+c)(s')+\lambda \phi(\pi(a|s))\right]\\
&=\max_{\pi} \sum_{a}\pi(a|s) \left[r + \gamma \mathbb{E}_{s'|s, a}V(s') + \gamma c+\lambda \phi(\pi(a|s))\right]=(\TP_{\lambda}V)(s) + \gamma c
\end{align*}
\end{proof}

\begin{lemma}[Convergence of Repeated Applications]
\label{lem:value-iteration}
For any initial value function $V_0$, define $V_n = \TP_{\lambda}^n V_0  \triangleq \underbrace{\TP_{\lambda}\cdots \TP_{\lambda}}_{n}V_0$ as the value function resulting from $n$ times application of $\TP_{\lambda}$ to $V_0$. Then
\[ \liml_{n \to \infty}  \|V_n - V_{\lambda}^*\|_{\infty}  = 0.  \]
\end{lemma}
\begin{proof}
Note that $V_{\lambda}^{*} = \TP_{\lambda} V_{\lambda}^*$. It follows that
\begin{align*}
\|V_n - V_{\lambda}^*\|_{\infty} = \|\TP_{\lambda} V_{n-1} - \TP_{\lambda} V_{\lambda}^*\|_{\infty}\le \gamma \|V_{n-1} - V_{\lambda}^*\|_{\infty}\le \cdots \le\gamma^n \|V_0- V_{\lambda}^*\|_{\infty}.
\end{align*}
The first equality follows from definition. The first inequality results from Lemma~\ref{lem:contraction}. The last inequality is due to $n$-times applications of the first inequality.
\end{proof}

\begin{proof}{\textbf{for Theorem~\ref{thm: difference}}}
Fix any initial value function $V_0$. We aim to use mathematical induction to prove the statement that  for any $n \ge 1$, it follows for any $s \in \SP$
\begin{equation}
\label{eq: bound-iteration}
\TP^{n}V_0(s) \le  \mathcal{T}_{\lambda}^{n} V_0(s) \le \TP^{n}V_0(s)  + \lambda\phi(\frac{1}{|\mathcal{A}|}) \sum_{t=0}^{n-1} \gamma^{t}.
\end{equation}
When $n=1$,~\eqref{eq: bound-iteration} results from Theorem~\ref{thm: max}. 

Suppose the statement holds when $n = k (k \ge 1)$. Consider the case where $n=k+1$. First it follows that
\[\TP^{k+1}V_0(s) \le  \TP \mathcal{T}_{\lambda}^{k} V_0(s) \le  \mathcal{T}_{\lambda}^{k+1} V_0(s).  \]
The first inequality follows from the hypothesis and the monotonicity of $\TP$ (which is a special case of $\TP_{\lambda}$ when $\lambda=0$) from Lemma~\ref{lem:monotonicity}. The second inequality results from letting $V = \mathcal{T}_{\lambda}^{k} V_0$ in Theorem~\ref{thm: max}.

Second, it follows that
\begin{align*}
\mathcal{T}_{\lambda}^{k+1}& V_0(s) =  \mathcal{T}_{\lambda}  \mathcal{T}_{\lambda}^{k} V_0(s)  \\
&\le\mathcal{T}_{\lambda}( \TP^{k}V_0(s)  + \lambda\phi(\frac{1}{|\mathcal{A}|}) \sum_{t=0}^{k-1} \gamma^{t})\\
&= \mathcal{T}_{\lambda} \TP^{k}V_0(s) +  \lambda\phi(\frac{1}{|\mathcal{A}|}) \sum_{t=1}^{k} \gamma^{t}\\
&\le \mathcal{T}^{k+1}V_0(s) +  \lambda\phi(\frac{1}{|\mathcal{A}|}) \sum_{t=0}^{k} \gamma^{t},
\end{align*}
where the first inequality follows from the induction where $n=k$ and the monotonicity of $\TP_{\lambda}$ from Lemma~\ref{lem:monotonicity}, the second equality holds by setting $V = \TP^{k}V_0$ and $c = \lambda\phi(\frac{1}{|\mathcal{A}|}) \sum_{t=0}^{k-1} \gamma^{t}$ in Lemma~\ref{lem:plus-comstant}. The last inequality results from letting $V = \mathcal{T}_{\lambda}^{k} V_0$ in Theorem~\ref{thm: max}.

Putting above results together, we prove that~\eqref{eq: bound-iteration} holds when $n=k+1$. Therefore by mathematical induction,~\eqref{eq: bound-iteration} holds for any positive integer $n$. From Lemma~\ref{lem:value-iteration}, we have $V^*(s) = \liml_{n \to \infty} \TP^n V_0(s)$ and $V_{\lambda}^*(s) = \liml_{n \to \infty} \TP_{\lambda}^n V_0(s)$. Now let $n$ go infinity in both sides of~\eqref{eq: bound-iteration}, we obtain
\[ V^*(s) \le V_{\lambda}^*(s) \le  V^*(s) + \frac{\lambda}{1-\gamma} \phi(\frac{1}{|\AP|}), \]
which proves the theorem.
\end{proof}

\section{Proof for Control the sparsity of Optimal Policy}\label{ap:control-sparse}
In this section, we show that the number of positive actions can be controlled by regularization coefficient $\lambda$. Similar results about Tsallis entropy regularized MDPs can be found in~\cite{lee2018sparse}. However their proof focuses on a specific regularization. The proof we provide is suitable for any regularizors satisfying Assumption~\ref{assump:1}.

\begin{proof}{\textbf{for Theorem~\ref{thm:control}}.}
At first we prove that the optimal policy will approximate uniform distribution on action space. Under such situation, it is obvious that the optimal policy will have no sparsity as $\lambda\to\infty$. Denote $H=\max_{\pi}H(\pi)$. For an arbitrary $\delta>0$, there exists $\lambda_{0}$, such that $\forall \lambda>\lambda_{0}$, $|\frac{r(s, a)}{\lambda}|\le\delta$. Next we estimate the error between $\frac{Q_{\lambda}^{*}(s, a)}{\lambda}$ and $\max_{\pi}\mathbb{E}[\sum_{t=1}^{\infty}\gamma^{t}\phi(\pi(a_{t}|s_{t}))|s_{0}=s, a_{0}=a]=\frac{\gamma}{1-\gamma}H$:
\begin{align}
    \frac{Q_{\lambda}^{*}(s, a)}{\lambda}-\frac{\gamma}{1-\gamma}H&=\frac{Q_{\lambda}^{*}(s, a)}{\lambda}-\max_{\pi}\mathbb{E}[\sum_{t=1}^{\infty}\gamma^{t}\phi(\pi(a_{t}|s_{t}))|s_{0}=s, a_{0}=a, \pi]\notag\\
    &\le\frac{Q_{\lambda}^{\pi_{\lambda}^{*}}(s, a)}{\lambda}-\mathbb{E}[\sum_{t=1}^{\infty}\gamma^{t}\phi(\pi_{\lambda}^{*}(a_{t}|s_{t}))|s_{0}=s, a_{0}=a, \pi_{\lambda}^{*}]\notag\\
    &=\mathbb{E}[\sum_{t=0}^{\infty}\gamma^{t}\frac{r(s_{t}, a_{t})}{\lambda}|s_{0}=s, a_{0}=a, \pi_{\lambda}^{*}]\notag\\
    &\le\frac{\delta}{1-\gamma}
\end{align}
On the other hand, denote $\pi_{H}^{*}=\argmax_{\pi}H(\pi)$, we have:
\begin{align}
     \frac{Q_{\lambda}^{*}(s, a)}{\lambda}-\frac{\gamma}{1-\gamma}H&=\frac{Q_{\lambda}^{*}(s, a)}{\lambda}-\max_{\pi}\mathbb{E}[\sum_{t=1}^{\infty}\gamma^{t}\phi(\pi(a_{t}|s_{t}))|s_{0}=s, a_{0}=a, \pi]\notag\\
     &\ge\frac{Q_{\lambda}^{\pi_{H}^{*}}(s, a)}{\lambda}-\mathbb{E}[\sum_{t=1}^{\infty}\gamma^{t}\phi(\pi_{H}^{*}(a_{t}|s_{t}))|s_{0}=s, a_{0}=a, \pi_{H}^{*}]\notag\\
     &=\mathbb{E}[\sum_{t=0}^{\infty}\gamma^{t}\frac{r(s_{t}, a_{t})}{\lambda}|s_{0}=s, a_{0}=a, \pi_{H}^{*}]\notag\\
     &\ge-\frac{\delta}{1-\gamma}
\end{align}
So $|\frac{Q_{\lambda}^{*}(s, a)}{\lambda}-\frac{\gamma}{1-\gamma}H|\le\frac{\delta}{1-\gamma}$.

Fix any $s\in\mathcal{S}$, denote $\mu_{\lambda}(s)$ is the solution satisfies Equation~\ref{eq:solve-mu-from-anyQ}:
\begin{align}
    1&= \sum_{a}  \max \left\{g_{\phi}\left(\frac{1}{\lambda}(\mu_{\lambda}(s) - Q_{\lambda}^{*}(s, a))\right), 0\right\}\notag\\
    &>|\mathcal{A}|\max\left\{g_{\phi}\left(\frac{1}{\lambda}(\mu_{\lambda}(s) - \min_{a}Q_{\lambda}^{*}(s, 
    a))\right), 0\right\}\notag\\
    &>|\mathcal{A}|g_{\phi}\left(\frac{1}{\lambda}(\mu_{\lambda}(s) - \min_{a}Q_{\lambda}^{*}(s, 
    a))\right)\notag
\end{align}
As $g_{\phi}$ is strictly decreasing, we have $\frac{1}{\lambda}\left(\mu_{\lambda}(s)-\min_{a}Q_{\lambda}^{*}(s, a)\right)>f'_{\phi}(\frac{1}{|\mathcal{A}|})$. By the same method, we can obtain that $\frac{1}{\lambda}\left(\mu_{\lambda}(s)-\max_{a}Q_{\lambda}^{*}(s, a)\right)<f'_{\phi}(\frac{1}{|\mathcal{A}|})$. Combining with $|\frac{Q_{\lambda}^{*}(s, a)}{\lambda}-\frac{\gamma}{1-\gamma}H|\le\frac{\delta}{1-\gamma}$:
\begin{align}
    \frac{\mu_{\lambda}(s)}{\lambda}-\frac{\gamma}{1-\gamma}H&=\frac{\mu_{\lambda}(s)}{\lambda}-\frac{\min_{a}Q_{\lambda}^{*}(s, a)}{\lambda}+\frac{\min_{a}Q_{\lambda}^{*}(s, a)}{\lambda}-\frac{\gamma}{1-\gamma}H\notag\\
    &>f'_{\phi}(\frac{1}{|\mathcal{A}|})-\frac{\delta}{1-\gamma}\notag\\
    \frac{\mu_{\lambda}(s)}{\lambda}-\frac{\gamma}{1-\gamma}H&=\frac{\mu_{\lambda}(s)}{\lambda}-\frac{\max_{a}Q_{\lambda}^{*}(s, a)}{\lambda}+\frac{\max_{a}Q_{\lambda}^{*}(s, a)}{\lambda}-\frac{\gamma}{1-\gamma}H\notag\\
    &<f'_{\phi}(\frac{1}{|\mathcal{A}|})+\frac{\delta}{1-\gamma}\notag
\end{align}

For arbitrary $s,a$, the following inequality holds:
\begin{align}
    \frac{\mu_{\lambda}(s)-Q_{\lambda}^{*}(s, a)}{\lambda}&=\frac{\mu_{\lambda}(s)}{\lambda}-\frac{\gamma}{1-\gamma}H+\frac{\gamma}{1-\gamma}H-\frac{Q_{\lambda}^{*}(s, a)}{\lambda}\notag\\
    &>f'_{\phi}(\frac{1}{|\mathcal{A}|})-\frac{2\delta}{1-\gamma}\notag\\
    \frac{\mu_{\lambda}(s)-Q_{\lambda}^{*}(s, a)}{\lambda}&=\frac{\mu_{\lambda}(s)}{\lambda}-\frac{\gamma}{1-\gamma}H+\frac{\gamma}{1-\gamma}H-\frac{Q_{\lambda}^{*}(s, a)}{\lambda}\notag\\
    &<f'_{\phi}(\frac{1}{|\mathcal{A}|})+\frac{2\delta}{1-\gamma}\notag\\
\end{align}
which concludes that $|\frac{\mu_{\lambda}(s)-Q_{\lambda}^{*}(s, a)}{\lambda}-f'_{\phi}(\frac{1}{|\mathcal{A}|})|<\frac{\delta}{1-\gamma}$. By continuity of $g_{\phi}$, $\forall \varepsilon>0$, choose a proper $\delta$, $|g_{\phi}(\frac{\mu_{\lambda}(s)-Q_{\lambda}^{*}(s, a)}{\lambda})-\frac{1}{|\mathcal{A}|}|<\varepsilon$. 

Next we prove that the sparsity of optimal policy $\pi_{\lambda}^{*}$ varies as $\delta\to\frac{1}{|\mathcal{A}|}$ when $\lambda\to0$. 

For arbitrary $(s, a)\in\mathcal{S}\times\mathcal{A}$, $\varepsilon>0$ and $\lambda>0$ , the following inequality holds:
\begin{align}
    0\le Q_{\lambda}^{*}(s, a)-Q^{*}(s, a)&\le Q_{\lambda}^{\pi_{\lambda}^{*}}(s, a)-Q^{\pi_{\lambda}^{*}}(s, a)\notag\\
    &=\lambda\mathbb{E}[\sum_{t=1}^{\infty}\gamma^{t}\phi(\pi_{\lambda}^{*}(a_{t}|s_{t}))|s_{0}=s, a_{0}=a, \pi_{\lambda}^{*}]\notag\\
    &\le\lambda\frac{\gamma}{1-\gamma}H
    \label{eq:qbound}
\end{align}
Denote $G(s)=\min_{a_{1}, a_{2}\in\mathcal{A}}|Q^{*}(s, a_{1})-Q^{*}(s, a_{2})|$, if $\lambda<\frac{1-\gamma}{H\gamma}G(s)$, the order of Q-values $Q^{*}(s, \cdot)$ is exactly the same with the order of $Q^{*}_{\lambda}(s, \cdot)$. In other words, denote $Q^{*}(s, a_{1})<Q^{*}(s, a_{2})<...<Q^{*}(s, a_{|\mathcal{A}|})$, then $Q_{\lambda}^{*}(s, a_{1})<Q_{\lambda}^{*}(s, a_{2})<...<Q_{\lambda}^{*}(s, a_{|\mathcal{A}|})$ still holds for $\lambda<\frac{1-\gamma}{H\gamma}G(s)$.

Next we prove the desired result by contradiction. For any given $a_{k}\in\mathcal{A}$ and $s\in\mathcal{S}$, and $\lambda< \frac{1-\gamma}{H\gamma}G(s)$, $\exists \lambda_{0}<\lambda$, such that $\pi_{\lambda_{0}}(a_{k}|s)=g_{\phi}(\frac{\mu_{\lambda_{0}}(s)-Q_{\lambda_{0}}^{*}(s, a_{k})}{\lambda_{0}})>0$. With the assumption, we can construct a sequence $\frac{1-\gamma}{H\gamma}G(s)>\lambda_{1}>\lambda_{2}>...>\lambda_{n}>...$, which satisifies $\liml_{n\to\infty}\lambda_{n}=0$ and $g_{\phi}(\frac{\mu_{\lambda_{n}}(s)-Q_{\lambda_{n}}^{*}(s, a_{k})}{\lambda_{n}})>0$, which is equivalent with $\mu_{\lambda_{n}}(s)-Q_{\lambda_{n}}^{*}(s, a_{k})<\lambda_{n}f'_{\phi}(0)$ as $g_{\phi}$ is a strictly decreasing function. Combining with KKT conditions~\eqref{eq:stationarity-condition}: $\mu_{\lambda_{n}}(s)=Q_{\lambda_{n}}^{*}(s, a_{|\mathcal{A}|})+\lambda_{n}f'_{\phi}(\pi_{\lambda_{n}}^{*}(a_{|\mathcal{A}}|s))$, the following inequality holds:
\begin{align}
    Q_{\lambda_{n}}^{*}(s, a_{|\mathcal{A}|})-Q_{\lambda_{n}}^{*}(s, a_{k})<\lambda_{n}(f'_{\phi}(0)-f'_{\phi}(\pi_{\lambda_{n}}^{*}(a_{|\mathcal{A}|}|s)))<\lambda_{n}(f'_{\phi}(0)-f'_{\phi}(1))
    \label{eq:abound}
\end{align}
By~\eqref{eq:qbound} and~\eqref{eq:abound}, 
\begin{align}
Q^{*}(s, a_{|\mathcal{A}|})-Q^{*}(s, a_{k})&\le Q_{\lambda_{n}}^{*}(s, a_{|\mathcal{A}|})-Q^{*}(s, a_{k})\notag\\
&\le Q_{\lambda_{n}}^{*}(s, a_{|\mathcal{A}|})-Q_{\lambda_{n}}^{*}(s, a_{k})+\lambda_{n}\frac{\gamma}{1-\gamma}H\notag\\
&<\lambda_{n}(f'_{\phi}(0)-f'_{\phi}(1)+\frac{\gamma}{1-\gamma}H)
\label{eq:contradiction}
\end{align}
As $\liml_{n\to\infty}\lambda_{n}=0$ and $f'_{\phi}(0)-f'_{\phi}(1)+\frac{\gamma}{1-\gamma}H$ is a positive constant, then the limit of right hand side of~\eqref{eq:contradiction} is $0$, which causes conflicts with that the left hand side of~\eqref{eq:contradiction} is a positive constant. Therefore, we claim that for any given $a\in\mathcal{A}$, $s\in\mathcal{S}$ and $a\not=\arg\max Q^{*}(s, \cdot)$, $\exists \lambda_{a, s}>0$, such that $\forall \lambda\le\lambda_{a, s}$, $\pi_{\lambda}^{*}(a|s)=0$. So for all $\lambda<\min_{a, s}\lambda_{a, s}$, the sparsity of the optimal policy $\pi_{\lambda}^{*}$ is $\delta=\frac{1}{|\mathcal{A}|}$.
\end{proof}

\section{Regularized Policy Iteration (RPI)}
\label{ap: rpi}
To solve problem~\eqref{prob:genera-regularized-rl} , we introduce \textit{Regularized Policy Iteration} (RPI), an algorithm that alternates between policy evaluation and policy improvement in the maximum regularized MDP framework. We first derive RPI on a tabular setting and show it provably converges to an optimal policy. Then we approximate RPI into a more practical algorithm which is an actor-critic method and thus named as regularized actor-critic (RAC).

The derivation of RPI stems from generalized policy iteration \cite{sutton2018reinforcement} that alternates between policy evaluation and policy improvement. In the policy evaluation step, we wish to compute the Q-value $Q_{\lambda}^{\pi}$ of a given policy $\pi$. When $\pi$ is fixed, $Q_{\lambda}^{\pi}$ can be computed iteratively by initializing any Q-value function and repeatedly applying the modified Bellman backup operator $\TP_{\lambda}^{\pi}$ defined by
\begin{equation}
\label{eq:operator-Q}
\TP_{\lambda}^{\pi} Q_{\lambda}(s, a) \triangleq r(s, a) + \gamma \EB_{s'|s, a} V_{\lambda}(s'),
\end{equation}
where $V_{\lambda}$ is the state value function derived from $Q_{\lambda}$,
\begin{equation}
\label{eq:V=Q+phi}
V_{\lambda}(s') = \EB_{a' \sim \pi(\cdot|s')}[ Q_{\lambda}(s', a') + \phi(\pi(a'|s'))].
\end{equation}
One can show that by repeatedly applying $\TP_{\lambda}^{\pi}$ to any initialized value function, the regualrized Q-value $Q_{\lambda}^{\pi}$ of the policy $\pi$ will be obtained.

In the policy improvement step, we wish to update the evaluated policy $\pio$ to an improved policy $\pin$  in terms of its regularized Q-values.Therefore for each state $s$ we update the policy according to
\begin{equation}
\label{eq:new-policy}
\pin(a|s) = \arg\max_{\pi} \EB_{a \sim \pi(\cdot|s)}[Q_{\lambda}^{\pio}(s, a)  + \lambda \phi(\pi(a|s))].
\end{equation}

If $\phi$ is good enough, we can find a closed form of $\pin$ for problem~\eqref{eq:new-policy}. For example, for Shannon entropy \cite{nachum2017bridging} with $\phi(x) = -\log(x)$, $\pin(a|s) \propto \exp(\frac{Q_{\lambda}^{\pio}}{\lambda})$; for Tsallis entropy \cite{lee2018sparse} with $\phi(x) =\frac{1}{2}(1-x)$, $\pin(a|s) = \max \left(\frac{Q_{\lambda}^{\pio}(s, a)}{\lambda} - \tau(\frac{Q_{\lambda}^{\pio}(s, \cdot)}{\lambda}), 0 \right)$, where $\tau(\frac{Q_{\lambda}^{\pio}(s, \cdot)}{\lambda}) = \frac{\sum_{a \in S(s, \lambda)}  \frac{Q_{\lambda}^{\pio}(s, a)}{\lambda} -1 }{|S(s, \lambda)|}$ is the normalization term and $S(s, \lambda)$ is the number of non-zero probability state-action pair. 
However, for a general $\phi$, it is unlikely to find a closed form of $\pin$. In that case the solution can be obtained through a numerical optimization method, since the maximization problem~\eqref{eq:new-policy} is a convex optimization whose domain is the probability simplex $\Delta_{\AP}$ and traditional convex solvers could solve it efficiently. 
Actually, in the experiments of RPI, for the regularization forms introduced in Section~\ref{sec:assumption} except the two examples mentioned above, there is no closed form and we use numerical optimization to improve the old policy.

Once evaluating and improving the current policy $\pio$, we can prove the resulting policy $\pin$ has a higher regularized Q-value than that of the old one. Therefore, by alternating the policy evaluation and the policy improvement, any initializing policy will provably converge to the optima  policy $\pi_{\lambda}^*$ under the framework of regularized MDPs (Theorem~\ref{thm:convergence}). 

\begin{theorem}
	\label{thm:convergence}
	For any policy $\pi_0$, by repeatedly applying policy evaluation and regularized policy improvement, $\pi_0$ will converge to the optimal policy $\pi_{\lambda}^*$ in the sense that $Q_{\lambda}^{\pi_{\lambda}^*}(s, a) \ge Q_{\lambda}^{\pi}(s, a)$ for all $\pi$ and $s \in \SP, a \in \AP$.
\end{theorem}

\section{Proof for Regularized Policy Iteration}
\label{ap:policy}
In this section, we give the proof of convergence of RPI. We first that repeatedly applying $\TP_{\lambda}^{\pi}$ to any initialized policy leads to the Q-value of a given policy $\pio$. Then we prove the policy improvement step will lead to a new policy $\pin$ which has higher Q-value than $\pio$. 

\begin{lemma}[Policy Evaluation]
Fix any policy $\pi$. Consider the Bellman backup operator $\TP_{\lambda}^{\pi}$ in~\eqref{eq:operator-Q}, for any initial Q-value $Q_0$, let $Q_n = \TP_{\lambda}^{\pi} Q_{n-1} (n \ge 1)$. Then $\liml_{n \to \infty} \| Q_n - Q_{\lambda}^{\pi}\|_{\infty} = 0$.
\end{lemma}
\begin{proof}
Similar to Lemma~\ref{lem:contraction}, we can prove $\TP_{\lambda}^{\pi}$ is a $\gamma$ contraction. Note that $Q_{\lambda}^{\pi} = \TP_{\lambda}^{\pi}Q_{\lambda}^{\pi} $. Therefore we have that
\begin{align*}
\|&Q_n - Q_{\lambda}^{\pi}\|_{\infty}  = \|\TP_{\lambda}^{\pi} Q_{n-1} - \TP_{\lambda}^{\pi} Q_{\lambda}^{\pi}\|_{\infty} \le \gamma \|Q_{n-1} - Q_{\lambda}^{\pi}\|_{\infty}\le \cdots \le\gamma^n \|Q_0- Q_{\lambda}^{\pi}\|_{\infty}.
\end{align*}
When $n$ goes infinity, $Q_n$ will converge to the regularized Q-value of $\pi$.
\end{proof}
\begin{lemma}[Policy Improvement]
	\label{lem:policy-improvement}
	Let $\pio$ be the evaluated policy with $Q_{\lambda}^{\pio}$ its regularized Q-value and $\pin$ be the optimizer of the maximization problem defined in~\eqref{eq:new-policy}. Then $Q_{\lambda}^{\pio}(s, a) \le Q_{\lambda}^{\pin}(s, a)$ for all $s \in \SP$ and $a \in \AP$. 
\end{lemma}
\begin{proof}
 Since $\pin$ is the maximizer of the problem defined in~\eqref{eq:new-policy}, it follows that for all states and actions
 \begin{align*}
 \EB_{a \sim \pio}[Q_{\lambda}^{\pio}(s, a)  + \lambda \phi(\pio(a|s))]\le \EB_{a \sim \pin}[Q_{\lambda}^{\pio}(s, a)  + \lambda \phi(\pin(a|s))]
 \end{align*}
 
Let $\tau_t = (s_0, a_0, \cdots, s_t, a_t)$ denotes the trajectory and $\tau$ is the whole trajectory (with infinite horizon). $\tau \sim \pio$ means the trajectory is generated by $\pio$. It follows that 
\begin{align*}
&Q_{\lambda}^{\pio}(s_0, a_0)\\
&=\EB [ r(s_0, a_0) + \gamma \EB_{s_1|\tau_0} V_{\lambda}^{\pio}(s_1) ]\\
&=\EB [ r(s_0, a_0) + \gamma \EB_{s_1|\tau_0} \EB_{a_1\sim \pio}[Q_{\lambda}^{\pio}(s_1, a_1)  + \lambda \phi(\pio(a_1|s_1))] ]\\
&\le \EB [ r(s_0, a_0) + \gamma \EB_{s_1|\tau_0} \EB_{a_1\sim \pin}[Q_{\lambda}^{\pio}(s_1, a_1)  + \lambda \phi(\pin(a_1|s_1))] ]\\
&= \EB_{\tau_1 \sim \pin} [ r(s_0, a_0) +  \gamma (r(s_1, a_1) + \lambda \phi(\pin(a_1|s_1) ))\quad  + \gamma^{2} \EB_{s_{2}| \tau_1} V_{\lambda}^{\pio}(s_{2}) ]\\
& \le \EB_{\tau_n \sim \pin} [ r(s_0, a_0) + \sum_{t=1}^n \gamma^t(r(s_t, a_t) + \lambda \phi(\pin(a_t|s_t) ))  + \gamma^{n+1} \EB_{s_{n+1}| \tau_n} V_{\lambda}^{\pio}(s_{n+1}) ]\\
& \le  \EB_{\tau \sim \pin} [ r(s_0, a_0) + \sum_{t=1}^{\infty} \gamma^t(r(s_t, a_t) + \lambda \phi(\pin(a_t|s_t) ))]\\
&=Q_{\lambda}^{\pin}(s_0, a_0),
\end{align*}
where the last inequality is because we repeatedly expanded $Q_{\lambda}^{\pio}$ on the RHS by applying~\eqref{eq:V=Q+phi} and $Q_{\lambda}^{\pio}$ is bounded by $\frac{R_{\max}}{1-\gamma}$. 
\end{proof}

\begin{proof}{\textbf{for Theorem~\ref{thm:convergence}}}
	Let $\pi_i$ be the policy at iteration $i$ of RPI. By Lemma~\ref{lem:policy-improvement}, the sequence $Q_{\lambda}^{\pi_i }$ is monotonically increasing. Since $Q_{\lambda}^{\pi_i}$ is bounded by $\frac{R_{\max}}{1-\gamma}$ for any policy $\pi_i$, therefore $Q_{\lambda}^{\pi_i }$ will converge to a limit, denoted by $Q_{\lambda}^{\lim}$. Let $\pi_{\lim}= \argmax_{\pi} \EB_{a \sim \pi(\cdot|s)}[Q_{\lambda}^{\lim}(s, a)  + \lambda \phi(\pi(a|s))]$. It is obvious that $Q_{\lambda}^{\pi_{\lim}} = Q_{\lambda}^{\lim}$. We aim to prove $\pi_{\lim} = \pi_{\lambda}^*$. To that end, we only need to prove $Q_{\lambda}^* = Q_{\lambda}^{\lim}$. For one hand, $Q_{\lambda}^{\pi_{\lim}}(s, a)=\liml_{n \to \infty}Q_{\lambda}^{\pi_i} (s, a)\le Q_{\lambda}^*(s, a) = Q_{\lambda}^{\pi_{\lambda}^*}(s, a)$. For another hand, at convergence, it must be the case that for all policy $\pi$, 
	\[   \EB_{a \sim \pi}[Q_{\lambda}^{\pi_{\lim}}(s, a)  + \lambda \phi(\pi(a|s))]
	 \le \EB_{a \sim \pi_{\lim}}[Q_{\lambda}^{\pi_{\lim}}(s, a)  + \lambda \phi(\pi_{\lim}(a|s))].  \]
	 Using the same iterative argument as in the proof of Lemma~\ref{lem:policy-improvement}, we get $Q_{\lambda}^{\pi_{\lambda}^*}(s, a) \le Q_{\lambda}^{\pi_{\lim}}(s, a)$ for all states and actions. Putting above results together, it follows that $Q_{\lambda}^* = Q_{\lambda}^{\lim}$ therefore $\pi_{\lim} = \pi_{\lambda}^*$.
\end{proof}

\section{Experiment Details}\label{ap:exp}
\subsection{Discrete Environments}
\label{ap:exp_discrete}

\subsubsection{Environment setup}
For the random MDP model, we choose $|\AP|=50$, $|\SP|=10$ and $\gamma=0.99$. Each state is assigned an index ranging from $0$ to $49$. The transition probabilites are generated by uniform distribution $[0,1]$ and each entry of transition is clipped as zero with probability $0.95$. Then each row of the clipped matrix is scaled to a probability distribution. The state we monitored is the state with index zero. The rewards are generated by uniform distribution $[0,1]$. The initial Q-value is generated by uniform distribution $[0,10]$ and policies are calculated explicitly or implicitly from Q-values.

For $(2N-1) \times (2N-1) $ GridWorld model, we choose $N=10$ and $\gamma=0.99$. The action space includes four actions (left, right, up, down). Each grid is indexed by an Cartesian coordinates $(x, y)$ with $x$ the row index and $y$ the column index. $x$ and $y$ are all range from $-(N-1)$ to $N-1$. The agent is initialized at the origin $(0, 0)$. Once it achieves four corners (i.e., $\pm(N-1) \times \pm (N-1)$), a reward with value $1$ will be obtained. Otherwise, no reward will be given. Due to the symmetry of GridWorld, we are interesting on the three states $(0, 0), (0, N/2), (N/2, N/2)$. In the origin $(0, 0)$, all actions should be equal. While the agent locates at $(0, N/2)$ or $(N/2, N/2)$, the optimal policy should put more probability mass on the action which could lead to 

\subsubsection{Optimization}
In this section, we detail how we conduct RPI(Appendix~\ref{ap: rpi}) in two discrete environments. Given a regularization function, we run 500 iterations of RPI that alternates between policy evaluation and policy improvement. 

\paragraph{Policy evaluation}  Since in our experiments the transition probability is known, the evaluation of a given policy is conducted by DP. Specifically, let $\mathbb{P}^{\pi} \in \mathbb{R}^{|\SP| \times |\SP|}$ denote the transition matrix deduced from $\pi$, i.e., $\mathbb{P}^{\pi}(s, s') = \sum_{a'} \pi(a'|s)\mathbb{P}(s'|a', s)$ and $r_{\lambda}^{\pi} \in \mathbb{R}^{|\SP|}$ the reward vector deduced from $\pi$, i.e., $r_{\lambda}^{\pi}(s) = \sum_{a'}r(s, a')\pi(a'|s)$. Then the regularized state value function $V_{\lambda}^{\pi}$ is solved from
\[ V_{\lambda}^{\pi}  = r_{\lambda}^{\pi} + \gamma \mathbb{P}^{\pi} V_{\lambda}^{\pi}  \quad \Rightarrow \quad V_{\lambda}^{\pi} =  (1- \gamma \mathbb{P}^{\pi})^{-1}  r_{\lambda}^{\pi}\]
where by slightly notation abuse, $V_{\lambda}^{\pi} \in \mathbb{R}^{|\mathcal{S}|}$ is the vector with each coordinate $V_{\lambda}^{\pi}(s)$. Then $Q_{\lambda}^{\pi}$ can be computed from $V_{\lambda}^{\pi}$ by definition~\eqref{eq:Q=r+V}.

\paragraph{Policy improvement} The policy improvement step involves an possibly intricate convex optimization~\eqref{eq:new-policy}. Here we detail how we solve the involved convex optimization.

Let $Q_{\lambda}^{\pio}$ denote the already evaluated Q-value function of $\pio$. For $\phi(x)=\frac{1}{2}(1-x)$, since the improved policy has an explicit form \cite{lee2018sparse}. However, for $\phi(x)=\cos(\frac{\pi}{2}x)$ and $\phi(x)=\exp(1)-\exp(x)$ which do not have an closed form and their corresponding $g_{\phi}$ are hard to formulate, thus we solve the convex optimization problem~\eqref{eq:new-policy} directly. Specifically, for each $s \in \SP$, we solve
\[ \max_{\pi} \sum_{a} \pi(a|s) Q_{\lambda}^{\pio}(s, a) + \lambda \sum_{a} \pi(a|s) \phi(\pi(a|s)).   \]
In practice, we use CVXOPT \cite{vandenberghe2010cvxopt} package to compute the improved policy.

\subsubsection{Regularizers}
We test four \textit{basic} regularizers, including $-\log x$, $\frac{1}{2}(1-x)$, $\cos(\frac{\pi}{2}x)$ and $\exp(1)-\exp(x)$. From Proposition~\ref{prop:phi-operation} and~\ref{prop:phi-operation-}, we can combine different basic regularizers to more complicated ones, which we term as \textit{combined} regularizers. We test the following three combined regularizers, (1) \texttt{min}: the minimum of \texttt{tsallis} and \texttt{shannon}, i.e., $\min\{-\log(x), 2(1-x)\}$, (2) \texttt{poly}: the positive addition of two polynomial functions, i.e., $\frac{1}{2}(1-x) + (1-x^2)$ and (3) \texttt{mix}: the positive addition of \texttt{tsallis} and \texttt{shannon}, i.e., $-\log(x) + \frac{1}{2}(1-x)$. We draw these seven regularizers and their corresponding $f_{\phi}'$ respectively in Figure~\ref{phis} (a) and (b).

\begin{figure*}[ht]
    \centering
    \subfloat[Different $\phi$'s]{
    \includegraphics[width=0.25\textwidth] {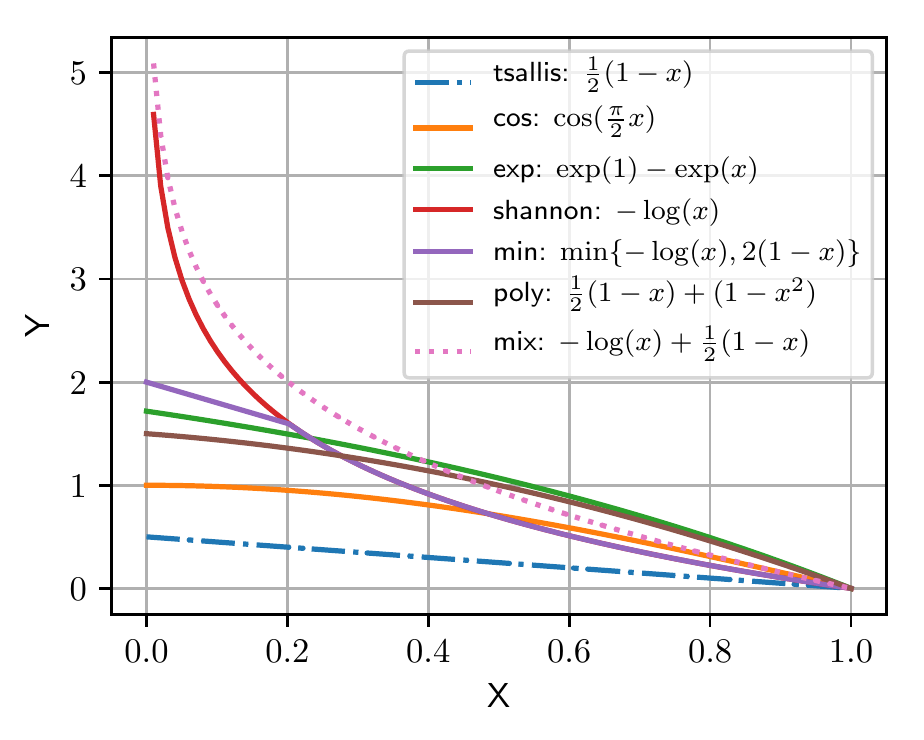}}
    \hspace{-.1in} 
    \subfloat[$f_{\phi}'$ for different $\phi$'s]{
    \includegraphics[width=0.25\textwidth] {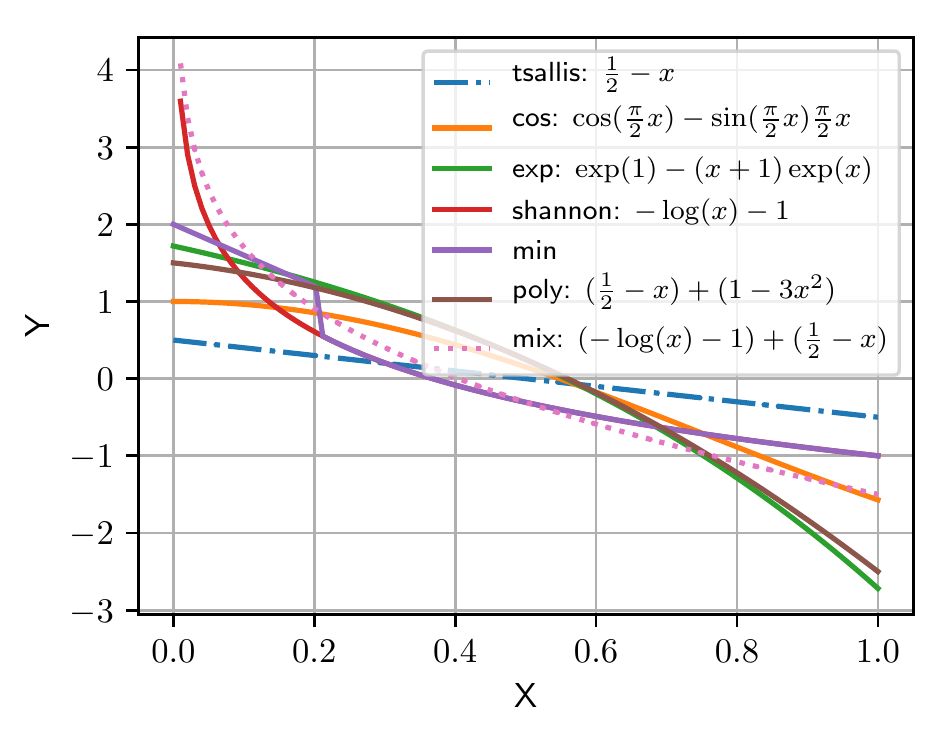}}
    \hspace{-.1in} 
    \subfloat[Random MDP]{
    \includegraphics[width=0.25\textwidth] {fig/random-action-ratio.pdf}}
    \hspace{-.1in} 
    \subfloat[Gridworld]{
    \includegraphics[width=0.25\textwidth] {fig/multigrid-action-ratio.pdf}}\\
    \caption{ (a) plots seven different regularization forms we will investigate. (b) shows the plot of $f_{\phi}' = \phi + x \phi'$ for corresponding regularizers. We prefer finite $f_{\phi}'(0)$ since it implies the optimal policy has a potentially sparse distribution if $\lambda$ is appropriately selected. (c) and (d) shows the results of the sparsity $\delta$ of the optimal policy on two envorinments (Random MDP and Gridworld). }
    \label{phis}
\end{figure*}

\subsubsection{Results for Random MDP}
Figure~\ref{random_action_dis} shows the probability mass of all actions in the optimal policy at selected state. When $\lambda$ is
small, all regularizers except shannon have some zero-probability actions. When $\lambda$ is just over 2, $\texttt{exp}$ and $\texttt{tsallis}$
already have a full action support set. By contrast, $\texttt{cos}$ is still sparse enough, implying the trigonometric function $\texttt{cos}$
has a stronger ability in modeling sparseness. In the extreme case where $\lambda$ is sufficiently large, the optimal policy will
converge to a uniform distribution on the action space as we expect.

\begin{figure*}[ht]
    \centering
    \subfloat[\texttt{cos}: $\cos(\frac{\pi}{2}x)$]{
    \includegraphics[width=0.25\textwidth] {fig/random-cosx-action-dis.pdf}}
    \hspace{-.1in} 
    \subfloat[\texttt{exp}: $\exp(1)-\exp(x)$]{
    \includegraphics[width=0.25\textwidth] {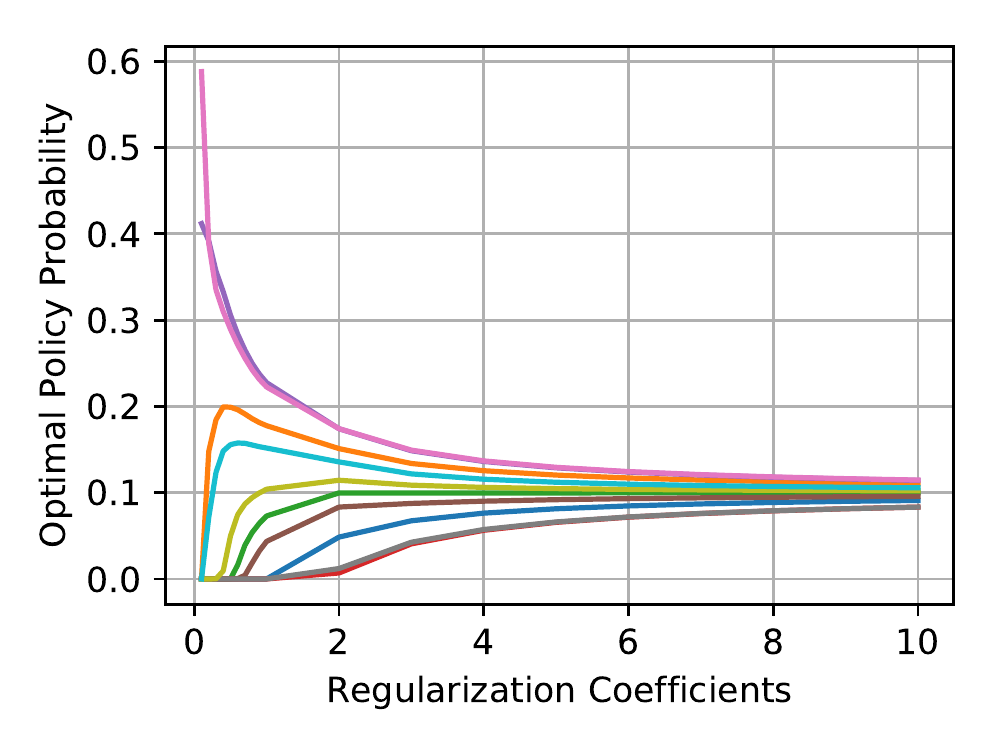}}
    \hspace{-.1in} 
    \subfloat[\texttt{tsallis}: $\frac{1}{2}(1-x)$]{
    \includegraphics[width=0.25\textwidth] {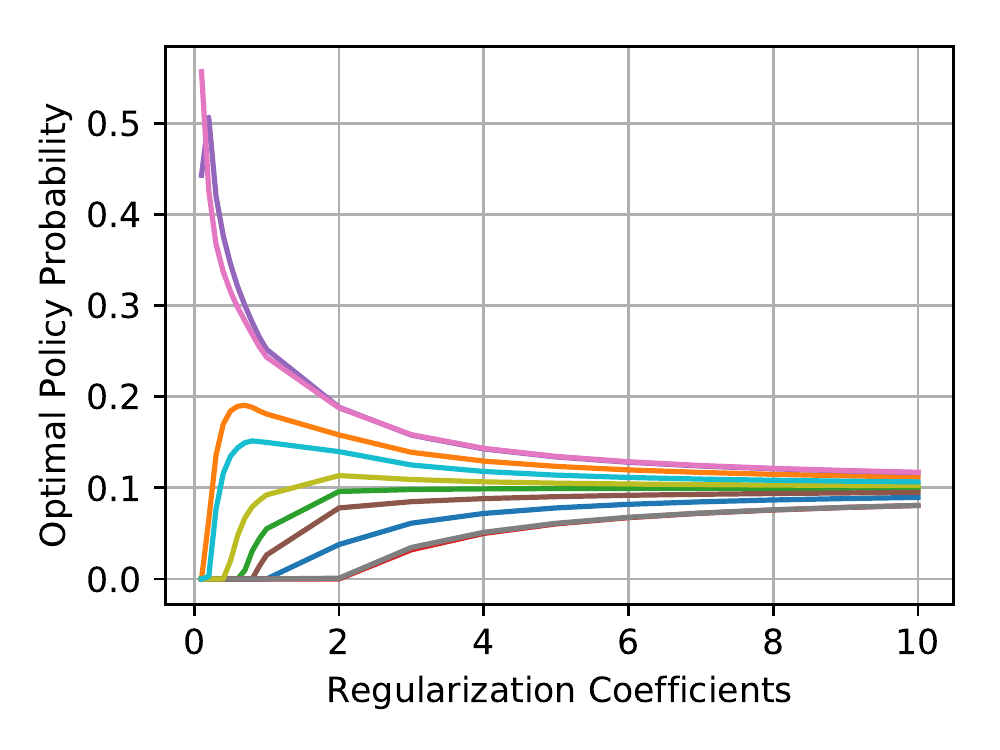}}
    \hspace{-.1in} 
    \subfloat[\texttt{shannon}: $-\log x$]{
    \includegraphics[width=0.25\textwidth] {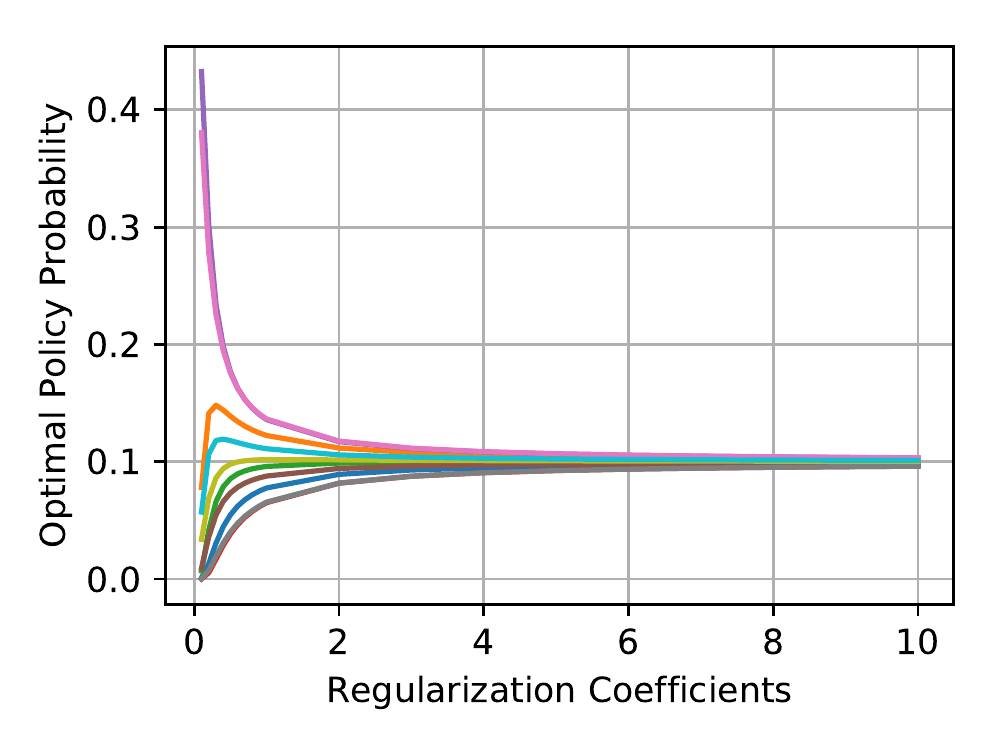}}
    \hspace{-.1in} \\
    \subfloat[\texttt{min}:$\min\{-\log(x),2(1-x)\}$]{
    \includegraphics[width=0.25\textwidth] {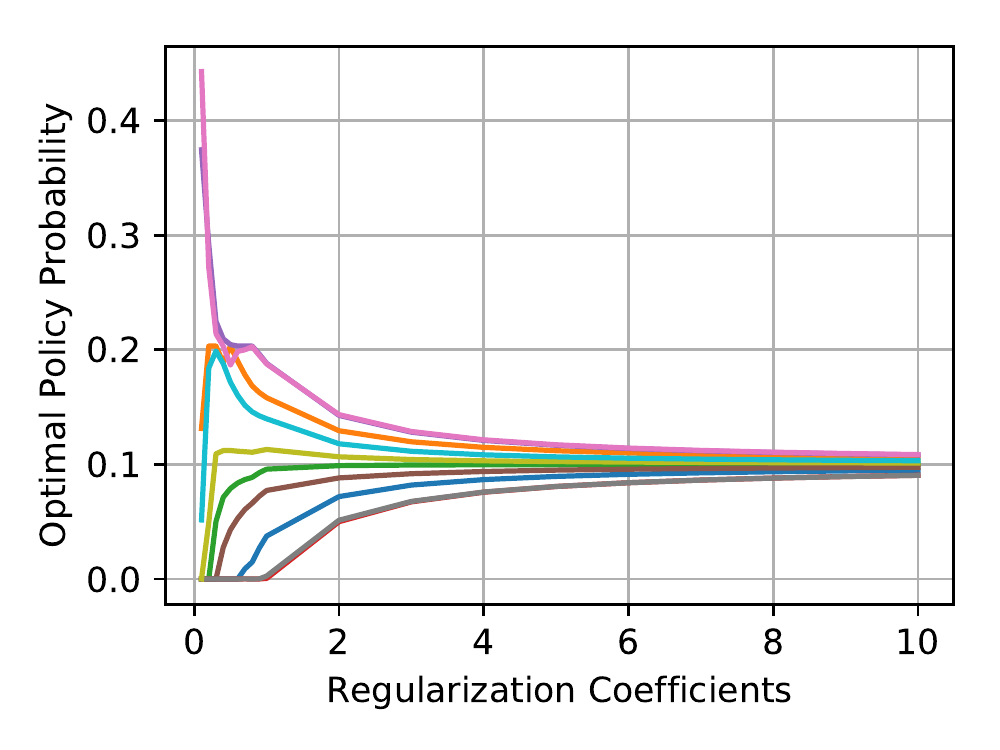}}
    \subfloat[\texttt{poly}: $\frac{1}{2}(1-x) + (1-x^2)$]{
    \includegraphics[width=0.25\textwidth] {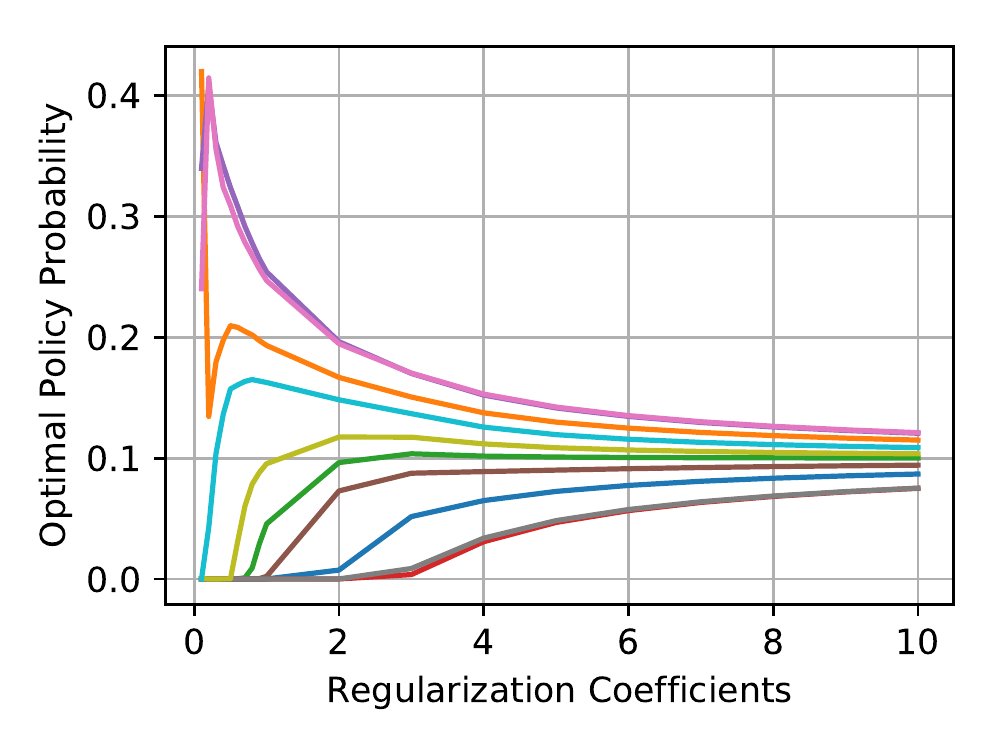}}
    \subfloat[\texttt{mix}: $-\log(x) + \frac{1}{2}(1-x)$]{
    \includegraphics[width=0.25\textwidth] {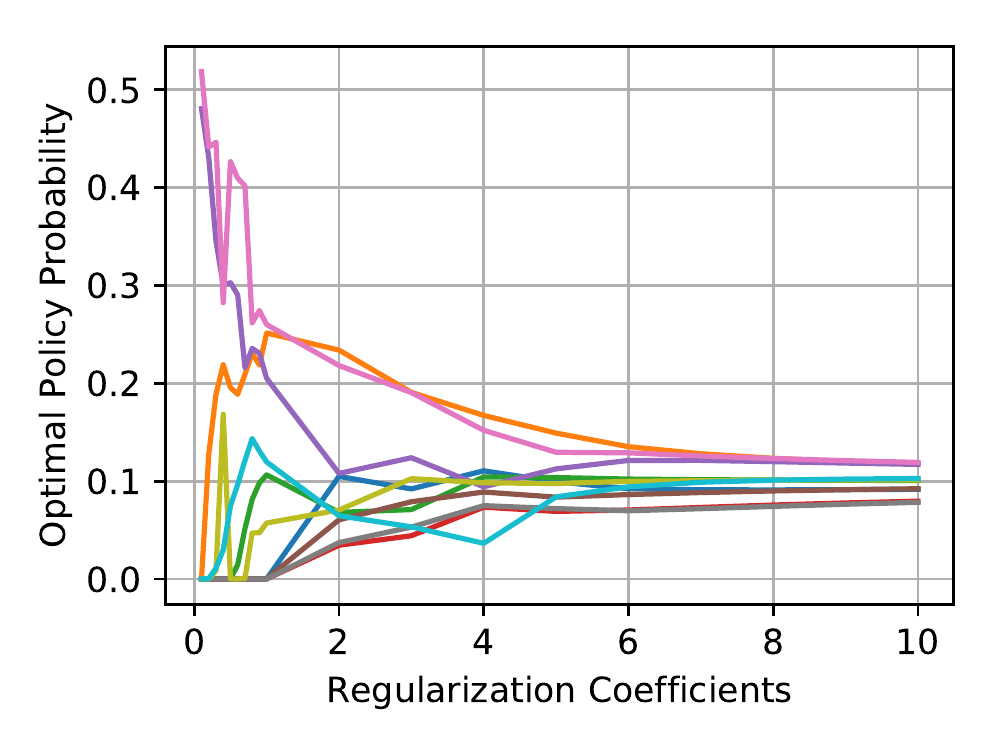}}
    \caption{(a)-(g) shows the changing process of the probability mass on each action in the optimal policy in a random MDP where $|\mathcal{A}|=10$. There are totally ten colored curves in each figure with one color representing one action. }
        \label{random_action_dis}
\end{figure*}

\subsubsection{Results for Gridworld}
Figure~\ref{multi_action} shows the probability mass of four actions in the optimal policy at selected three states. When $\lambda$ is large, the optimal policies tend to uniform distribution. We show the result of three combined regularizars in Figure~\ref{multi_combine_action}. It can be seen from these figures that in the regime of low $\lambda$, the optimal policy at different states show different preferrence. As shown in Random MDP, $\mathrm{cos}$ still has the strongest sparseness power.

\begin{figure*}[ht]
    \centering
    \subfloat[\texttt{cos}: $\cos(\frac{\pi}{2}x)$]{
    \includegraphics[width=0.25\columnwidth]{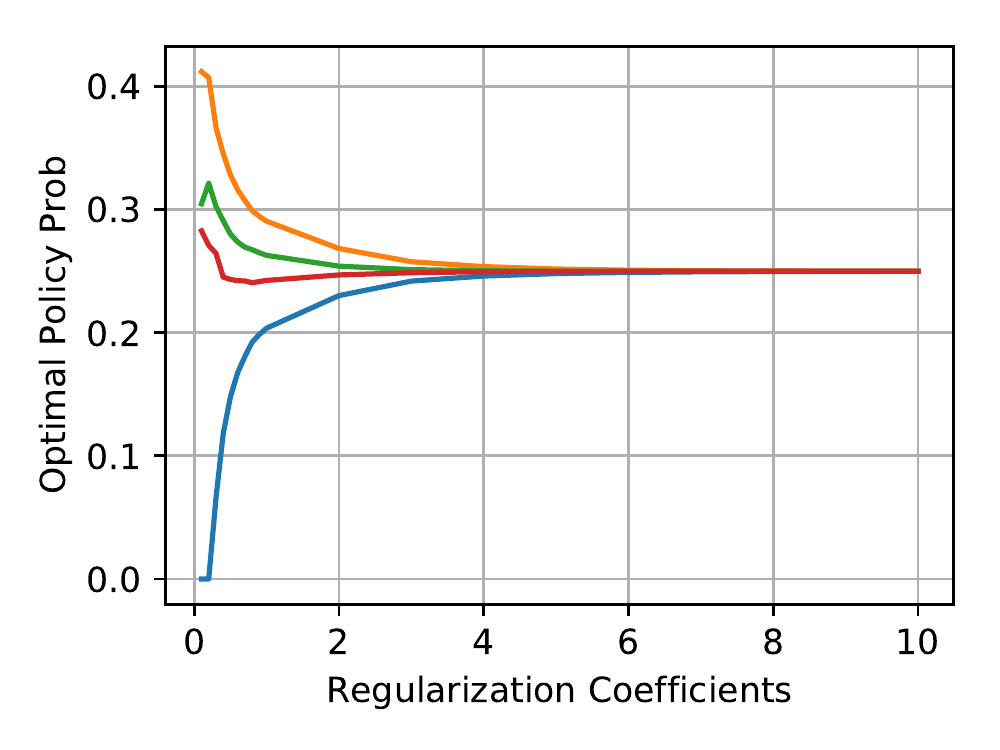}}
    \hspace{-.1in} 
    \subfloat[\texttt{exp}: $\exp(1)-\exp(x)$]{
    \includegraphics[width=0.25\columnwidth]{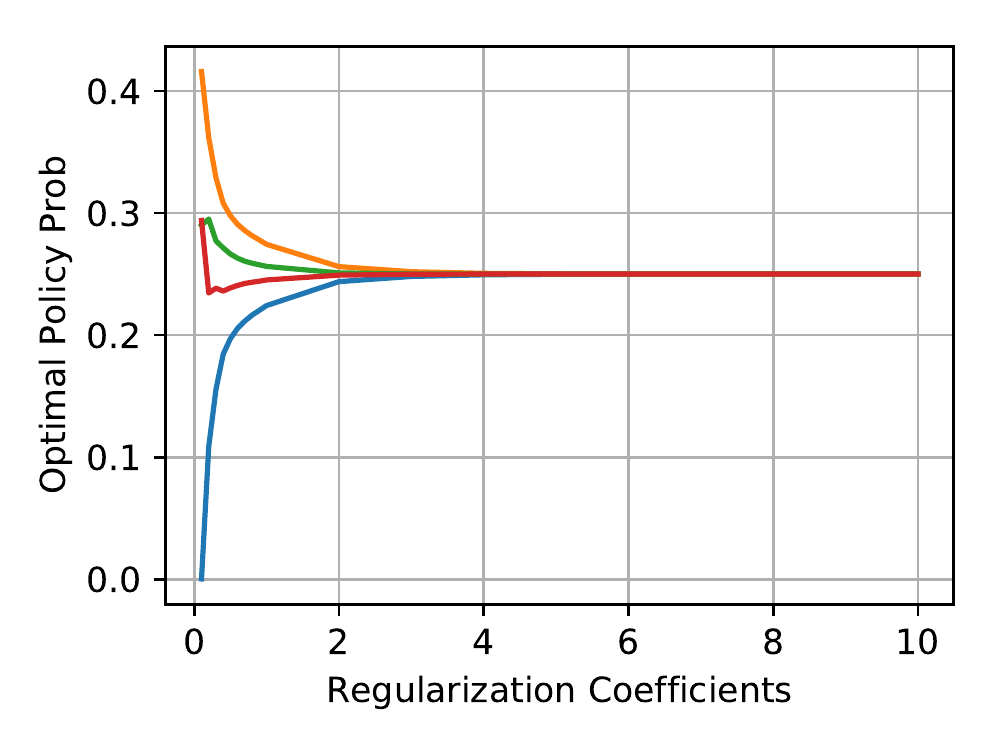}}
    \hspace{-.1in} 
    \subfloat[\texttt{tsallis}: $\frac{1}{2}(1-x)$]{
    \includegraphics[width=0.25\columnwidth]{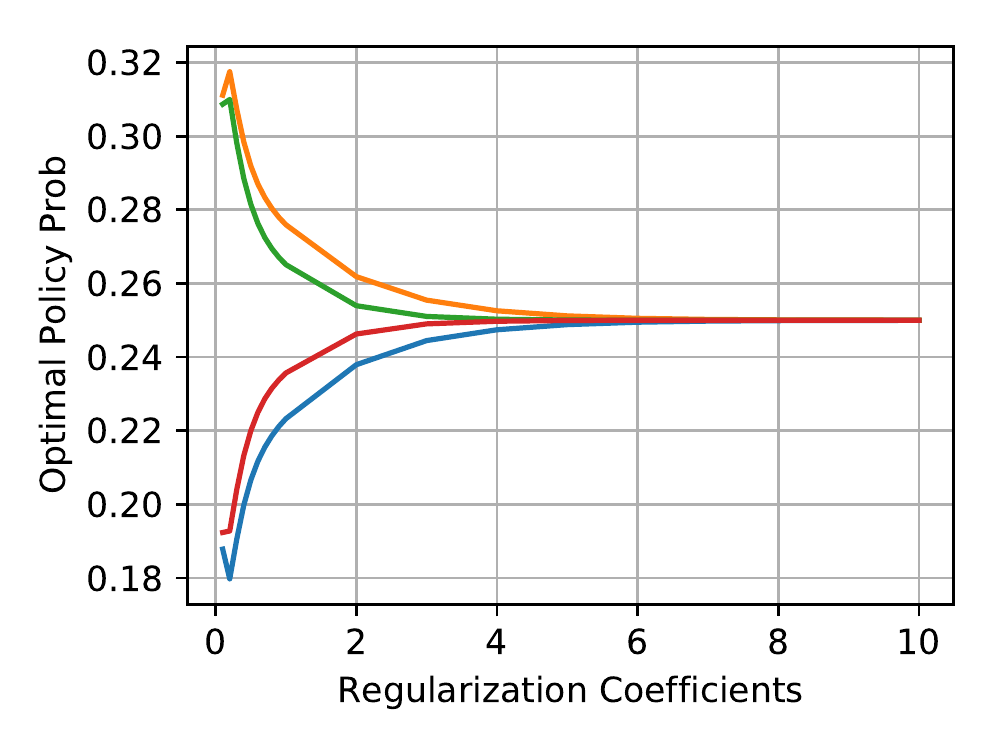}}
    \hspace{-.1in} 
    \subfloat[\texttt{shannon}: $-\log x$]{
    \includegraphics[width=0.25\columnwidth]{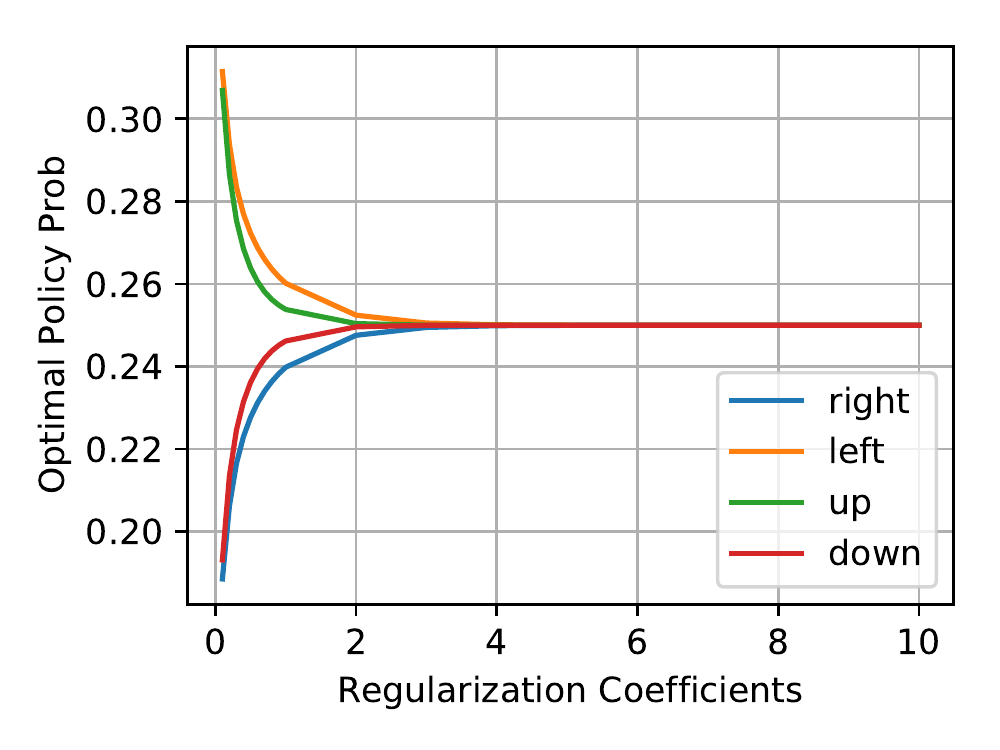}}
    \hspace{-.1in} \\
    \subfloat[\texttt{cos}: $\cos(\frac{\pi}{2}x)$]{
     \includegraphics[width=0.25\columnwidth]{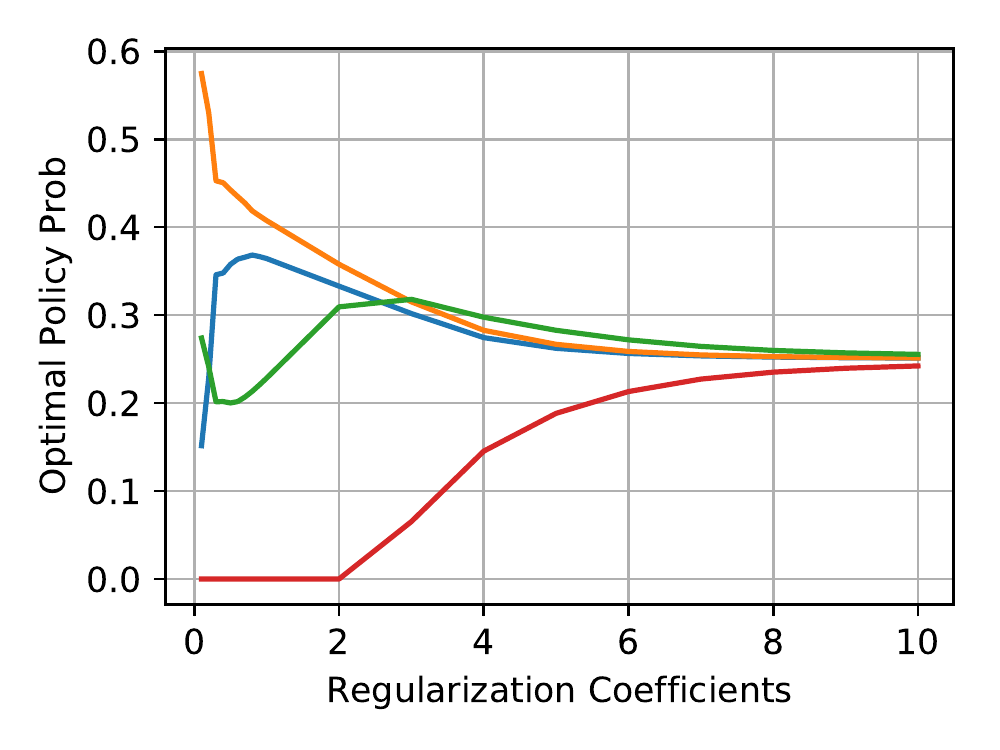}}
    \hspace{-.1in} 
    \subfloat[\texttt{exp}: $\exp(1)-\exp(x)$]{
    \includegraphics[width=0.25\columnwidth]{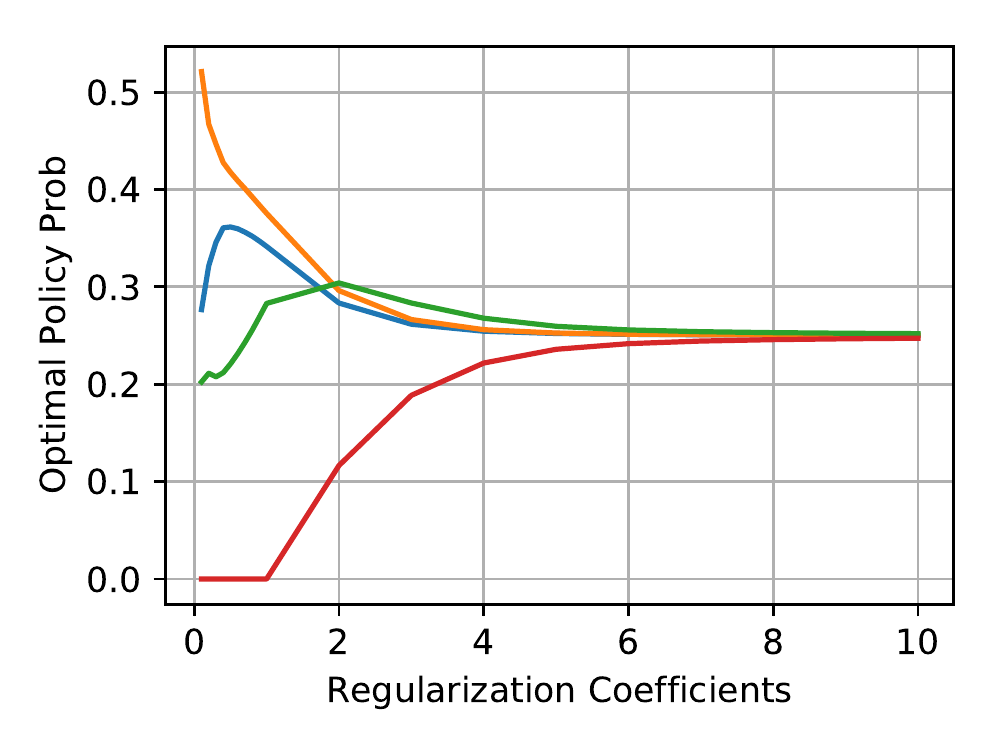}}
    \hspace{-.1in} 
    \subfloat[\texttt{tsallis}: $\frac{1}{2}(1-x)$]{
    \includegraphics[width=0.25\columnwidth]{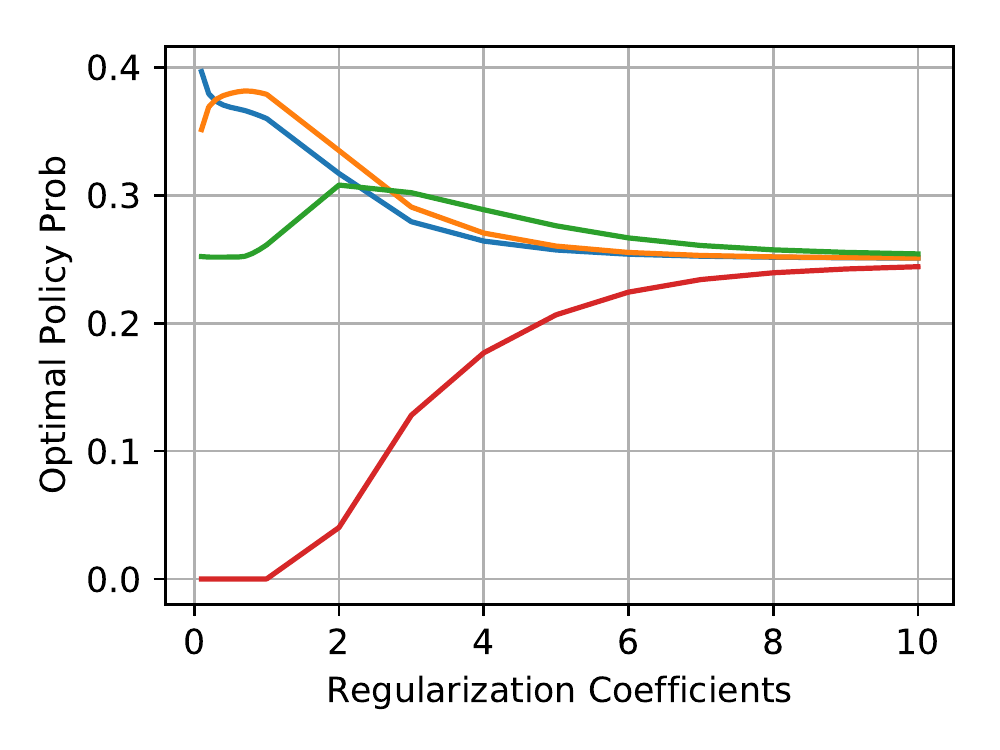}} \hspace{-.1in} 
    \subfloat[\texttt{shannon}: $-\log x$]{
    \includegraphics[width=0.25\columnwidth]{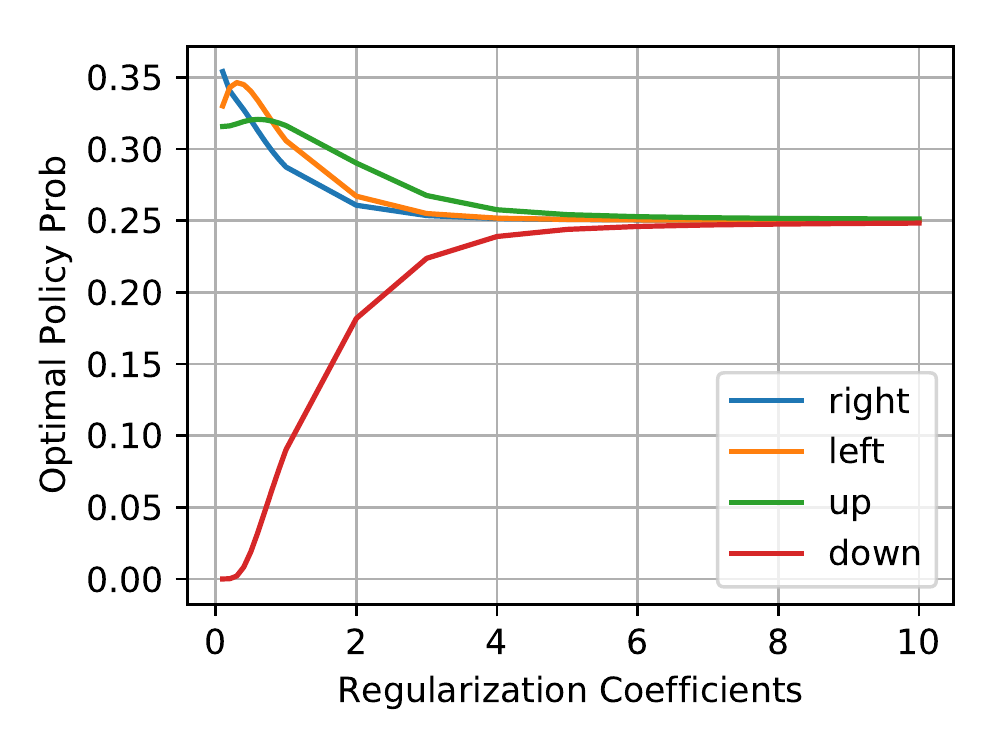}}
    \hspace{-.1in} \\
    \subfloat[\texttt{cos}: $\cos(\frac{\pi}{2}x)$]{
    \includegraphics[width=0.25\columnwidth]{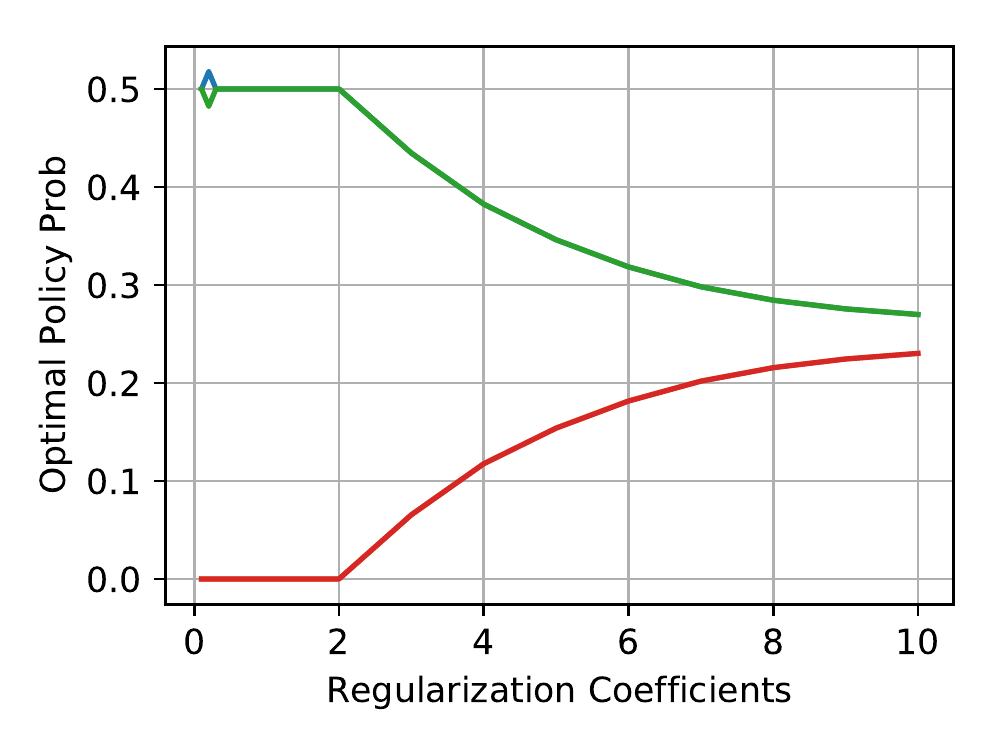}}
    \hspace{-.1in} 
    \subfloat[\texttt{exp}: $\exp(1)-\exp(x)$]{
    \includegraphics[width=0.25\columnwidth]{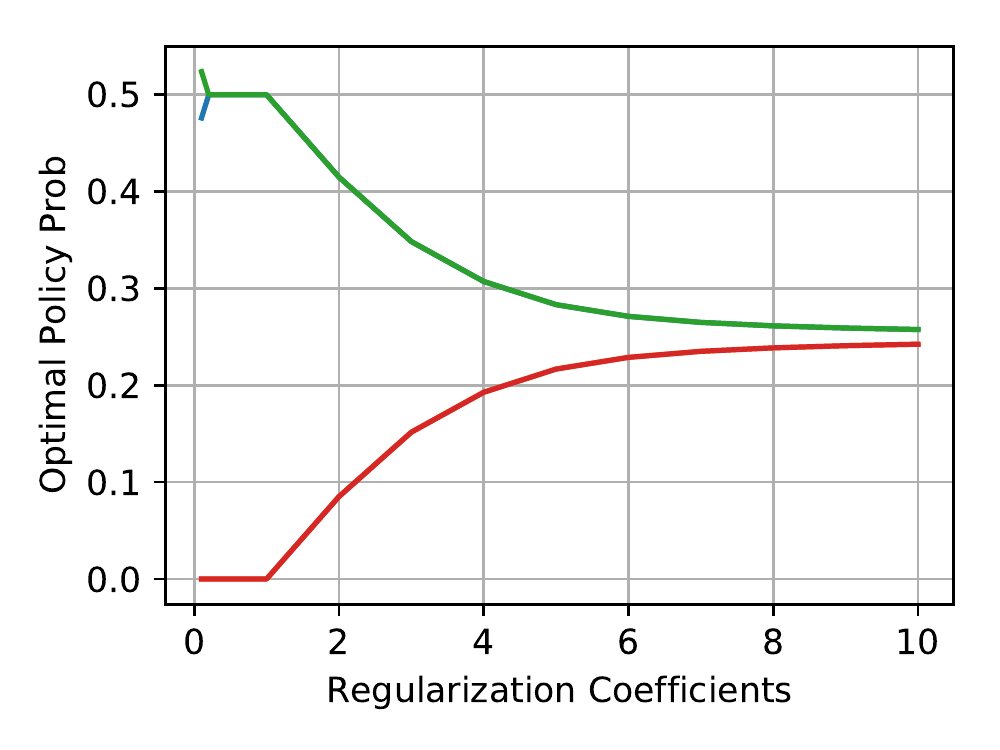}}
    \hspace{-.1in} 
    \subfloat[\texttt{tsallis}: $\frac{1}{2}(1-x)$]{
    \includegraphics[width=0.25\columnwidth]{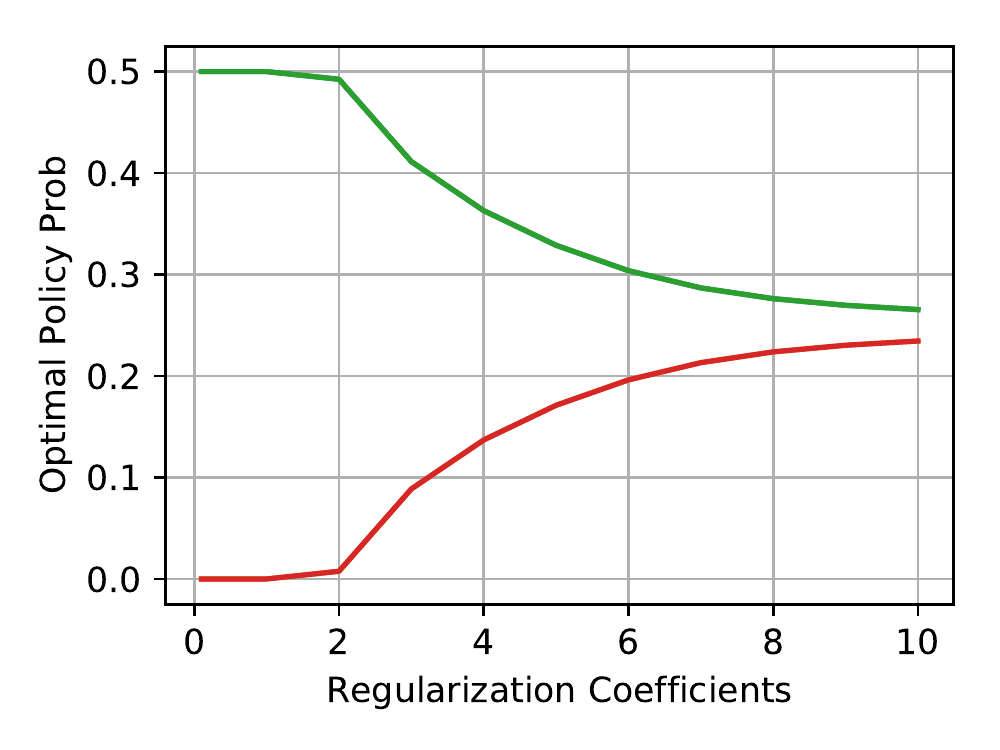}}
    \hspace{-.1in} 
    \subfloat[\texttt{shannon}: $-\log x$]{
    \includegraphics[width=0.25\columnwidth]{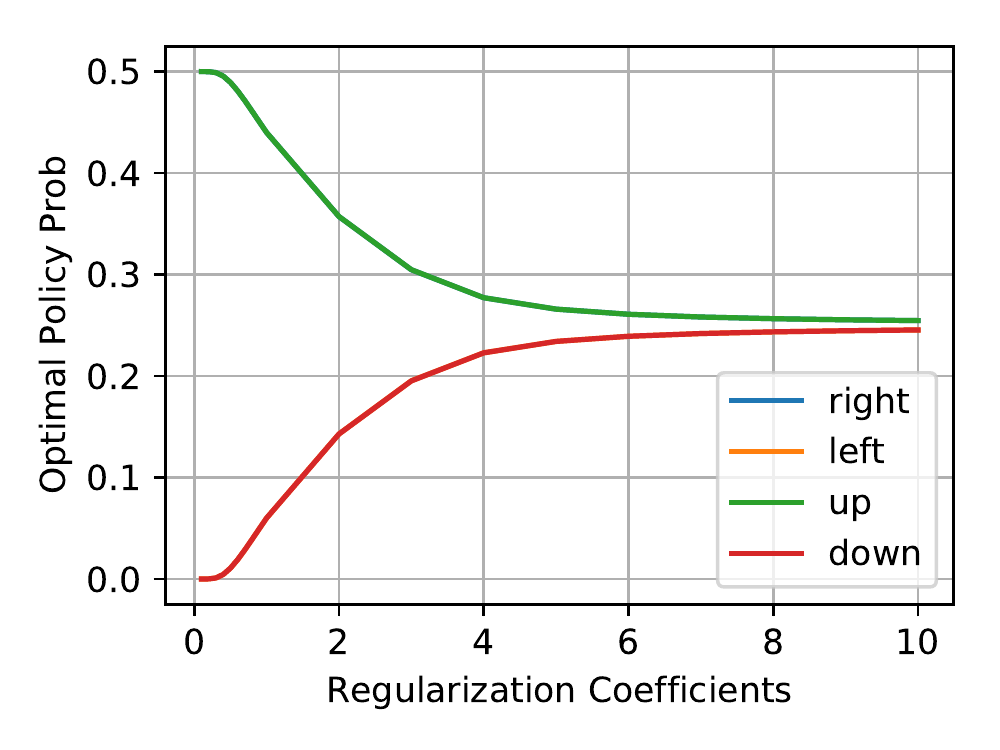}}
    \hspace{-.1in} \\
    \caption{The probability mass on four actions in the optimal policy regularized by four basic regularization functions at selected three states. (a)-(d) shows the results for the origin $(0,0)$. (e)-(h) shows the results for the state $(0, N/2)$ and (i)-(l) shows the results for the state $(N/2, N/2)$}
    \label{multi_action}
\end{figure*}

\begin{figure*}[ht]
    \centering
    \subfloat[\texttt{min}:$\min\{-\log(x),2(1-x)\}$]{
    \includegraphics[width=0.26\textwidth]{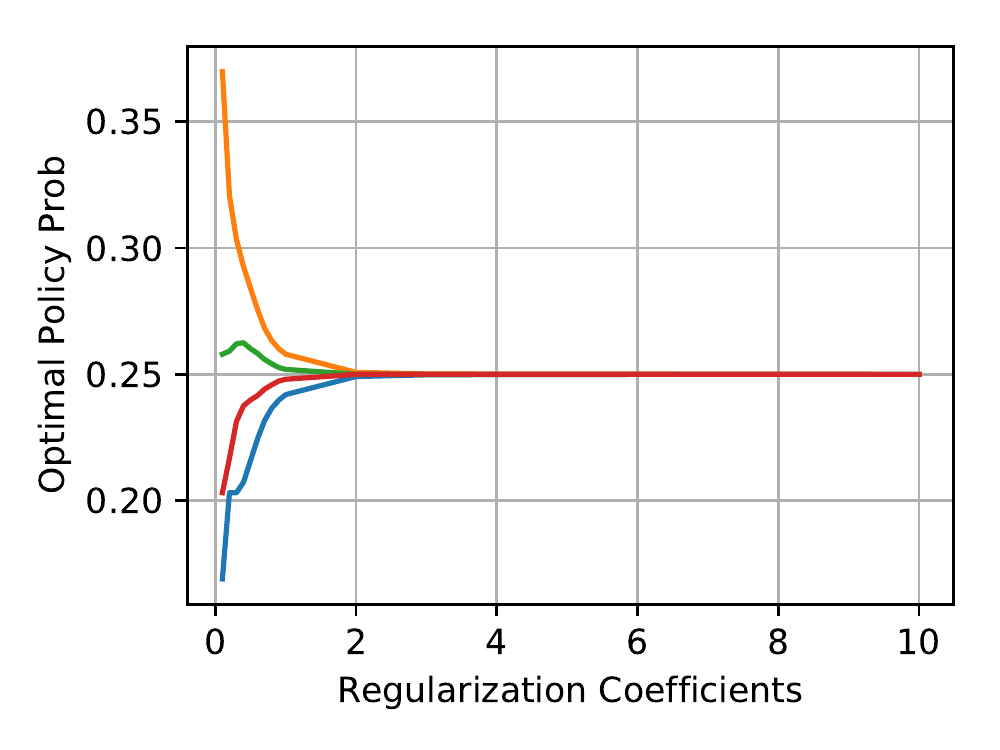}}
    \subfloat[\texttt{poly}: $\frac{1}{2}(1-x) + (1-x^2)$]{
    \includegraphics[width=0.26\textwidth]{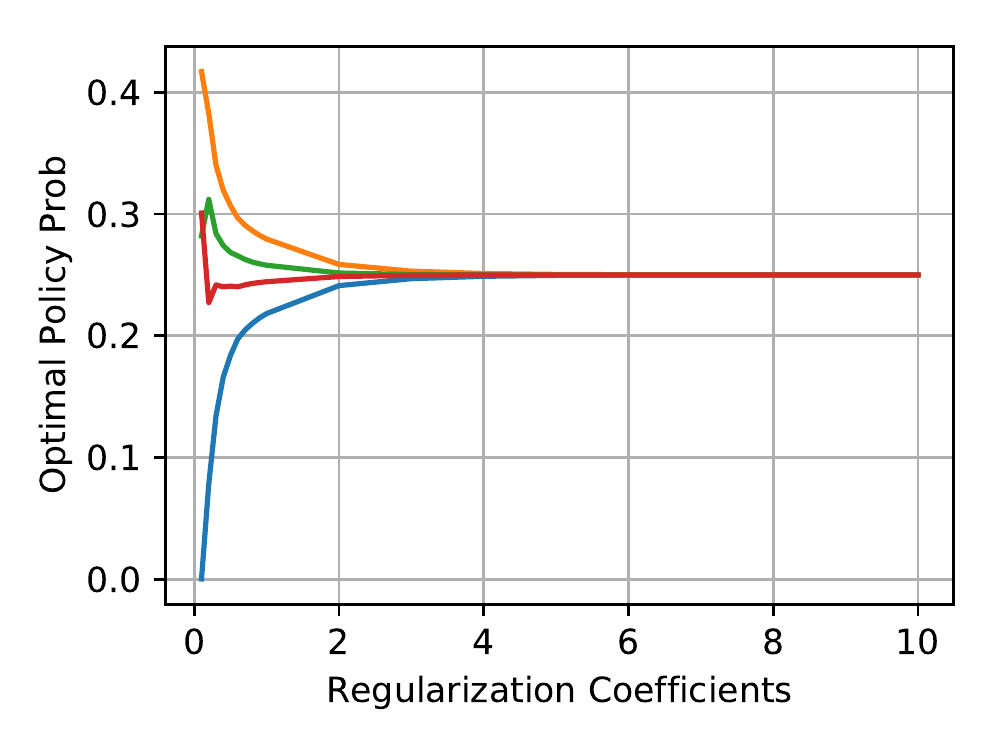}}
    \subfloat[\texttt{mix}: $-\log(x) + \frac{1}{2}(1-x)$]{
    \includegraphics[width=0.26\textwidth]{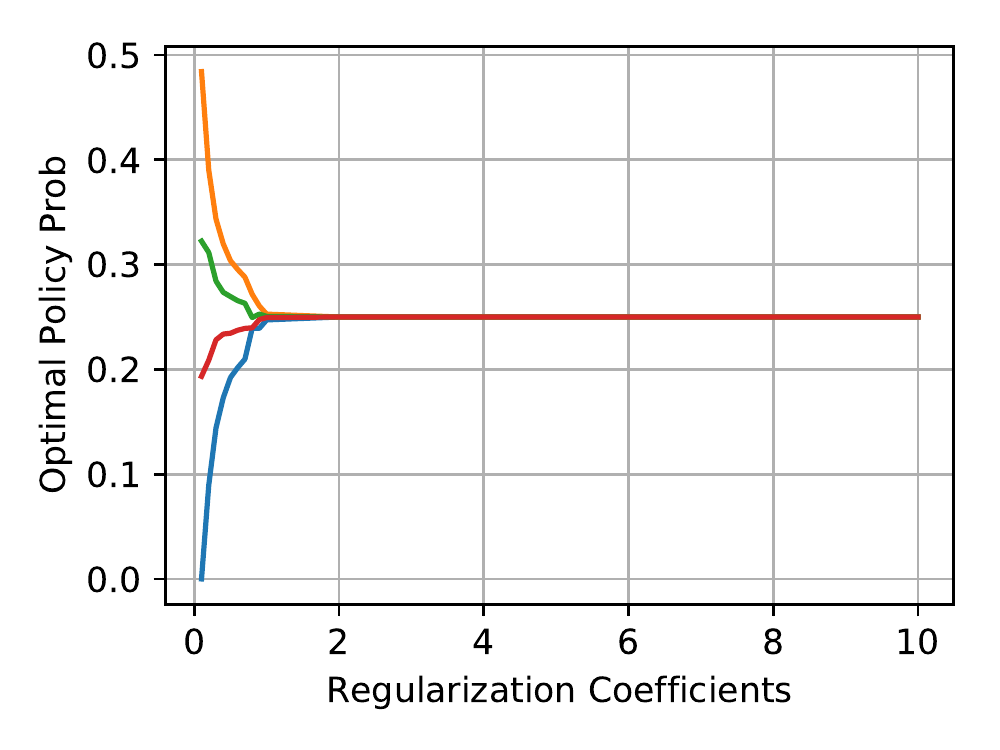}}\\
    \subfloat[\texttt{min}:$\min\{-\log(x),2(1-x)\}$]{
    \includegraphics[width=0.26\textwidth]{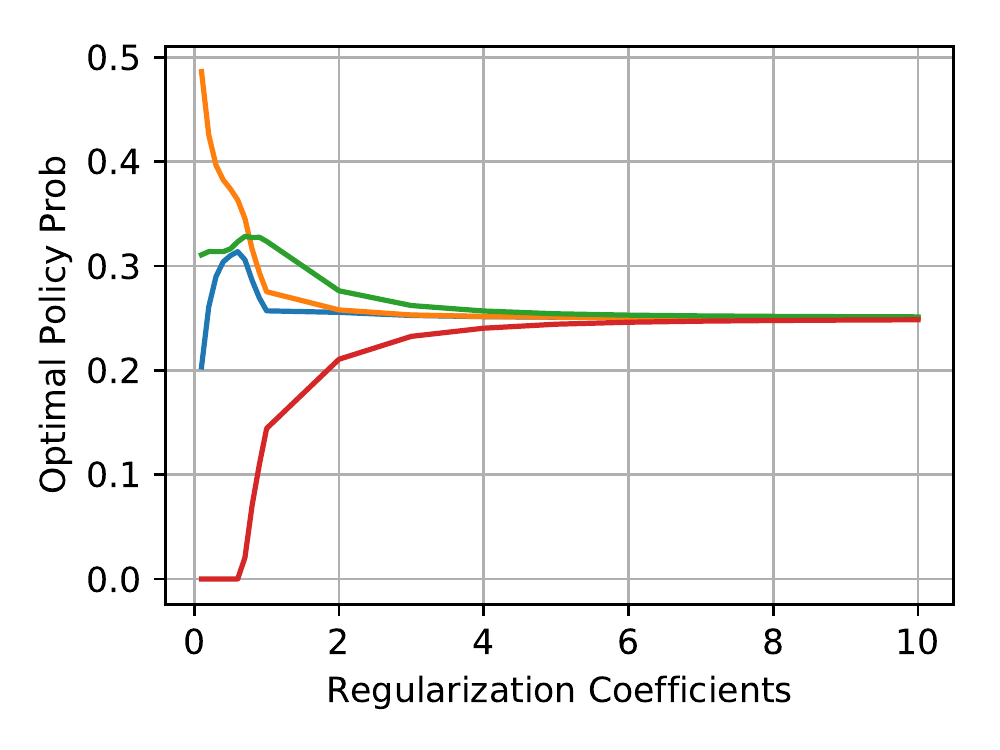}}
    \subfloat[\texttt{poly}: $\frac{1}{2}(1-x) + (1-x^2)$]{
    \includegraphics[width=0.26\textwidth]{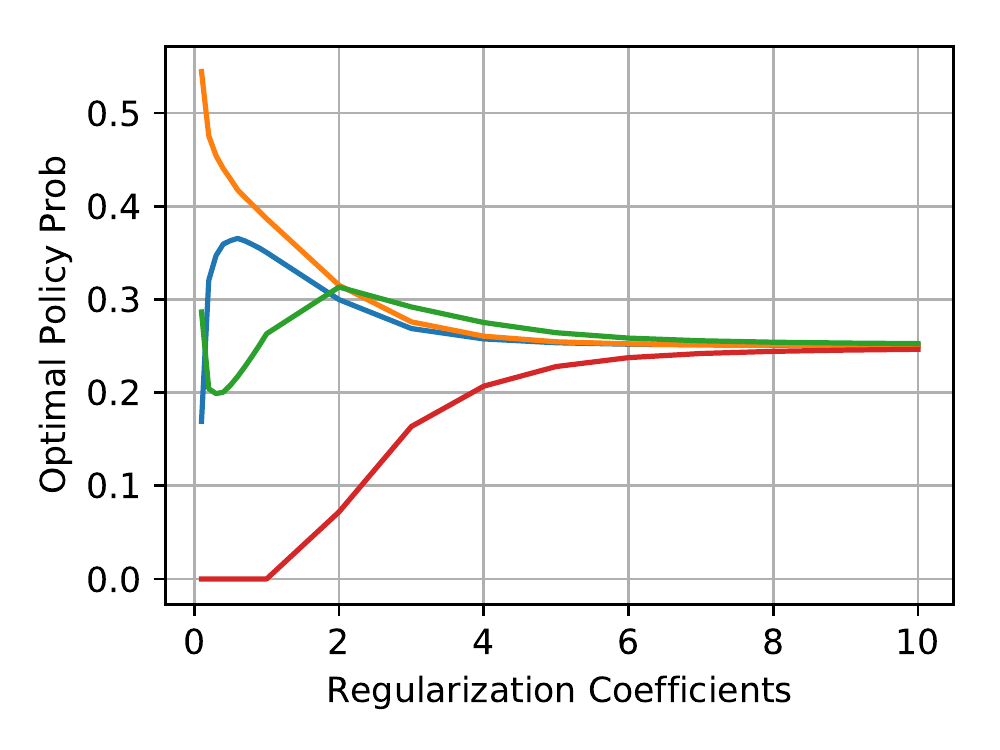}}
    \subfloat[\texttt{mix}: $-\log(x) + \frac{1}{2}(1-x)$]{
    \includegraphics[width=0.26\textwidth]{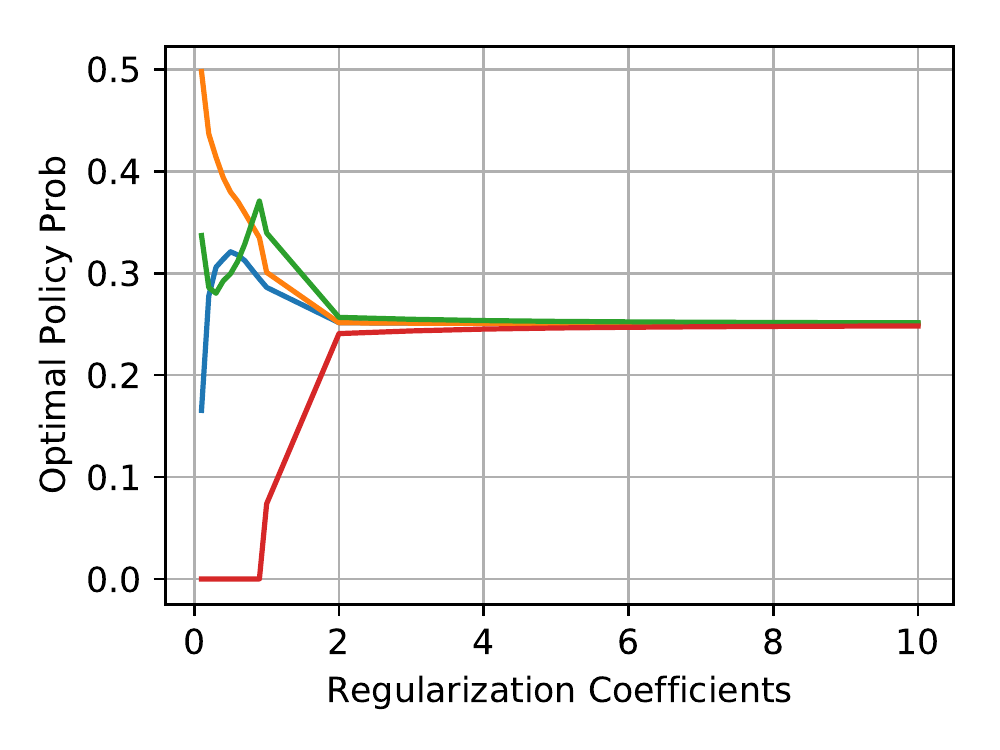}}\\
    \subfloat[\texttt{min}:$\min\{-\log(x),2(1-x)\}$]{
    \includegraphics[width=0.26\textwidth]{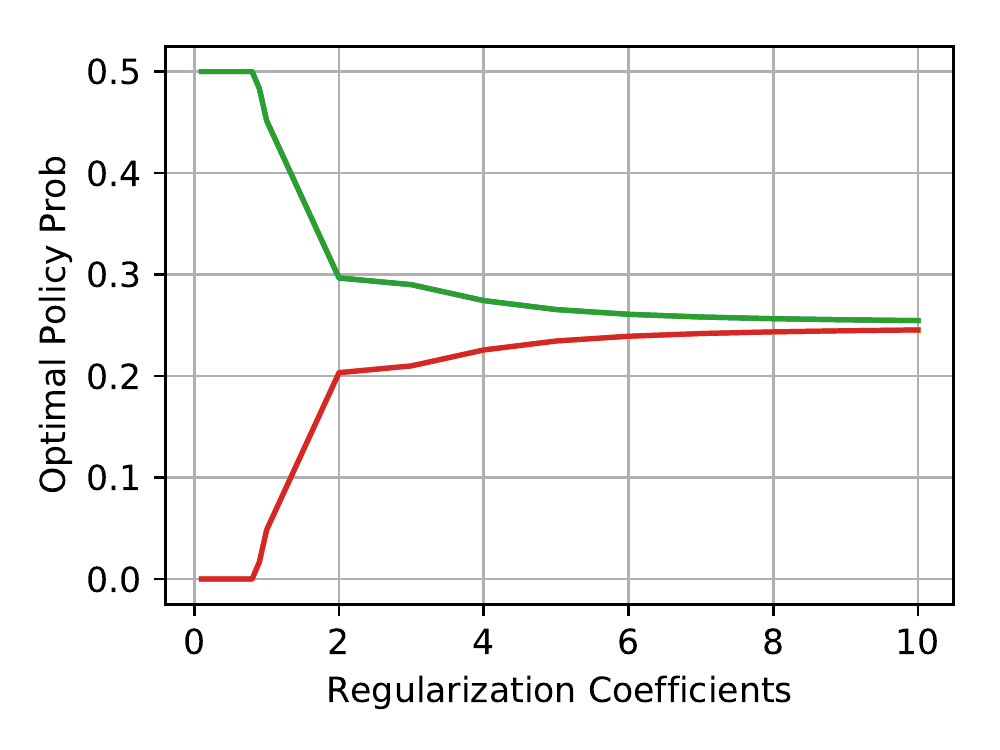}}
    \subfloat[\texttt{poly}: $\frac{1}{2}(1-x) + (1-x^2)$]{
    \includegraphics[width=0.26\textwidth]{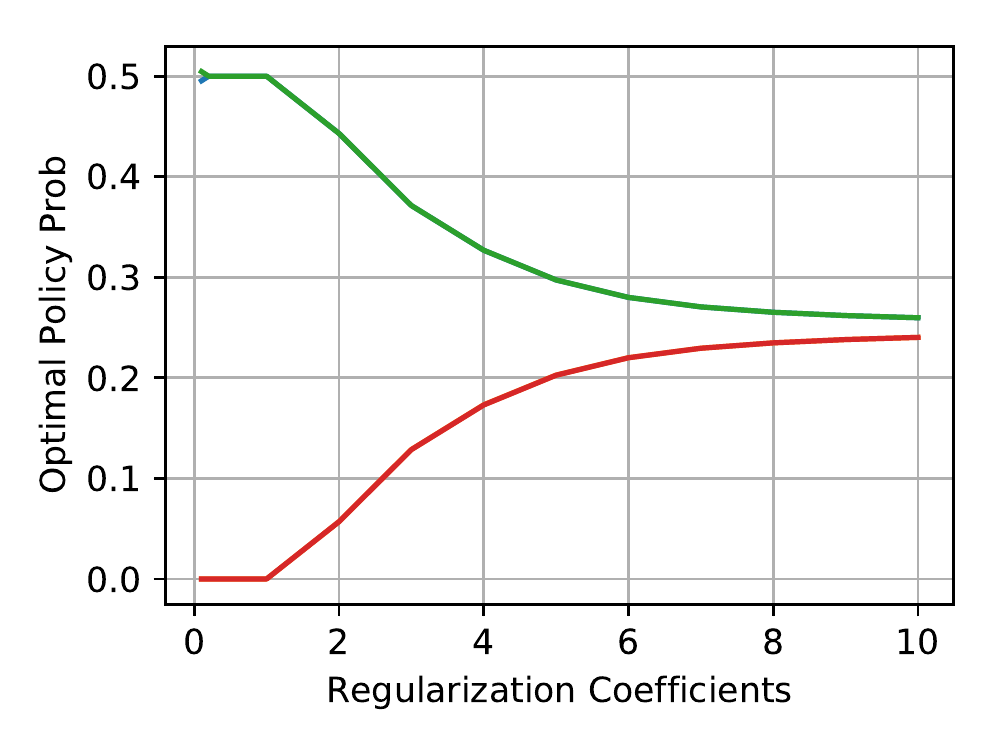}}
    \subfloat[\texttt{mix}: $-\log(x) + \frac{1}{2}(1-x)$]{
    \includegraphics[width=0.26\textwidth]{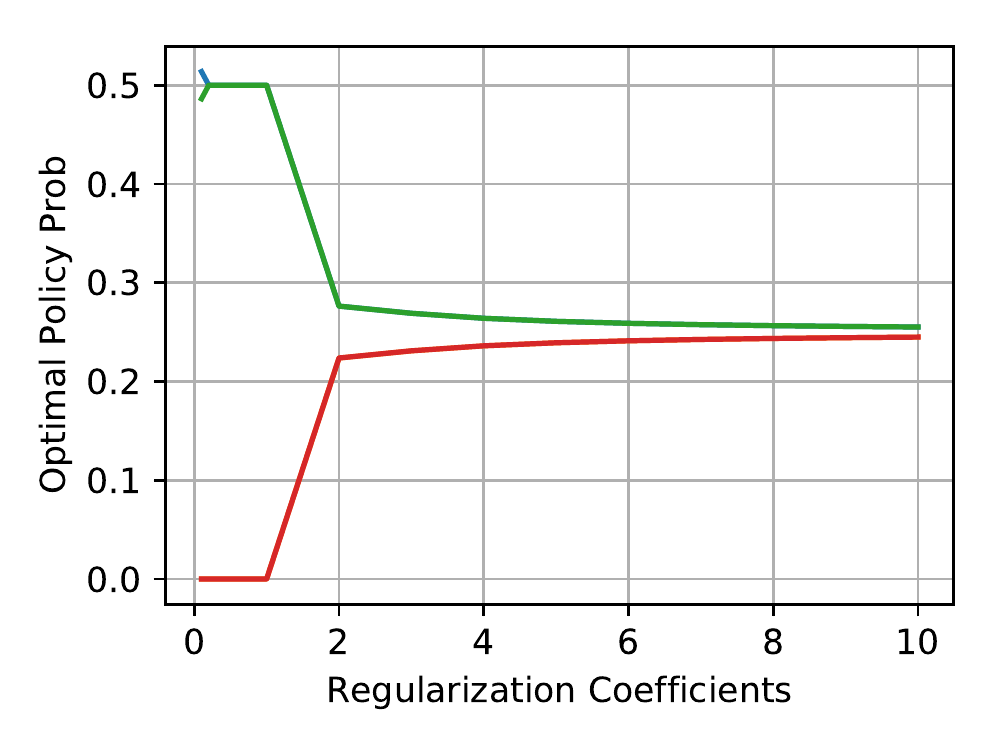}}\\
    \caption{The probability mass on four actions in the optimal policy regularized by three combined regularization functions at selected three states. (a)-(c) shows the results for the origin $(0,0)$. (d)-(f) shows the results for the state $(0, N/2)$ and (g)-(h) shows the results for the state $(N/2, N/2)$.}
    \label{multi_combine_action}
\end{figure*}

\subsection{Atari Environments}
\label{ap:exp_arari}

We test our regularizers on OpenAI Gym benchmark with Atari environments: AlienNoFrameskip-v4, BoxingNoFrameskip-v4, BreakoutNoFrameskip-v4 and SeaquestNoFrameskip-v4. 

 \paragraph{Architecture.} We model the Q-values and policies with deep neural networks. The Q-value network is composed of 3 convolutional layers, 1 fully connected laryer, and 1 output fully connected layer as the following scheme. In particalr, the first convolutional layer $C_1$ has 32 $8\times8$ filter with stride 4, the second $C_2$ contains $64$ $4\times4$ filters with stride 2, and the third $C_3$ has 64 $3\times3$ filters with stride 1. The fully connected layer $F_1$ consists of $512$ hidden units and the layer $F_2$ is a $512 \times |\mathcal{A} |$ fully connected layer. Each layer except the final layer is followed with a rectified linear activation(ReLU). For \texttt{shannon}, the architecture of policy network is the same as Q-value network except the final layer is replaced with softmax function. For other regularizers with $0\not\in\mathrm{dom}f_{\phi}'$, the final layers are replaced with a softmax fully connected layer and a ReLU fully connected layer. The final ReLU fully connected layer serves as dual variables. The output probability is the multiplication of the two layers and scale the sum to 1. We use the Adam optimizer with learning rate 0.0001 and $\varepsilon=0.0015$. The discount was set to $\gamma=0.99$. We update the target network every 10000 steps. The size of experience replay buffer is 100000 tuples, where 32 minibatches were sampled every 4 steps to update the network.
 
  \begin{flushleft}
 \begin{tikzcd}
\text{Q-value}: C_1\arrow[r, "ReLU"]&C_2\arrow[r, "ReLU"]&C_3\arrow[r, "ReLU"]&F_1\arrow[r, "ReLU"]& F_2
\end{tikzcd}
 \end{flushleft}
 
 \begin{flushleft}
 \begin{tikzcd}
\text{Policy (\texttt{shannon})}: C_1\arrow[r, "ReLU"]&C_2\arrow[r, "ReLU"]&C_3\arrow[r, "ReLU"]&F_1\arrow[r, "Softmax"]& F_2 
\end{tikzcd}
 \end{flushleft}

 \begin{flushleft}
 \begin{tikzcd}
\text{Policy (\texttt{Other})}: C_1\arrow[r, "ReLU"]&C_2\arrow[r, "ReLU"]&C_3\arrow[r, "ReLU"]&F_1\arrow[rd, "ReLU"']\arrow[r, "Softmax"]& F_2 \arrow[r] & F_4 \arrow[r, "Scale"] &F_5 \\
& & & & F_3 \arrow[ur, "\odot"] & &
\end{tikzcd}
 \end{flushleft}

\textbf{Parameter sensitivity.} We show how learning performance changes when $\lambda$ varies in Figure~\ref{fig:atari_sense}. Large $\lambda$ will make the policy becomes nearly uniform and unable to make use of the information of rewards. Small $\lambda$ will make the policy becomes nearly deterministic and therefore be stuck in local minima since no sufficient exploration is made. In the experiment of Breakout, we find that \texttt{shannon} is insensitive to $\lambda$. However, for other regularizers, small or large $\lambda$ would make the algorithm fail to converge.
\begin{figure*}[ht]
    \vspace{-10pt}
    \centering
    \subfloat[Breakout Shannon]{
    \includegraphics[width=0.25\textwidth] {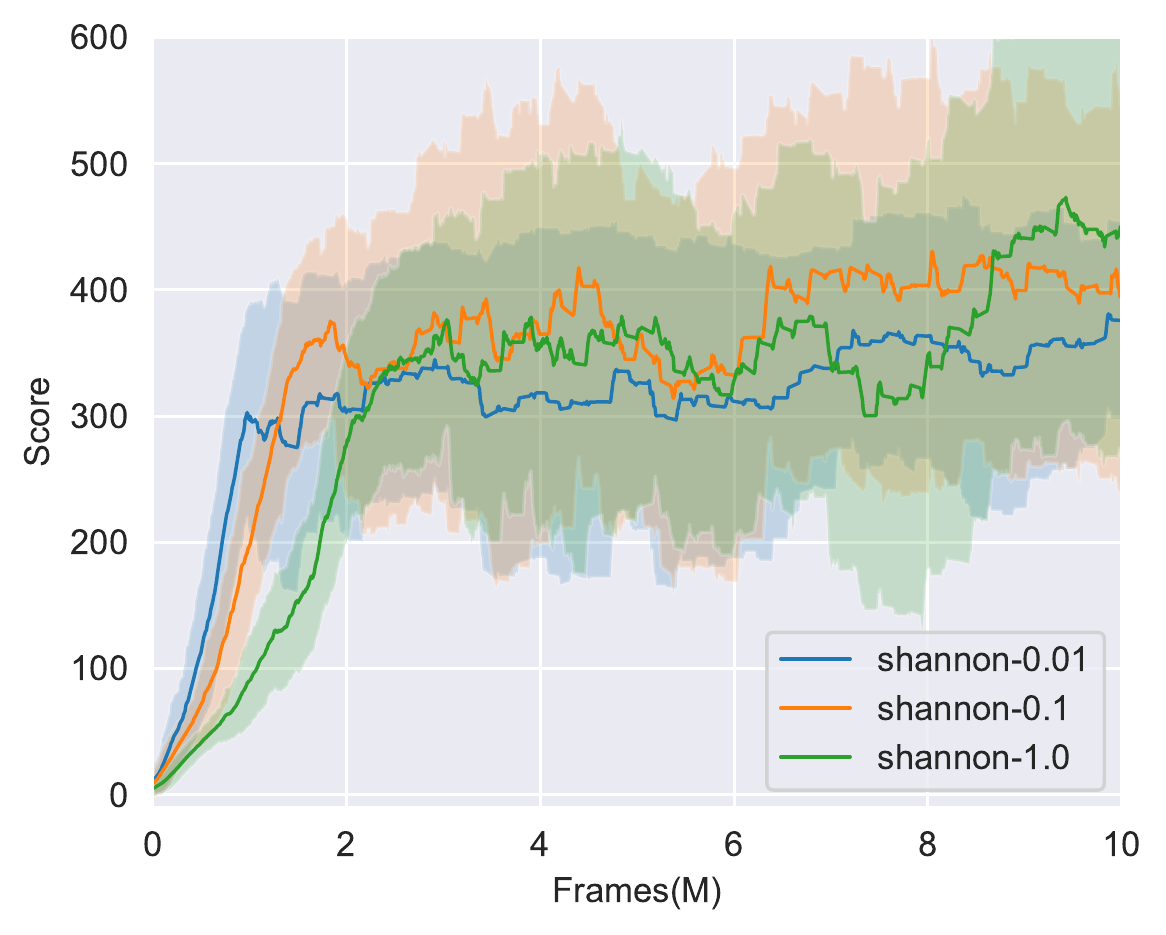}}
    \hspace{-.1in}
    \subfloat[Breakout Tsallis]{
    \includegraphics[width=0.25\textwidth] {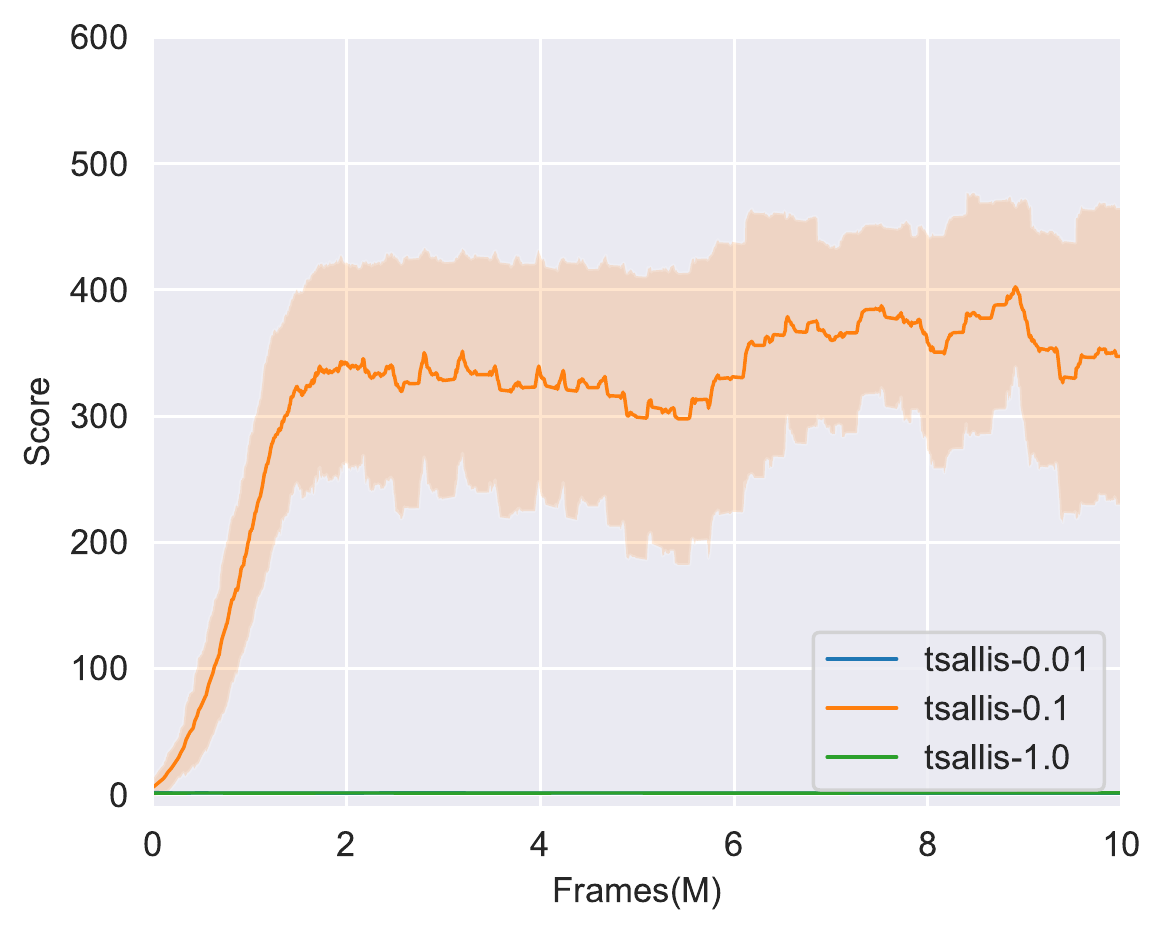}}
    \hspace{-.1in}
    \subfloat[Breakout Expx]{
    \includegraphics[width=0.25\textwidth] {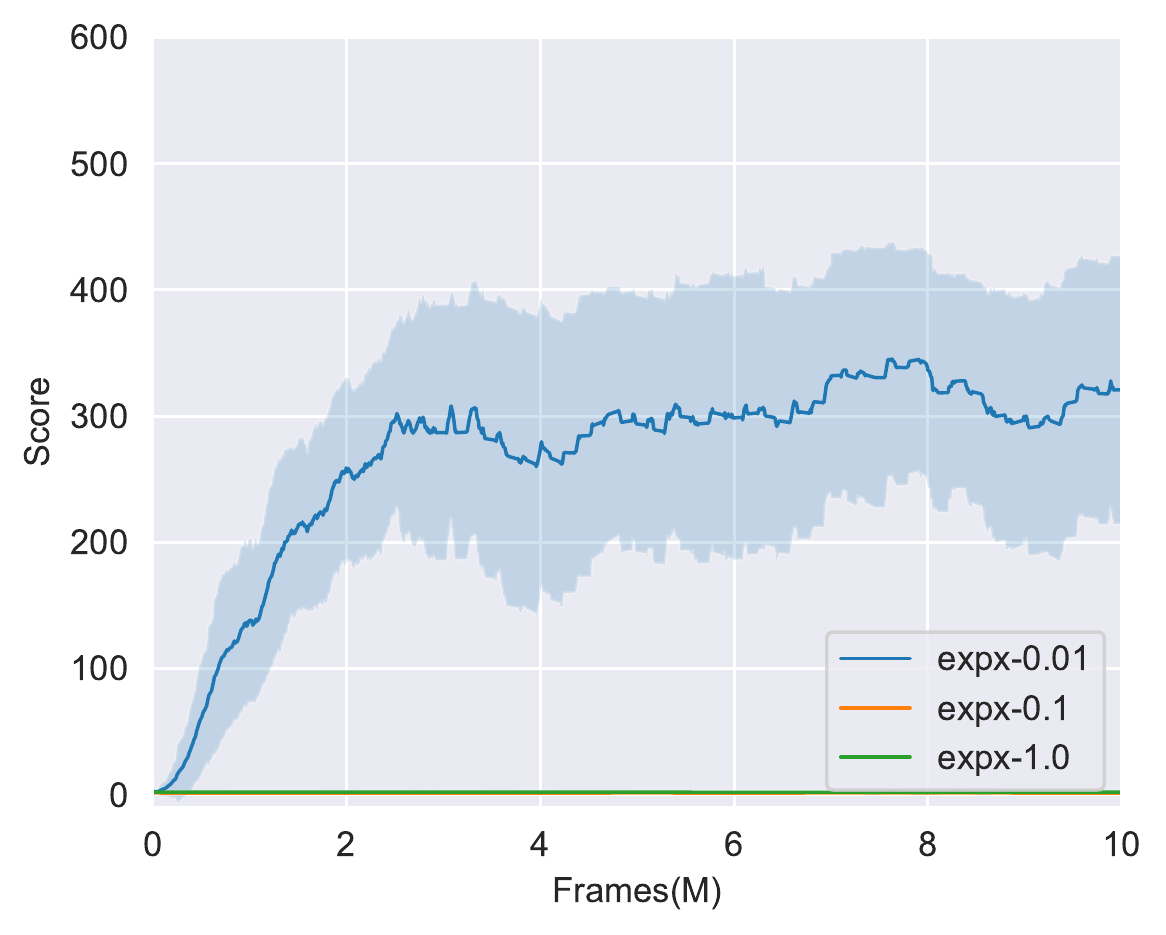}}
    \hspace{-.1in} 
    \subfloat[Breakout Cosx]{
    \includegraphics[width=0.25\textwidth] {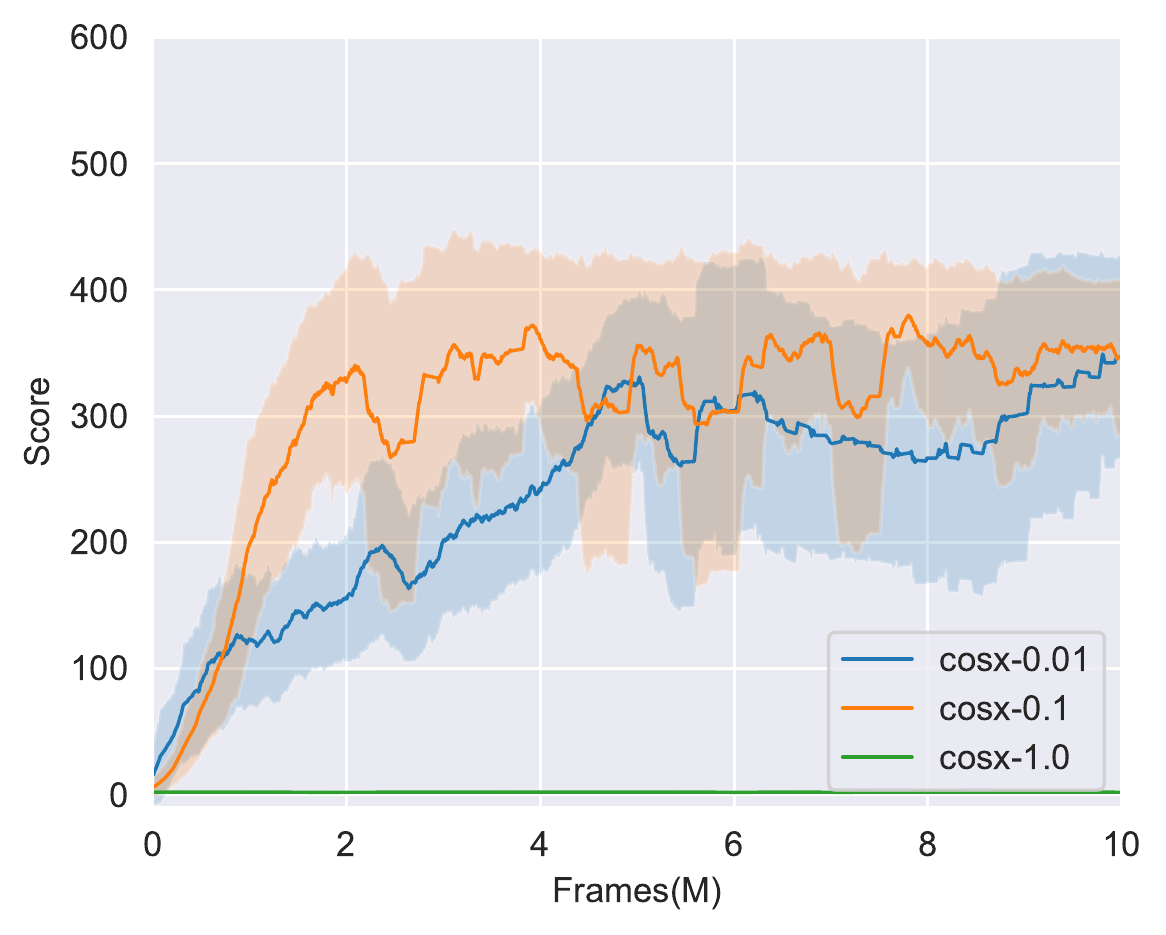}}
    \hspace{-.1in} 
    \caption{Training curves on Atari games. Each entry in the legend is named with the rule $\texttt{the regularization form} + \lambda$. The score is smoothed with 100 windows while the shaded area is the one standard deviation.}
    \label{fig:atari_sense}
    \vspace{-10pt}
\end{figure*}

\subsection{Mujoco Environments}
\label{ap:exp_mujoco}
We choose OpenAI Gym benchmark with the MuJoCo simulator for our test environments, including Hopper-v2, Walker-v2, HalfCheetah-v2 and Ant-v2. We exclude \texttt{cos} due to its numerical unstability. Then
we only consider three regularizers, i.e., $-\log x$, $\frac{1}{2}(1-x)$ and $\exp(1)-\exp(x)$ for their stable performance in deep RL training process. Since RAC is very similar to SAC except RAC is agonostic to regularization forms. We build our code on the work of SAC~\cite{haarnoja2018soft1}. For comparison's purpose, we use the same network structure and hyper-parameter settings. Figure~\ref{fig:mujoco_all} shows that the full experiments we conducted in each environment. Each regularizer is coupled with three regularization parameter $\lambda \in \{1, 0.1, 0.01\}$.

\paragraph{The reparameterization trick.} Mujoco is continuous problem, we model the policy $\pi_{\psi}(a|s)$ as a factorized Gaussian distribution with the mean and variance modeled as neural networks. Besides, we can update policy parameters like eqn\eqref{eq:loss-policy} as there is no access to compute expectation over $\pi$ in continuous setting. We apply the reparameterization trick to update the policy. The policy is reparameterized as an factorized gaussian with tanh output, i.e.,
\begin{equation*}
    \pi_{\psi}(s, \epsilon)=\tanh(\text{mean}_{\psi}(s)+\epsilon\cdot\text{std}_{\psi}(s))
\end{equation*}
where $\epsilon$ is an input noise vector, sampled from the standard Gaussian. Denote the generative action $a_{t}=\tanh(Z_{t})$ and $Z_{t}$ is a multivariate normal distribution, we have the density transformation $\pi(a_{t})=\mathcal{N}(Z_{t})|\det(\frac{da_{t}}{d Z_{t}})|^{-1}$, where $\log\det(\frac{da_{t}}{dZ_{t}})=\sum_{i=1}^{\mathcal{A}}\log(1-\tanh^{2}(Z_{t,i}))$. Therefore, we can rewrite the policy loss as:
\begin{equation}
    J_{\pi}(\psi) = \hat{\EB}_{\DP}\left[ -\lambda \phi(\pi_{\psi}(s_t, \epsilon_t))  - Q_{\theta}(s_t, \phi(\pi_{\psi}(s_t, \epsilon_t)) ) \right].
\end{equation}
We now approximate the gradient of $J_{\pi}(\psi)$ with:
\begin{equation*}
    \Hat{\nabla}J_{\pi}(\psi) 
= \nabla_{\psi} \phi(\pi_{\psi}(a_t|s_t)) + (\nabla_{a_t} \phi(\pi_{\psi}(a_t|s_t)) {-} \nabla_{a_t}Q(s_t, a_t)) \nabla_{\psi} \pi_{\psi}(s_t, \epsilon_t),
\end{equation*}
where $a_{t}$ is evaluated at $\pi_{\psi}(s_{t}, \epsilon_{t})$.

\textbf{Parameter sensitivity.} As reported by~\citet{haarnoja2018soft}, $\texttt{shannon}$ is very sensitive to the regularization coefficient $\lambda$ (which is also referred as the temperature parameter). As an extreme example, when $\lambda=1$, $\texttt{shannon}$ fails to converge in Walker-v2 and Ant-v2. By contrast, $\texttt{tsallis}$ is less sensitive to $\lambda$. As $\lambda$ varies from 0.01 to 1, the performance of $\texttt{tsallis}$ doesn't degrade to much. $\texttt{Exp}$ is also insensitive to hyperparameter $\lambda$.

\begin{figure*}[ht]
    \centering
    \subfloat[Ant-v2]{
    \includegraphics[width=0.25\textwidth] {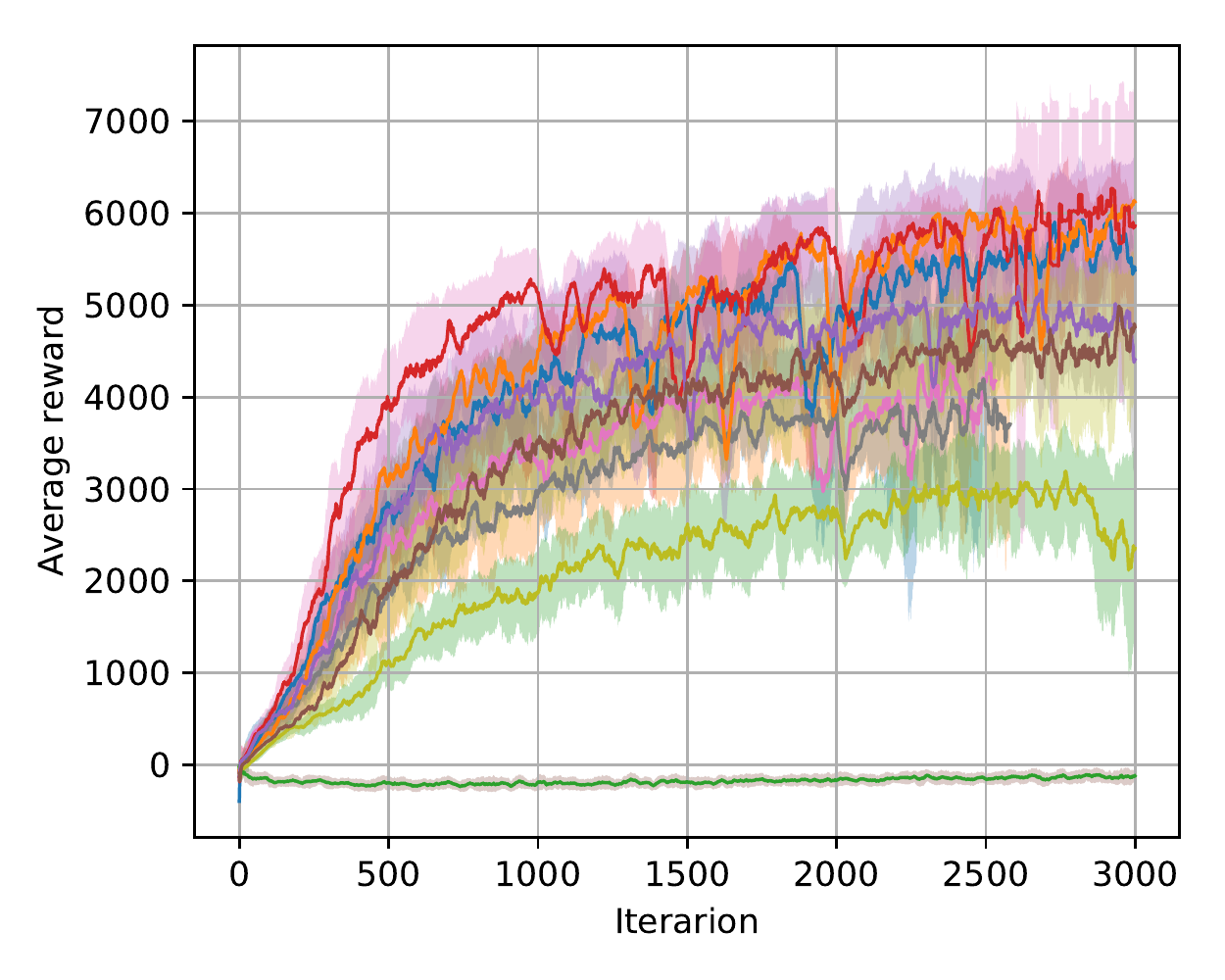}}
    \hspace{-.1in}
    \subfloat[Walker-v2]{
    \includegraphics[width=0.25\textwidth] {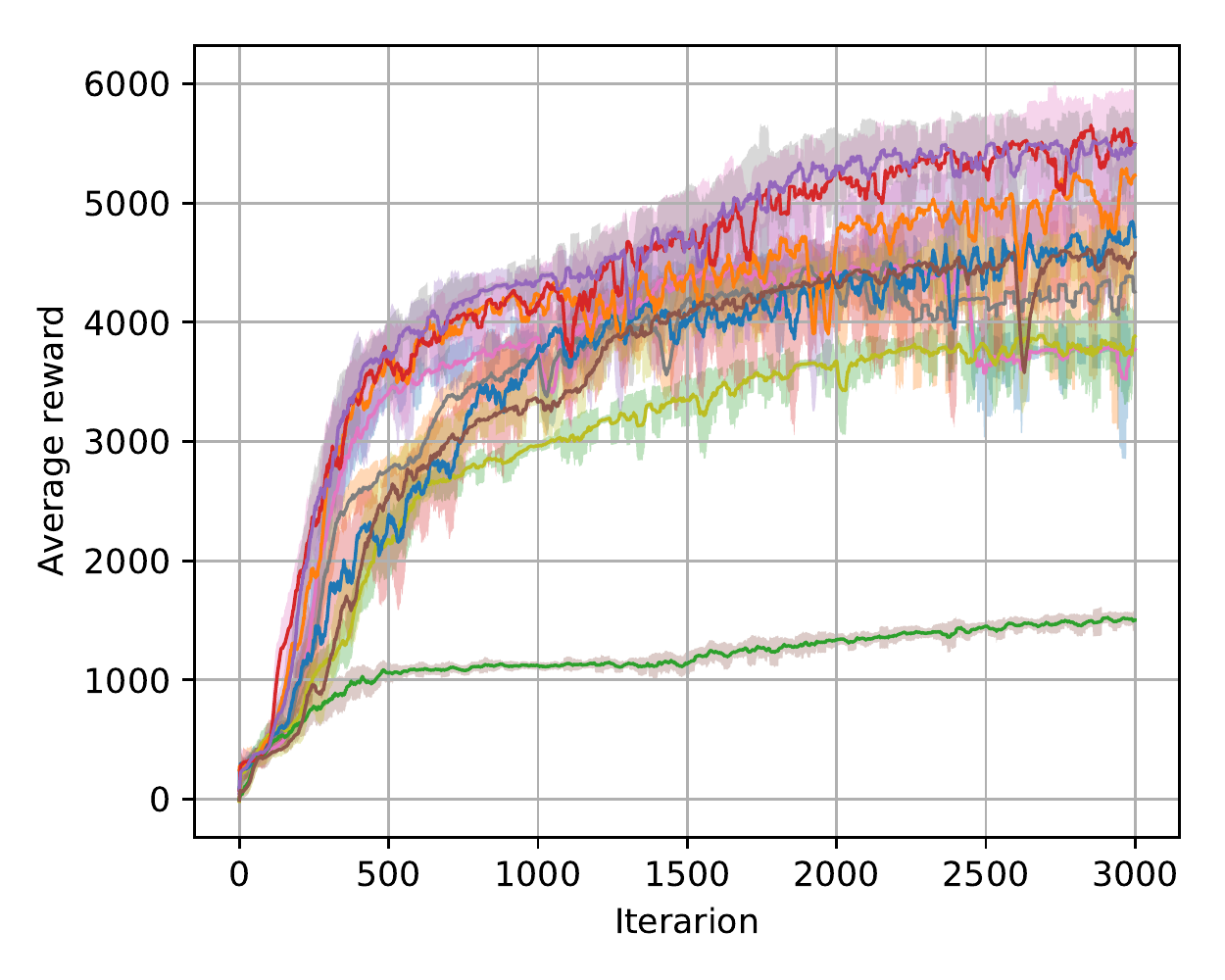}}
    \hspace{-.1in}
    \subfloat[Hopper-v2]{
    \includegraphics[width=0.25\textwidth] {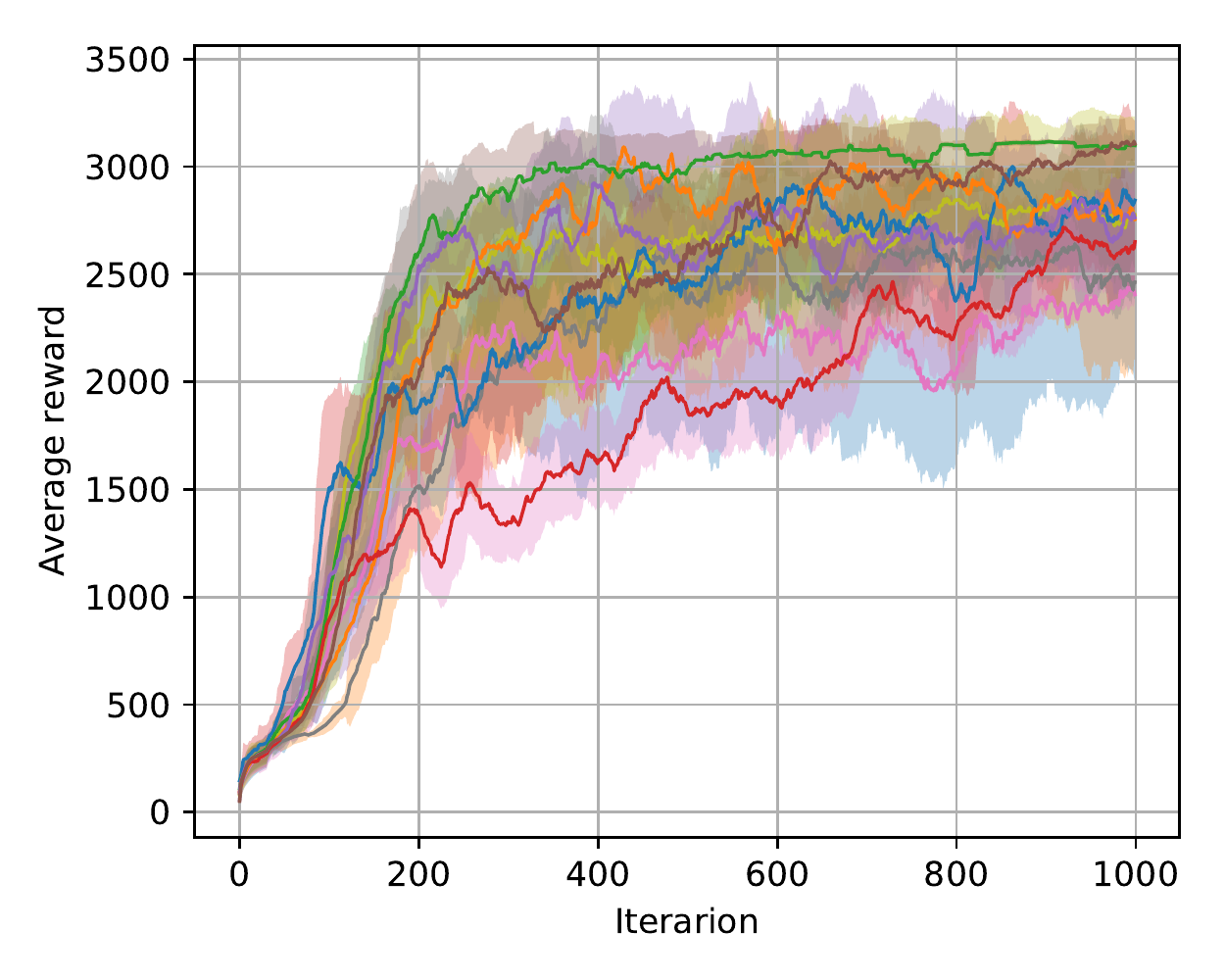}}
    \hspace{-.1in} 
    \subfloat[HalfCheetah-v2]{
    \includegraphics[width=0.25\textwidth] {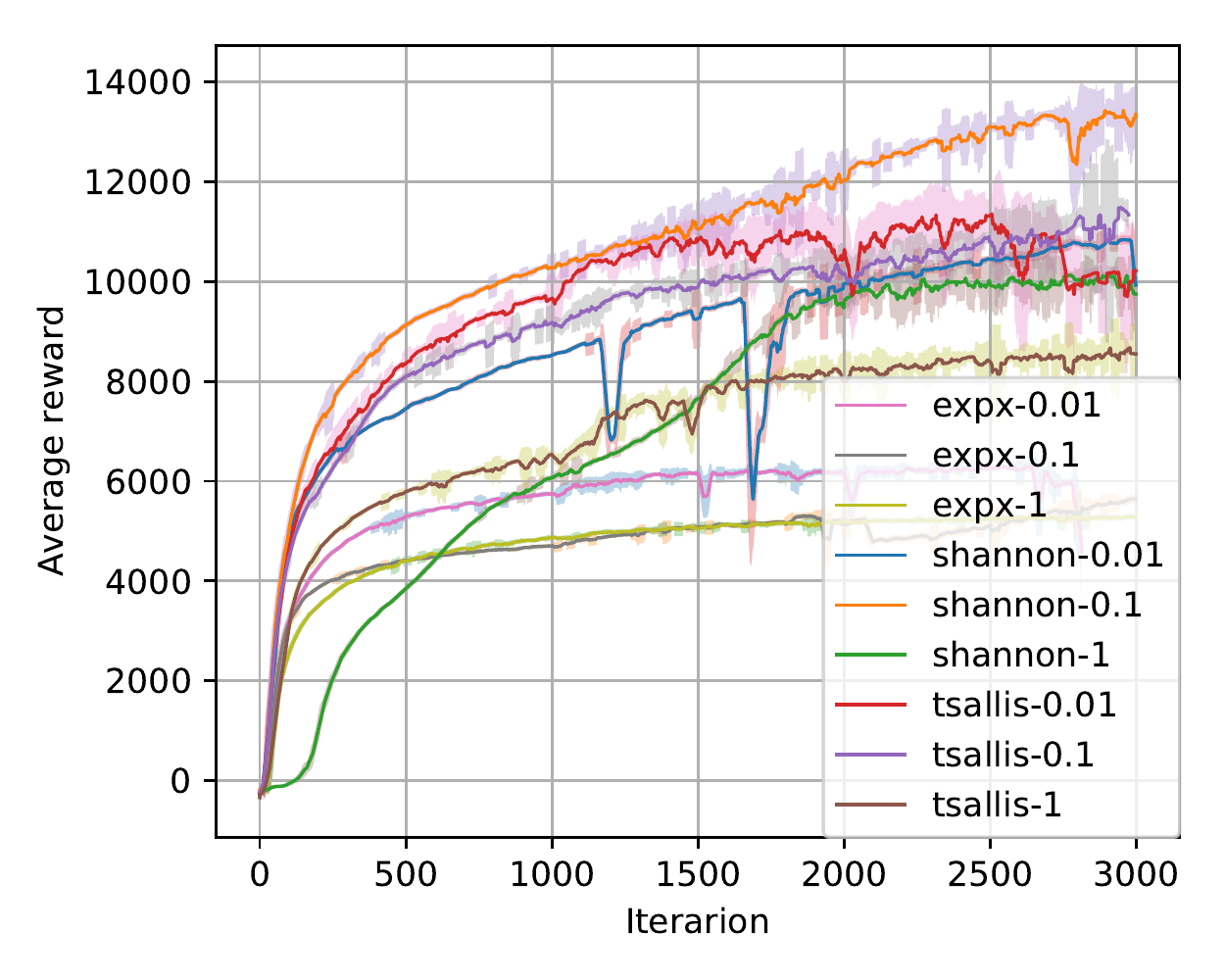}}
    \hspace{-.1in} 
    \caption{Training curves on continuous control benchmarks. Each curve is the average of four experiments with different seeds. Each entry in the legend is named with the rule $\texttt{the regularization form} + \lambda$. The score is smoothed with 30 windows while the shaded area is the one standard deviation.}
    \label{fig:mujoco_all}
\end{figure*}

\end{appendix}

\end{document}